\documentclass[11pt]{article}
\usepackage[left=1in,top=1in,right=1in,bottom=1in,letterpaper]{geometry}
\usepackage[linktocpage,linkcolor=blue,anchorcolor=blue,citecolor=blue,urlcolor=blue]{hyperref}

\newcommand{\kb}[1]{{\color{red}\bf[KB: #1]}}
\newcommand{\murat}[1]{{\color{magenta}\bf[Murat: #1]}}

\usepackage{fancyhdr,listings,amsfonts,amsmath,color,amsthm}
\hypersetup{colorlinks=true,linkcolor=blue}
\usepackage{latexsym,amssymb,amsmath,amsfonts,amsthm,xcolor,graphicx}
\usepackage[utf8]{inputenc}
\usepackage{amsmath}
\usepackage{graphicx,enumitem}
\usepackage{float}
\usepackage{color}
\usepackage{pifont}
\usepackage{amssymb}
\usepackage{multirow}
\usepackage{diagbox}
\usepackage{cancel}
\usepackage{bbm}
\usepackage{xcolor}
\usepackage{mathtools}
\newcommand{\lv}{\left\lVert}
\newcommand{\rv}{\right\rVert}
\newcommand{\m}{\mathbf}

\newcommand{\mb}{\mathbb}

\newcommand{\n}{\nu}

\newtheorem{theorem}{Theorem}
\newtheorem{lemma}{Lemma}
\newtheorem{remark}{Remark}

\usepackage{mathtools}

\newtheorem{proposition}{Proposition}[section]
\newtheorem{assumption}{Assumption}[section]

\allowdisplaybreaks

\title{On the Ergodicity, Bias and Asymptotic Normality of \\ Randomized Midpoint Sampling Method}
\author{Ye He\thanks{Department of Mathematics, University of California, Davis.\texttt{leohe@ucdavis.edu}. Research of this author was supported in part by NSF TRIPODS  CCF-1934568 and UC Davis CeDAR (Center for Data Science and Artificial Intelligence Research) Innovative Data Science Seed Funding Program.} 
\and Krishnakumar Balasubramanian\thanks{Department of Statistics, University of California, Davis. \texttt{kbala@ucdavis.edu}. Research of this author was supported in part by UC Davis CeDAR (Center for Data Science and Artificial Intelligence Research) Innovative Data Science Seed Funding Program.}
\and Murat A. Erdogdu\thanks{Department of Computer Science and Department of Statistical Sciences at
   the University of Toronto, and Vector Institute. \texttt{erdogdu@cs.toronto.edu}. Research of this author was supported in part by NSERC Grant [2019-06167], Connaught New Researcher Award, CIFAR AI Chairs program, and CIFAR AI Catalyst grant}
}

\begin{document}
\maketitle

\begin{abstract}
The randomized midpoint method, proposed by~\cite{shen2019randomized}, has emerged as an optimal discretization procedure for simulating the continuous time Langevin diffusions. Focusing on the case of strong-convex and smooth potentials, in this paper, we analyze several probabilistic properties of the randomized midpoint discretization method for both overdamped and underdamped Langevin diffusions. We first characterize the stationary distribution of the discrete chain obtained with constant step-size discretization and show that it is biased away from the target distribution. Notably, the step-size needs to go to zero to obtain asymptotic unbiasedness. Next, we establish the asymptotic normality for numerical integration using the randomized midpoint method and highlight the relative advantages and disadvantages over other discretizations. Our results collectively provide several insights into the behavior of the randomized midpoint discretization method, including obtaining confidence intervals for numerical integrations.  
\end{abstract} 

\section{Introduction}
We consider the problem of computing the following expectation
\begin{align}\label{eq:expectation}
  \mathbb{E}_{\pi}[\varphi(x)] \ \ \text{ where }\ \ \pi(x) = \tfrac{1}{Z_f}e^{-f(x)},
\end{align}
for a potential function $f:\mathbb{R}^d \to \mathbb{R}$
and a test function $\varphi:\mathbb{R}^d \to \mathbb{R}$, when the normalization constant $Z_f=\int e^{-f(x)}dx$
is unknown.
This problem frequently arises in statistics and machine learning with numerous
applications to high-dimensional Bayesian
inference~\cite{welling2011bayesian, li2016preconditioned, mandt2017stochastic, durmus2019high},
numerical integration~\cite{lamberton2002recursive, hairer2006geometric},
volume computation~\cite{vempala2010recent}, optimization and
learning~\cite{raginsky2017non, erdogdu2018global, mazumdar2020thompson},
graphical models~\cite{koller2009probabilistic}, and molecular
dynamics~\cite{milstein2013stochastic, leimkuhler2016molecular}.
Markov chain Monte Carlo (MCMC) methods provide a powerful framework for computing the integral in~\eqref{eq:expectation}, and have been successfully deployed in various scientific fields~\cite{liu2008monte}.

In particular, MCMC algorithms that are based on diffusion processes 
have received a lot of attention recently. The fundamental idea behind such algorithms is that a continuous-time diffusion with its invariant measure as the target $\pi$ is approximately simulated via a numerical sampler. The intuition behind the success of these methods is that by appropriately selecting the step-size parameter, the discrete approximation resulting from the numerical sampler tracks the continuous-time diffusion. Thus, rapid convergence properties of the diffusion process (see, for example,~\cite{roberts1996exponential, lelievre2016partial, eberle2016reflection, eberle2017couplings, livingstone2019geometric, durmus2017convergence}) is inherited by the discrete algorithm with an invariant measure that is close to that of the diffusion, which is the target $\pi$. While a variety of diffusion processes can lead to a rich class of MCMC samplers, algorithms that are based on discretizing Langevin dynamics have been the primary focus of research due to their simplicity, accuracy, and well-understood theoretical guarantees in high-dimensional settings~\cite{dalalyan2017furthur, cheng2017convergence,cheng2018sharp,durmus2019analysis,vempala2019rapid, ma2019there, cheng2017convergence, durmus2018efficient, durmus2017nonasymptotic,erdogdu2020convergence}. 
 
Although motivated by the problem of computing the integral in~\eqref{eq:expectation}, much of the theoretical focus on analyzing sampling methods in the recent literature has been on providing guarantees for the sampling problem itself (see~\cite{teh2016consistency} for an exception), i.e., the number of iterations needed to reach $\epsilon$-neighborhood of a $d$-dimensional target distribution in some probability metric. The choice of step-size of the sampler is crucial to obtain such theoretical guarantees.  While the problem of estimating expectations such as in~\eqref{eq:expectation} is based on sampling from the target $\pi$ itself, the theoretical guarantees established for the sampling problem can provide very little to no information on computing the expectation in~\eqref{eq:expectation} based on the sampler. The main reason for this is, the step-size choice of the sampler required to obtain optimal theoretical guarantees for numerical integration of ~\eqref{eq:expectation} turns out to be different from that of sampling. Furthermore, if the ultimate task is to perform inference on the quantity $\mathbb{E}_\pi[\varphi(x)]$, confidence intervals are required. Thus, one needs central limit theorems (CLT) to quantify the fluctuations of the estimator of the expectation in~\eqref{eq:expectation}, depending on a specific numerical integrator being used.

The randomized midpoint method, a numerical sampler proposed by~\cite{shen2019randomized}, has emerged as an optimal algorithm for sampling from strongly log-concave densities, achieving the information theoretical lower bound for this problem in terms of both dimension and tolerance dependency~\cite{cao2020complexity}.
In lieu of this optimality result, one anticipates a superior performance from the randomized midpoint method in other fundamental problems that relies on a
MCMC sampler as the main computation tool, e.g. estimating expectations of the form~\eqref{eq:expectation}. However, properties of this sampler for the purpose of numerical integration, in particular its inferential properties, are not well-studied. In this paper, we explore various probabilistic properties of the randomized midpoint discretization method, when used as a numerical integrator. Towards that, we examine several results for the randomized midpoint method considering both the overdamped and underdamped Langevin diffusions. Our first contribution is the explicit characterization of the bias of the randomized midpoint numerical scheme, namely the difference between its stationary distribution and the target distribution $\pi$. We show that asymptotic unbiasedness, a desired property in general, can be achieved under a decreasing step-size sequence. As our  principal contribution, we establish the ergodicity of the randomized midpoint method and prove a central limit theorem which can be leveraged for inference on the expectation~\eqref{eq:expectation}. We compute the bias and the variance of the asymptotic normal distribution for various step size choices, and show that different step-size sequences are suitable for making inference in different settings. 


 \vspace{0.1in}
\textbf{Our Contributions.} We summarize our contributions as follows:
 \vspace{-0.1in}
\begin{enumerate}[leftmargin=18pt,itemsep=3pt]
\item We show the ergodicity of constant step-size (denoted as $h$) randomized midpoint discretization of the overdamped and underdamped Langevin diffusions in Theorems~\ref{ergodicityRLMC} and~\ref{RHMCErgodicity}, respectively. For both cases, the stationary distribution $\pi_h$ of the resulting discretized Markov chain is unique and is biased away from the target distribution $\pi$.  
\item The choice of a constant step-size for the randomized midpoint discretization causes bias in sampling. We characterize this bias explicitly in Propositions~\ref{biasRLMC} and~\ref{biasRHMC} for the overdamped and underdamped Langevin diffusions, respectively. We show that Wasserstein-2 distance between $\pi_h$ and $\pi$ is of order $\mathcal{O}(h^{0.5})$ and $\mathcal{O}(h^{1.5})$ respectively for the overdamped and underdamped Langevin diffusions.
\item The established order of bias points toward using particular choices of decreasing step-size sequence for the sake of inference. Specifically, we prove a CLT for numerical integration using the randomized midpoint discretization of the overdamped and underdamped Langevin diffusions in Theorems~\ref{RLMCCLT} and~\ref{generalKLMCCLT} respectively, for various choices of decreasing step-size. Depending on the specific choice of step-size sequence, the CLT is either unbiased or biased. When discretizing the overdamped Langevin diffusion with polynomially decreasing step-size choices, the rate of unbiased CLT turns out to be $\mathcal{O}(n^{({1}/{3})-\epsilon})$ for any $\epsilon>0$. But the optimal rate turns out to be $\mathcal{O}(n^{1/3})$ for which one can only obtain a biased CLT. When discretizing underdamped Langevin diffusions with polynomially decreasing step-size choices, 
we show that the optimal rate can be improved to $\mathcal{O}(n^{5/8})$ under a certain condition, which is satisfied only by the class of constant test functions. 


\end{enumerate}


\subsection{Notations and Preliminaries} 

We denote an $\ell$-th order symmetric tensor of dimension $d$ by $A \in \mathbb{R}^{d\otimes \ell}$.  For a given vector $u \in \mathbb{R}^d$, we use $\|u\|$ to denote the Euclidean-norm of the vector. We define the $\ell$-th order  rank-1 tensor formed from $u\in\mathbb{R}^d$ as $ u^{\otimes \ell}$. In addition, let $A$ and $B$ be two $\ell$-th order tensors,  we define the inner product between $A$ and $B$ as $\langle A, B \rangle  =  \sum_{j_1=1}^d \cdots \sum_{j_\ell = 1}^d A_{j_1j_2\ldots j_\ell} \cdot B _{j_1 j_2 \ldots j_k }$. For a function $f:\mathbb{R}^d \to \mathbb{R}$, $\nabla f \in \mathbb{R}^d$ and $D^\ell \in \mathbb{R}^{d \otimes \ell}$ represents the gradient, and $\ell$-th order derivative tensor (for $\ell >1$). We let~$(\Omega,\mathcal{F},P)$ represent a probability space, and denote by $\mathcal{B}(\mathbb{R}^d)$, the Borel~$\sigma$-field of $\mathbb{R}^d$. We use $\stackrel{d}{\to}$ and $\stackrel{p}{\to}$ to denote convergence in distribution and probability respectively. The set of all twice continuously differentiable functions $f:\mathbb{R}^d\to\mathbb{R}$ is denoted as $\mathcal{C}^2(\mathbb{R}^d)$. We use $I_d$ to represent the $d\times d$ identity matrix. Let $x_0, x_1, \ldots$ be a $d$-dimensional Markov chain. The transition probability of the chain, at the $k$-th step is defined as $P^k(x,A):=  P(x_k \in A| x_0 =x)$, for some $x \in \mathbb{R}^d$ and represents the probability that the chain is in set $A$ at time $n$ given the starting point was $x\in\mathbb{R}^d$. We use $\mathcal{\tilde{O}}$ to hide $\log$ factors. Finally, for a sequence $\gamma_k$ and positive integer $\ell$, we define  $\Gamma_n^{(\ell)}:=\sum_{k=1}^n \gamma^\ell_k$. We also make the following assumption on the potential function.

\begin{assumption}\label{assumptionsonf}
  The potential function $f \in \mathcal{C}^2(\mathbb{R}^d)$ satisfies the following properties. For some $0 < m \leq M < \infty$:
  (a) $f$ has a $M$-Lipschitz gradient; that is, $D^2 f\preceq M I_d$, and (b) $f$ is $m$-strongly convex; that is, $m I_d \preceq D^2 f$. We also define the condition number as $\kappa:=M/m$.
\end{assumption}

 
\section{Results for the Overdamped Langevin Diffusion}\label{sec:Langevin}
The overdamped Langevin diffusion is described by the following stochastic differential equation:
\begin{align}\label{eq:langevindiff}
dx(t)=-\nabla f(x(t)) dt+\sqrt{2}dW(t),
\end{align}
where $W(t)$ is a $d$-dimensional Brownian motion. 
It is well-known that this diffusion has $\pi(x) \propto e^{-f(x)}$ as its stationary distribution under mild regularity conditions. 
In general, simulating a continuous-time diffusion such as \eqref{eq:langevindiff} is impractical; thus, a numerical integration scheme is needed. 

We now describe the \emph{randomized midpoint discretization} of the above diffusion in~\eqref{eq:langevindiff}, which we denote as RLMC. 
Denoting the $n$-th iteration of the algorithm with $x_n$, the integral formulation of the diffusion with $x_n$ as the initial value would then be $x_n^*(t)=x_n-\int_0^t \nabla f(x_n^*(s))ds+\sqrt{2}W(t)$. Let $h>0$ be the choice of step size for the discretization and, let $(\alpha_n) $ be an i.i.d. sequence of random variables
following uniform distribution on $[0,1]$, i.e. $\alpha_n\sim U[0,1]$. The fundamental idea behind the randomized midpoint technique is to use $h\nabla f(x_n^*(\alpha_{n+1} h))$ to approximate the integral
$\int_0^h \nabla f(x_n^*(s)) ds$. Indeed, notice that $\mathbb{E}[h\nabla f(x_n^*(\alpha_{n+1} h))] = h\int_0^1\nabla f(x_n^*(\alpha h))d\alpha=\int_0^h\nabla f(x_n^*(s))ds$. RLMC proceeds by
approximating $x_n^*(\alpha_{n+1} h)$ with the Euler discretization, which ultimately yields an explicit numerical integration step. Although~\cite{shen2019randomized} considered this discretization only for the constant step-size choice and
the \emph{underdamped} Langevin diffusion (which we discuss in Section~\ref{sec:HMCresults}), below we present a single iteration of the RLMC algorithm with the choice of variable step-size $\gamma_{n+1}$ for the overdamped diffusion in~\eqref{eq:langevindiff}:
\begin{equation}\tag{\textsc{RLMC}}
\begin{aligned}\label{RLMCdisc}
    x_{n+\frac{1}{2}}&=x_n-\alpha_{n+1} \gamma_{n+1} \nabla f(x_n)+\sqrt{2\alpha_{n+1} \gamma_{n+1}} U_{n+1}', \\
    x_{n+1}&=x_n-\gamma_{n+1}\nabla f(x_{n+\frac{1}{2}})+\sqrt{2\gamma_{n+1}}U_{n+1},
\end{aligned}
\end{equation}
where $(U_n)$ and $(U_n')$ are sequences of i.i.d $d$-dimensional standard Gaussian vectors with cross-covariance matrix $\sqrt{\alpha_n}I_d$ for each $n$ and the initial point $x_0$. We briefly digress now to make the following remark. If instead of $\alpha_n \sim U[0,1]$, one uses $\alpha_n =1$ for all $n$ deterministically, then the iterates of~\eqref{RLMCdisc} algorithm is reminiscent of the extra-gradient descent algorithm from the optimization literature~\cite{luo1993error}, perturbed by Gaussian noise in each step. Furthermore, its noteworthy that with the deterministic choice of $\alpha_n=1$, one cannot obtain the improved rates that the uniformly random $\alpha_n$ provides. Lastly, the filtration $(\mathcal{F}_n)$ is defined by $\mathcal{F}_n:=\sigma(\alpha_k, U_k, U_k'; 1\le k\le n)$, the smallest $\sigma$-algebra generated by the noise sequence and uniform random variables that are used in the first $n$ iterations.

\subsection{Wasserstein-2 Rates for Constant Step-size RLMC}\label{sec:W2ConstantRLMC}
Before, we state our main result, we investigate a few important characteristics of the~\eqref{RLMCdisc} algorithm that are not explored yet.
We start with its rate of convergence in Wasserstein-2 distance (see~\cite{villani2009wasserstein} for definition) for the~\eqref{RLMCdisc} algorithm. The proof of the proposition below essentially follows from a similar idea of the more general result for the underdamped Langevin dynamics in~\cite{shen2019randomized}. We include the result with its proof for the sake of completeness. 
\begin{proposition}\label{W2RLMC}
  Suppose $f$ satisfies Assumption~\ref{assumptionsonf}. Set $x_0=\arg\min_x f(x)$,
  $\gamma_n:=h=\mathcal{O}({\epsilon^{2/3}}/{\kappa^{1/3}M})$ when $\kappa hM>{1}$,
  and $\gamma_n:=h=\mathcal{O}({\epsilon}/{M})$
  when $\kappa hM\leq 1$ with $Mh<\frac{1}{4}$.
  After running the~\eqref{RLMCdisc} algorithm for 
\[
    K= \Tilde{\mathcal{O}}\left(\frac{\kappa^{4/3}}{\epsilon^{2/3}}+\frac{\kappa}{\epsilon}\right) \  \text{ steps,}
 \]
we have $W_2(\nu_K,\pi)\le \epsilon \sqrt{{d}/{m}}$, where $\nu_K$ is the probability distribution of $x_K$. 
\end{proposition}
When $\kappa$ is of constant order, we see that $W_2$ rate is of order $\tilde{\mathcal{O}}(1/\epsilon)$. Notably, with the randomized midpoint technique, we obtain this particular $\epsilon$-dependency by discretizing just the overdamped Langevin diffusion with only the Lipschitz gradient condition on the potential function $f$. Prior works require Euler-discretization of higher-order Langevin diffusions to obtain a $W_2$ rate of order $\tilde{\mathcal{O}}(1/\epsilon)$~\cite{dalalyan2019user, mou2019high} or require higher-order smoothness assumption along with other specialized discretization methods~\cite{sabanis2019higher, li2019stochastic, durmus2019high, dalalyan2019user}. 

\subsection{Analysis of the Markov Chain Generated by Constant Step-size RLMC }\label{sec:RLMCChainAnalyis}
Using the randomized midpoint technique, we obtain an improved dependency on $\epsilon$ for the $W_2$ rate under weaker assumptions while discretizing the Langevin diffusion in~\eqref{eq:langevindiff}. Although not explicit from the proof of Proposition~\ref{W2RLMC}, the rate improvement is obtained by a careful balancing of bias and variance through the choice of step-size parameter $h$.  In this section, in Theorem~\ref{ergodicityRLMC}, we first show that the~\eqref{RLMCdisc} Markov chain is ergodic and has a unique stationary distribution, denoted by $\pi_h$. Due to the choice of constant step-size $h$, it is not hard to see that the stationary distribution of the \eqref{RLMCdisc} is different from the stationary distribution $\pi$ of the Lanvegin diffusion in~\eqref{eq:langevindiff}, i.e $\pi_h\neq \pi$. Hence, in Proposition~\ref{biasRLMC},
we characterize the Wasserstein-2 distance between $\pi$ and $\pi_h$. 

Firstly, if $f\in \mathcal{C}^2(\mb{R}^d)$ and $f$ has a Lipschitz gradient with parameter $M$, then we can immediately see that the transition kernel of chain $(x_n)$, $P(x,y)\in \mathcal{C}(\mb{R}^d\times\mb{R}^d)$ is positive everywhere. Therefore, it's easy to obtain that the chain $(x_n)$ is $\mu^{\text{Leb}}$-irreducible and aperiodic. Given all this information, we can give a sufficient condition to make sure that the chain has a unique invariant probability measure, and it is ergodic.
    \begin{theorem}\label{ergodicityRLMC}
      Let the potential function $f$ satisfy part (a) of Assumption~\ref{assumptionsonf}, and let $\gamma_n:=h$ be small enough. Then the~\eqref{RLMCdisc} Markov chain $(x_n)$ has a unique stationary probability measure $\pi_h$, and for every $x\in \mb{R}^d$, we have
 \begin{align*}
        \sup_{A\in\mathcal{B}(\mb{R}^d)}|P^n(x,A)-\pi_h(A)|\to 0\ \ \ \ \ \ \text{as }\ \ \ \ \ n\to \infty.
\end{align*}
    \end{theorem}
    We next address the question: how far is $\pi_h$ from $\pi$? This question can be typically answered
    by a careful inspection on the proof of Proposition~\ref{W2RLMC}. However, for \eqref{RLMCdisc}, this is not the case, and requires using a different technique. Towards that, we derive an upper bound of $W_2(\pi,\pi_h)$ under the same assumptions in the previous theorem and the additional assumption that $f$ is also strongly convex with parameter $m$. 
\begin{proposition}\label{biasRLMC}
Let the potential function satisfy Assumption~\ref{assumptionsonf}, and let $\gamma_n:=h\in (0,\frac{2}{m+M})$ in the~\eqref{RLMCdisc} algorithm. Then, we have 
\begin{align}
 W_2(\pi,\pi_h) \leq &~3\sqrt{dh}\frac{(1+2Mh)^2}{\kappa^{-1 }- Mh/\sqrt{3}}.
\end{align}
\end{proposition}
\begin{remark}
The above proposition shows that the order of the bias between the stationary distribution of the Langevin diffusion and that of the~\eqref{RLMCdisc} chain is of the order $\mathcal{O}(\sqrt{h})$. 
\end{remark}

\subsection{Wasserstein-2 rates and CLT with Decreasing Step-size}\label{sec:RLMCDecreasingStepsize}
In this part, we consider the \eqref{RLMCdisc} algorithm with a fast decreasing time step sequence $(\gamma_n)$ and establish a convergence rate in $W_2$ distance as well as a CLT for the numerical integration~\eqref{eq:expectation}. 
\begin{proposition}\label{W2RLMC-fast}
Suppose $f$ satisfies Assumption~\ref{assumptionsonf}. Let $x_0:=\arg\min_x f(x)$ and 
$\gamma_{n+1}\le \frac{m}{m^2+M^2(33+n)}$. After running~\eqref{RLMCdisc} algorithm for $K= \mathcal{O}\left({\kappa^{1.5}}/{\epsilon}\right)$ steps, we obtain $W_2(\nu_K,\pi)\le \epsilon \sqrt{{d}/{m}}$,
    where $\nu_K$ is the probability distribution of $x_K$.
    \end{proposition}
    \begin{remark}
      There are two aspects of this result. The first aspect is rather standard; there is no logarithmic factor in $1/\epsilon$ compared to the result in Proposition~\ref{W2RLMC}. Similar phenomenon has been previously observed for the LMC algorithm~\cite{dalalyan2019user}. The second aspect is that we never obtain the $\mathcal{O}(1/\epsilon^{2/3})$ term as in Proposition~\ref{W2RLMC}, with the constant step-size choice. This is not an artifact of our analysis. This is due to the fact that with this choice of decreasing step-size, we reduce the bias much more at the expense of slightly increased variance. However, as we demonstrate next, this choice of decreasing step-size is crucial for obtaining an unbiased CLT for numerical integration. 
\end{remark}

As the main contribution of this section, we characterize the fluctuations of~\eqref{RLMCdisc} when it is used for computing the integral $\int_{\mb{R}^d} \varphi d\pi$ for a $\pi$-integrable function $\varphi$. Choosing the Langevin diffusion in~\eqref{eq:langevindiff} with the stationary distribution $\pi$,
we have by Theorem~\ref{ergodicityRLMC} that it is ergodic, and $
    \lim_{t\to +\infty}\frac{1}{t}\int_0^t \varphi(X(s))ds=\int_{\mb{R}^d} \varphi d\pi\coloneqq\pi(\varphi)$, almost surely. Motivated by this, we first discretize the diffusion using~\eqref{RLMCdisc} and then compute a discrete analogue of the average. The procedure consists of two successive phases: 
    \begin{enumerate}[leftmargin=20pt,noitemsep]
    \item [(a)] \textbf{Discretization:}
      The \eqref{RLMCdisc} algorithm is run with a step size sequence $(\gamma_n)$ satisfying
      for all $n$, $\gamma_n>0,\ \lim_{n\to +\infty}\gamma_n=0$, and $\lim_{n\to +\infty}\Gamma_n=+\infty$,
      where $\Gamma_n:=\sum_{k=1}^n \gamma_k$.
    \item [(b)] \textbf{Averaging:} Using the \eqref{RLMCdisc} iterates $(x_n)$, construct a weighted empirical measure via the same weight sequence ${\gamma}\coloneqq (\gamma_n)$: 
 For every $n\ge{1}$ and every $\omega\in\Omega$, set  
    \[
    \pi_{n}^{\gamma}(\omega, dx):=\frac{\gamma_1 \delta_{x_0(\omega)}+\dots+\gamma_{k+1}\delta_{x_k(\omega)}+\dots+\gamma_n\delta_{x_{n-1}(\omega)}}{\gamma_1+\dots+\gamma_n},
    \]
    and use $\pi_{n}^{\gamma}(\omega,\varphi)\!:=\!\int_{\mb{R}^d}\varphi\pi_{n}^{\gamma}(\omega,dx)\!=\!\frac{1}{\Gamma_n}\!\sum_{k=1}^n \gamma_k\varphi(x_{k-1}(\omega))$ to estimate the expectation~\eqref{eq:expectation}. 
  \end{enumerate}

  For numerical purposes, for a fixed function $\varphi$, $\pi_{n}^{\gamma}(\omega,\varphi)$ can be recursively computed as follows:
  \[ \pi_{n+1}^{\gamma}(\omega,\varphi)=\pi_{n}^{\gamma}(\omega,\varphi)+\Tilde{\gamma}_{n+1}\left(\varphi(x_n(\omega))-\pi_{n}^{\gamma}(\omega,\varphi)\right)\ \ \ \text{with}\ \Tilde{\gamma}_{n+1}:=\frac{\gamma_{n+1}}{\Gamma_{n+1}}.
  \]
  We now provide the main result of this section, a central limit theorem for the algorithm~\eqref{RLMCdisc} when it is used to compute integrals of the form in~\eqref{eq:expectation}.
\begin{theorem}\label{RLMCCLT}
  Let $\pi$ be such that its potential $f$ satisfies Assumption~\ref{assumptionsonf}. Consider a test function $\varphi:\mathbb{R}^d \to \mathbb{R}$ of the form $\varphi = \mathcal{A} \phi$ for some function $\phi:\mathbb{R}^d \to \mathbb{R}$, where $\mathcal{A}$ denotes the generator of the diffusion~\eqref{eq:langevindiff}, i.e.,
  $\mathcal{A}\phi:=-\langle \nabla f ,\nabla \phi\rangle+\Delta \phi$.
  Define $\hat{\gamma}_n \coloneqq \frac{1}{\sqrt{\Gamma_n}}\sum_{k=1}^n \gamma_k^2$ and let $\hat{\gamma}_\infty = \lim_{n\to\infty}\hat\gamma_n$. Then for all $\phi \in \mathcal{C}^4(\mathbb{R}^d)$ with $D^2\phi$, $D^3\phi$ being bounded, and $D^4\phi$ being bounded and Lipschitz, and $\sup_{x\in\mathbb{R}^d} \|\nabla \phi(x)\|^2/(1+\|x\|^2) < +\infty$, we have the following central limit theorem for the numerical integration computed via~\eqref{RLMCdisc}:
  \begin{itemize}[itemsep=1pt]
  \item[(i)] If $\hat{\gamma}_\infty =0$,
    then $\sqrt{\Gamma_n}\pi_n^{\gamma}(\varphi)  \stackrel{d}{\to} \mathcal{N}(0, 2\int_{\mb{R}^d}\|\nabla \phi(x)\|^2 \pi(dx))$,

  \item [(ii)] If $\hat{\gamma}_\infty \in (0,+\infty)$, then
    $\sqrt{\Gamma_n}\pi_n^{\gamma}(\varphi) \stackrel{d}{\to} \mathcal{N}(\varrho~\hat{\gamma}_\infty,2\int_{\mb{R}^d}\|\nabla \phi(x)\|^2 \pi(dx) )$,

  \item[(iii)] If $\hat{\gamma}_\infty =+\infty$, then $\frac{\sqrt{\Gamma_n}}{\hat\gamma_n}\pi_n^{\gamma}(\varphi)  \stackrel{p}{\to}  \varrho$,
    \end{itemize}
    where the mean $\varrho$ is given as 
    \begin{equation*}
      \begin{aligned}
        \varrho=&\smallint\smallint\langle D^3\phi(x),\nabla f(x)  \otimes u \otimes u\rangle\mu(du)\pi(dx)
        -\tfrac{1}{2}\smallint\langle D^2f(x), \nabla \phi(x)\otimes\nabla f(x)\rangle \pi(dx) \\
        & +\tfrac{1}{2}\smallint\smallint\langle D^3 f(x), \nabla \phi(x) \otimes u \otimes u \rangle \mu(du)\pi(dx)
        -\tfrac{1}{2}\smallint\langle D^2\phi(x), \nabla f(x) \otimes \nabla f(x) \rangle\pi(dx)\\
        & +\int_{\mb{R}^d} trace(D^2\phi(x)^2) \pi(dx)-\tfrac{1}{6}\smallint\smallint\langle D^4\phi(x), u^{\otimes4} \rangle\mu(du)\pi(dx),
      \end{aligned}
    \end{equation*}
    and $\mu$ is the distribution for a $d$-dimensional standard Gaussian measure.
\end{theorem}

\begin{remark}
  First note that a CLT for the Euler discretization of Langevin diffusion follows from~\cite[Thm.~10]{lamberton2002recursive}. The rates of the CLT established in Theorem~\ref{RLMCCLT} are similar to that case, with only the bias term $\rho$ being different. Specifically, following the same computation in~\cite{lamberton2002recursive}, we see that the optimal rate with polynomially decaying step-size choice $\gamma_k = k^{-\alpha}$, for some $\alpha>0$, is $\mathcal{O}(n^{1/3})$. But in this case, the established CLT is biased. However, for any $0< \alpha <1/3$, we obtain an unbiased CLT as well. Hence, although the~\eqref{RLMCdisc} chain provides rate improvements for sampling (with respect to $W_2$ distance), as demonstrated in~\cite{shen2019randomized} and in Proposition~\ref{W2RLMC}, it does not seem to provide any improvements for CLT. In retrospect, this is expected as the rate improvements for sampling is achieved by the choice of constant step-size for which it is not possible to establish even a nearly unbiased CLT.
\end{remark}

The class of test functions that the above CLT can cover is intimately related to the solution of the \emph{Stein equation} (or Poisson equation) $\varphi = \mathcal{A} \phi$. Given $\varphi$, there is an explicit characterization of $\phi$ that solves the Stein's equation, and various properties of $\varphi$ are translated to $\phi$~\cite{gorham2016measuring, erdogdu2018global}. 


\section{Results for the Underdamped Langevin Diffusion}\label{sec:HMCresults}
The underdamped Langevin diffusion is given by
 \begin{align}\label{underdampedlangevidiff}
        d\begin{bmatrix}
        x(t) \\
        v(t)
        \end{bmatrix}
        =\begin{bmatrix}
        v(t) \\
        -(\beta v(t)+u\nabla f(x(t))) 
        \end{bmatrix}
        dt+\sqrt{2\beta u}\begin{bmatrix}
        {0}_d \\
        {I}_d
        \end{bmatrix} dW(t),
    \end{align}
    where $\beta>0$ is the friction coefficient and $u>0$ is the inverse mass. For simplicity, we will consider $\beta=2$ in the later text. Under mild conditions, it is well-known that the continuous-time Markov process $(x(t),v(t))$ is positive recurrent, and its invariant distribution is given by $\nu(x,v)\propto \exp\big\{-f(x)-\tfrac{1}{2u}\lv{v}\rv^2\big\},\ x\in\mathbb{R}^d,\ v\in\mathbb{R}^d$. This diffusion, with an additional Hamiltonian component, has gathered a lot of attention recently due to its improved convergence properties~\cite{dalalyan2018sampling, cheng2017underdamped, shen2019randomized, livingstone2019geometric,durmus2017convergence} and empirical performance~\cite{neal2011mcmc,chen2014stochastic}.

The \emph{randomized midpoint discretization} of the underdamped Langevin diffusion~\eqref{underdampedlangevidiff} is given as:
       \begin{align}
       \nonumber
         \hspace{-1.3in} x_{n+\frac{1}{2}}&=x_n+{\scriptstyle \tfrac{1}{2}(1-e^{-2\alpha_{n+1}\gamma_{n+1}})}v_n
         -\tfrac{u}{2} \left(\alpha_{n+1}\gamma_{n+1}\!-\!
           {\scriptstyle\tfrac{1}{2}(1-e^{-2\alpha_{n+1}\gamma_{n+1}})}\right)
         \nabla f(x_n)+\sqrt{u}\sigma_{n+1}^{(1)}U_{n+1}^{(1)},\\ \nonumber
         x_{n+1}&=x_n+\tfrac{1}{2}{\scriptstyle (1-e^{-2\gamma_{n+1}})}v_n-\tfrac{u}{2}\gamma_{n+1}{\scriptstyle(1-e^{-2(1-\alpha_{n+1})\gamma_{n+1}})} 
         \nabla f(x_{n+\frac{1}{2}})+\sqrt{u}\sigma_{n+1}^{(2)}U_{n+1}^{(2)}, \\
         v_{n+1}&=v_n {\scriptstyle e^{-2\gamma_{n+1}}} \label{rhmc}
         -u\gamma_{n+1}{\scriptstyle e^{-2(1-\alpha_{n+1})\gamma_{n+1}}}
         \nabla f(x_{n+\frac{1}{2}})+2\sqrt{u}\sigma_{n+1}^{(3)}U_{n+1}^{(3)},        \tag{\textsc{RULMC}}
        \end{align}
    where $(\gamma_n)$ is the sequence of time steps, $\sigma_n^{(1)}$, $\sigma_n^{(2)}$ and $\sigma_n^{(3)}$ are positive with $({\sigma_n^{(1)}})^2=\alpha_n\gamma_n+\frac{1-e^{-4\alpha_n\gamma_n}}{4}-(1-e^{-2\alpha_n\gamma_n})$, $({\sigma_n^{(2)}})^2=\gamma_n+\frac{1-e^{-4\gamma_n}}{4}-(1-e^{-2\gamma_n})$ and $({\sigma_n^{(3)}})^2=\frac{1-e^{-4\gamma_n}}{4}$, and $(\alpha_n)$ is a sequence of identically distributed random variables following the distribution $\alpha_n\sim U[0,1]$.  $(U_n^{(1)}, U_n^{(2)}, U_n^{(3)})$ are independent centered Gaussian random vectors in $\mb{R}^{3d}$, also independent of $(\alpha_n)$ and initial point $(x_0,v_0)$, having the following pairwise covariances:
\begin{align*}
  \text{cov}(\sigma_n^{(1)} U_n^{(1)}, \sigma_n^{(2)} U_n^{(2)})&=
                                                                    \left(\alpha_n \gamma_n - \left( e^{-\alpha_n\gamma_n}+e^{-2\gamma_n}\sinh(\alpha_n\gamma_n)\right)\sinh(\alpha_n \gamma_n)\right) I_{d\times d},\\
  \text{cov}(\sigma_n^{(2)} U_n^{(2)}\!, \sigma_n^{(3)} U_n^{(3)})&=\left(e^{-2\gamma_n} \sinh(\gamma_n)^2\right) I_{d\times d},\  \ \\
  \text{cov}(\sigma_n^{(1)} U_n^{(1)}\!, \sigma_n^{(3)} U_n^{(3)})&=\left(e^{-2\gamma_n} \sinh(\alpha_n\gamma_n)^2\right) I_{d\times d}.
\end{align*}
The~\eqref{rhmc} algorithm has emerged as an optimal sampling algorithm in the sense that it achieves the information theoretical lower bound in both tolerance $\epsilon$ and dimension $d$ for sampling from a strongly log-concave densities~\cite{cao2020complexity, shen2019randomized}. Therefore, it is interesting to examine if~\eqref{rhmc} based numerical integrator have any benefits in other MCMC-based tasks such as \eqref{eq:expectation}. Towards that, we characterize the order of bias with a constant step-size choice for~\eqref{rhmc} iterates as proposed in~\cite{shen2019randomized}. Compared to the bias result in Proposition~\ref{biasRLMC} for the~\eqref{RLMCdisc} discretization, we note that order of bias is increased (i.e. smaller bias). Next, in Theorem~\ref{generalKLMCCLT} we provide a CLT for numerical integration with \eqref{rhmc}. Our results show that when it comes to computing expectations of the form in~\eqref{eq:expectation} using~\eqref{rhmc} and characterizing its fluctuations, the~\eqref{rhmc} discretization obtains rate improvements only for a class of constant test functions (as described in Remark~\ref{1dCLTRHMC}).


\subsection{Analysis of the Markov Chain generated by Constant Step-size RULMC }\label{ANLSRHMC}

Recall that $\pi(x)$ is the marginal density function of $\nu(x,v)$ with respect to $x$. Similarly $\nu_h(x,v)$ be the stationary density function of the Markov chain generated by~\eqref{rhmc} chain and $\pi_h(x)$ be the marginal density function of $\nu_h(x,v)$, with respect to $x$. Furthermore, the filtration $(\mathcal{F}_n)$ is defined as $\mathcal{F}_n:=\sigma(\alpha_k, U_k^{(i)}; 1\le k\le n, i=1,2,3)$. When $f\in \mathcal{C}^2(\mb{R}^d)$ and is gradient Lipschitz with parameter $M$, then we can immediately see that the transition kernel of chain $(x_n, v_n)$: $P((x,v),(x',v'))\in \mathcal{C}(\mb{R}^{2d}\times\mb{R}^{2d})$ is positive everywhere. Therefore, it's easy to obtain that the chain $(x_n, v_n)$ is $\mu^{\text{Leb}}$-irreducible and aperiodic. Given all this information, we can give a sufficient condition to make sure that the chain has a unique invariant probability measure and is ergodic.

    \begin{theorem}\label{RHMCErgodicity}
  Let the potential function $f$ satisfy part (a) of Assumption~\ref{assumptionsonf}, and let $\gamma_n:=h$ be small enough.  Then if $u\in(0,\frac{4}{2M-m})$, the~\eqref{rhmc} Markov chain $(x_n,v_n)$ has a unique stationary probability measure $\nu_h$ and for every $(x,v)\in \mb{R}^{2d}$, we have
        \[
        \sup_{A\in\mathcal{B}(\mb{R}^{2d})}|P^n((x,v),A)-\nu_h(A)|\to 0\ \ \ \ \ \ \text{as }\ \ \ \ n\to \infty.
        \]
    \end{theorem}
We next derive an upper bound on the bias $W_2(\pi,\pi_h)$ of \eqref{rhmc} algorithm, under the additional strong convexity assumption on the potential function $f$. 

\begin{proposition}\label{biasRHMC}
Suppose that $f$ satisfies Assumption~\ref{assumptionsonf}.  If we run the~\eqref{rhmc} algorithm with $u=1/M$ and $\gamma_n:=h$, for universal constants $C_1,C_2>0$, we have
$$
W^2_2(\pi,\pi_h)\le \frac{C_1 h^3(\kappa h^3+1)d}{1-\frac{h}{4\kappa}-C_2 h^3\kappa(1+\kappa h^3)}.
$$
\end{proposition}

\begin{remark}
Note that we have $W_2(\pi,\pi_h)\to 0$ as $h\to 0$. Furthermore, as $h\to 0$, $W_2(\pi,\pi_h)< \mathcal{O}(h^{\frac{3}{2}})$. Hence, the bias order is increased for the underdamped Langevin diffusion compared to the overdamped case (cf. Proposition~\ref{biasRLMC}), providing a smaller bias for the same step-size.
\end{remark}

\subsection{Wasserstein-2 rates and CLT with Decreasing Step-size}\label{sec:RHMCdecreasing}
We now provide the rate of convergence in Wasserstein-2 metric with decreasing step-size for~\eqref{rhmc}. The specific choice for the decreasing step-size that we consider below, also is satisfied for our CLT result in Remark~\ref{1dCLTRHMC}.

\begin{proposition}\label{W2RHMC-fast}
Suppose $f$ satisfies Assumption~\ref{assumptionsonf}. Fix $u=1/M$. Let $x_0:=\arg\min_x f(x)$ and choose
$\gamma_n=\frac{16\kappa}{32\kappa^{\frac{5}{3}}+(n-K_1)^+}$, for a  $K_1\in(0,\infty) $ (where $(a)^+:= \max(0,a)$). After running ~\eqref{rhmc} for $K= \Tilde{\mathcal{O}}\left({\kappa^{3/2}}/{\epsilon^{2/3}}\right)$ steps, we obtain $W_2(\nu_K, \pi)\leq \epsilon \sqrt{{d}/{m}}$, where $\nu_K$ is the probability distribution of $x_K$.
\end{proposition}

\begin{remark}
Similar to the result in Proposition~\ref{W2RLMC-fast}, there are two aspects of this result. 
The first aspect is again removing the logarithmic factor in $1/\epsilon$ compared to the result in Theorem 3 in~\cite{shen2019randomized}, which is quite standard in the literature.  The second aspect is that we never obtain the $\mathcal{O}(1/\epsilon^{1/3})$ part, as in Theorem 3 in~\cite{shen2019randomized} with the constant step-size choice. 
\end{remark}

Similar to the previous case, we now describe the numerical integration procedure using the~\eqref{rhmc} discretization. We denote the $n$-th iterate as $(x_n,v_n)$. The time-step we use is $(\gamma_n)$ such that $\forall n\in \mathbb{N}^*, \gamma_n\ge 0, \lim_n \gamma_n=0~\text{and}~\lim_{n} \Gamma_n^{(1)}=+\infty,~\text{where}~\Gamma_n^{(\ell)}:=\sum_{i=1}^n\gamma_i^\ell$.
Our averaging is a weighted empirical measure with $Y_n=(x_n, v_n)$ using 
the step size sequence $\gamma:=(\gamma_n)$ as the weights. Let $\delta_x$ denote the Dirac mass at $x$. Then for every $n\ge 1$, set 
\[
\nu_n^{\gamma}(\omega, dx):=\frac{\gamma_1\delta_{Y_0(\omega)}+\cdots+\gamma_{k+1}\delta_{Y_k(\omega)}+\cdots+\gamma_n\delta_{Y_{n-1}(\omega)}}{\gamma_1+\cdots+\gamma_n}
\]
and we can use $\nu_n^{\gamma}(\omega, \varphi)$ to approximate $\nu(\varphi)=\mb{E}_{\nu}[\varphi'(Y)]$, where $\varphi':\mathbb{R}^{2d} \to \mathbb{R}$.  

If we assume $g:\mb{R}^{2d} \to \mathbb{R}$ such that $\mathcal{L}g = \varphi'$, we can establish the following theorem, in which we state only the unbiased CLT result for simplicity. 

\begin{theorem}\label{generalKLMCCLT}
 Let $\pi$ be such that its potential function $f$ satisfies Assumption~\ref{assumptionsonf}. Assume $u\in(0,\frac{4}{2M-m})$. Consider a test function $\varphi' = \mathcal{L} g$, for some function $g:\mathbb{R}^{2d} \to \mathbb{R}$, where  $\mathcal{L}=2u\Delta_v-2\langle v, \nabla _v\rangle-u\langle\nabla f(x), \nabla_v\rangle+\langle v, \nabla_x\rangle$ denotes the generator of the diffusion~\eqref{underdampedlangevidiff}. Suppose the step-size $(\gamma_k )$ is non-increasing, $\lim_{n\to+\infty} (1/\sqrt{\Gamma_n})\sum_{k=1}^n\gamma_k^{{3}/{2}}=+\infty$. Then, if $\lim_{n\to +\infty}(1/\sqrt{\Gamma_n})\sum_{k=1}^n\gamma_k^2=0$, for every $ g \in \mathcal{C}^4(\mathbb{R}^{2d})$ function with $D^2 g$ bounded, $D^3 g$ bounded and Lipschitz, and if the condition $\sup_{(x,v)\in\mathbb{R}^{2d}} \|\nabla g(x,v)\|/ (1+\|x\|^2+\|v\|^2)<+\infty$ holds, we have the following central limit theorem for the numerical integration computed using the~\eqref{rhmc} iterates:
\begin{align*}
    \sqrt{\Gamma_n}\nu_n^\gamma (\mathcal{L}g)\stackrel{d}{\to} \mathcal{N}\big(0, 4u\smallint\|\nabla_v g(x,v)\|^2\nu(dx,dv)\big).
\end{align*}    
\end{theorem}
The rate of convergence of the CLT in Theorem~\ref{generalKLMCCLT} follows exactly the same behavior in Theorem~\ref{RLMCCLT}. Hence, for the class of general test functions, Theorem~\ref{generalKLMCCLT} does not exhibit a rate improvement.
%
%
Towards that, we make the following remarks under a carefully constructed condition for the class of test functions.

\begin{remark} \label{1dCLTRHMC}
 Let $\pi$ be such that its potential function $f$ satisfies Assumption~\ref{assumptionsonf}. Assume $u\in(0,\frac{4}{2M-m})$. Consider a test function $\varphi = \mathcal{L}g$ which could be written as $\mathcal{L}g(v,\phi(x))=\langle v ,\nabla \phi(x) \rangle$, for some function $\phi:\mathbb{R}^d \to \mathbb{R}$, where $\mathcal{L}=2u\Delta_v-2\langle v, \nabla _v\rangle-u\langle\nabla f(x), \nabla_v\rangle+\langle v, \nabla_x\rangle$ denotes the generator of the diffusion~\eqref{underdampedlangevidiff}. Suppose the time step-size $(\gamma_k)$ is non-increasing, and satisfies $\lim_{n\to \infty}(\gamma_{n-1}-\gamma_n)/\gamma_n^4=0$ and  $\lim_{n\to\infty} \Gamma_n^{(4)}=+\infty$. 
 Define $\hat{\gamma}_n:= \Gamma_n^{(4)}/\sqrt{\smash[b]{\Gamma_n^{(3)}}}$ and 
 let $\hat{\gamma}_\infty =\lim_{n\to\infty}  \hat{\gamma}_n$. Then, for all $\phi\in \mathcal{C}^4(\mathbb{R}^d)$ with $D^2\phi $, $D^3\phi$ and $D^4\phi$ bounded and Lipschitz and $\sup_{(x,v)\in \mb{R}^{2d}} \|\nabla \phi(x)\|^2/(1+\|x\|^2+\|v\|^2)<+\infty$, we obtain the following central limit theorem for numerical integration computed using the~\eqref{rhmc} algorithm:
 \begin{enumerate}[itemsep=.1pt]
 \item [(i)] If $\hat{\gamma}_\infty =0$, we have $\frac{\Gamma_n}{\sqrt{\Gamma_n^{(3)}}}\nu_n^\gamma(\mathcal{L}\phi)\stackrel{d}{\to} \mathcal{N}(0, \frac{10}{3}u\int_{\mb{R}^d}\|\nabla \phi (x)\|\pi(dx))$,
\item[(ii)]  If $\hat{\gamma}_\infty \in (0,+\infty)$, we have $\frac{\Gamma_n}{\Gamma_n^{(4)}}\nu_n^\gamma(\mathcal{L}\phi)\stackrel{d}{\to} \mathcal{N}(\rho, \frac{10}{3}u\hat{\gamma}_{\infty}^{-2}\int_{\mb{R}^d}\|\nabla \phi (x)\|\pi(dx))$,
\item [(iii)] If $\hat{\gamma}_\infty = +\infty$, we have $\frac{\Gamma_n}{\Gamma_n^{(4)}}\nu_n^\gamma(\mathcal{L}\phi)\stackrel{p}{\to} \rho,$
 \end{enumerate}   
where,
\begin{align*}
       \rho&=\tfrac{5u}{12}\smallint \smallint \langle D^3\phi(x), \nabla f(x) \otimes v\otimes v \rangle\nu(dx,dv)+\tfrac{u}{24}\smallint \smallint \langle D^3f(x), \nabla \phi(x) \otimes v\otimes v \rangle \nu(dx,dv)\\
    &\ +\tfrac{7u}{12}\smallint \smallint(D^2\phi D^2f)(x)v^{\otimes 2}\nu(dx,dv)-\tfrac{u^2}{4}\smallint\langle D^2\phi(x),\nabla f(x)^{\otimes 2}\rangle\pi(dx)\\
    &\ -\tfrac{u^2}{24}\smallint \langle D^2 f(x), \nabla \phi(x) \otimes\nabla f(x)\rangle \pi(dx).\\
\end{align*}
\end{remark}

\begin{remark}
For polynomial time steps $\gamma_k:=k^{-\alpha}$, since we require that $\Gamma_n^{(4)}\to +\infty$ as $n\to +\infty$, we need $0<\alpha\le \frac{1}{4}$. Using L'Hospitals rule, it is straightforward to check that the condition $\lim_{n\to +\infty} \frac{\gamma_{n-1}-\gamma_n}{\gamma_n^4}=0$ is satisfied when $\alpha\in (0,\frac{1}{4}]$. We then have the following order estimates:
\begin{align*}
\Gamma_n\sim \frac{n^{1-\alpha}}{1-\alpha},\ \ \ \ \sqrt{\Gamma_n^{(3)}}\sim \frac{n^{\frac{1}{2}-\frac{3}{2}\alpha}}{\sqrt{1-3\alpha}},\ \ \ \ 
\Gamma_n^{(4)} \sim \left\{
\begin{aligned}
&\frac{n^{1-4\alpha}}{1-4\alpha},\ \ \ \text{if }\alpha\in (0,\tfrac{1}{4}), \\
& \sqrt{\ln n},\ \ \ \ \ \text{if }\alpha=\tfrac{1}{4}.
\end{aligned}
\right.
\end{align*}
Hence, as $n\to +\infty$,
\begin{align*}
    \frac{\Gamma_n^{(4)}}{\sqrt{\Gamma_n^{(3)}}}\to \hat{\gamma}_{\infty}=\left\{
    \begin{aligned}
    &\ \ 0\ \ \ \ \ \ \ \ \text{if }\alpha\in (\tfrac{1}{5},\tfrac{1}{4}],\\
    &\sqrt{10}\ \ \ \ \ \ \text{if }\alpha=\tfrac{1}{5},\\
    &+\infty\ \ \ \ \ \text{if }\alpha\in (0,\tfrac{1}{5}).
    \end{aligned}
    \right.
\end{align*}
If $\alpha\in (\frac{1}{5}, \frac{1}{4}]$, the unbiased CLT holds at rate $\Gamma_n/\sqrt{\smash[b]{\Gamma_n^{(3)}}} =\mathcal{O}(n^{\frac{1}{2}(1+\alpha)})\le \mathcal{O}(n^{\frac{5}{8}})$. The optimal rate is achieved when $\alpha=\frac{1}{4}$. If $\alpha=\frac{1}{5}$, the biased CLT holds at rate $\Gamma_n/\sqrt{\smash[b]{\Gamma_n^{(3)}}} =\mathcal{O}(n^{3\alpha})=\mathcal{O}(n^{\frac{3}{5}})$. If $\alpha\in (0,\frac{1}{5})$, the rate of the convergence in probability is $\Gamma_n/\sqrt{\smash[b]{\Gamma_n^{(3)}}} =\mathcal{O}(n^{3\alpha})<\mathcal{O}(n^{\frac{3}{5}})$. Therefore the optimal convergence rate $\mathcal{O}(n^{\frac{5}{8}})$ is obtained when an unbiased CLT holds. While the rate of this CLT is faster than the one obtained in Theorem~\ref{RLMCCLT}, the test functions that satisfy this condition is severely restricted.
\end{remark}

\section{Discussion}
In this work, we present several probabilistic properties of the randomized midpoint discretization technique, focussing our attention on overdamped and underdamped Langevin diffusion. Our results could be biased as follows: To obtain optimal rates for sampling (in $W_2$ distance), one needs to have a constant choice of step-size. With such a constant step-size choice, the Markov chain generated by the discretization process is biased. This suggest that a decreasing step-size choice is required for using the randomized midpoint method for sampling and the related task of numerical integration. For several decreasing choices of step-sizes, we establish CLTs and highlight the relative merits and disadvantages of using randomized midpoint technique for numerical integration. In particular, our results have interesting consequence for computing confidence interval for numerical integration.    
\section{Additional Notations}
We also use the following notations for the proofs. Due to the ease of presentation, whenever it is clear in the proof, we refer to the inner product between two compatible vectors $\langle a,b\rangle$ simply by $a\cdot b$. For any random variable $X$, $\lv X \rv_{L^2}:=\mb{E}[\lv X \rv]$ where the expectation is taken over all randomness of $X$.

\section{Proofs for Section~\ref{sec:Langevin}}
We now define the following condition, which is a consequence of Assumption~\ref{assumptionsonf}
\begin{assumption}
\label{ConditionV}
  There exists a twice differentiable function $V:\ \mb{R}^d\to [1,\infty)$ such that: \\
  ($i$) $\lim_{\|x\|\to \infty} V(x)=+\infty$,
  ($ii$) there exists $\alpha>0$ and $\beta>0$: $\langle \nabla V(x), \nabla f(x) \rangle\ge \alpha V(x)-\beta$ for every $x$, ($iii$) there exists $c_V>0$: $\|\nabla V(x)\|^2+\|\nabla f(x)\|^2\le c_V V(x)$ for every $x$,
  and ($iv$) $\lv D^2 V\rv_{\infty}: = \sup_{x \in \mathbb{R}^d} \| D^2V\|_{\text{op}}<\infty$ (where $\|\cdot \|_{\text{op}}$ denotes the operator norm).    
\end{assumption}
\begin{lemma}
Assumption~\ref{assumptionsonf} implies Assumption~\ref{ConditionV}.
\end{lemma}
\begin{proof}
Since $f\in \mathcal{C}^2(\mb{R}^d)$ is strongly convex, $\lim_{|x|\to +\infty}f(x)=+\infty$ and $f$ has a unique global minimizer $x^*\in \mb{R}^d$. It's easy to observe that $\nabla f(x^*)=0$. We consider our $V(x)=f(x)-f(x^*)+1$. Then it's easy to see ($i$) is satisfied. ($iv$) is also satisfied because $f$ is gradient Lipschitz. ($iii$) is equivalent to that there exists a $C>0$ such that
\[
\frac{|\nabla f(x)|^2}{f(x)-f(x^*)+1}\le C\ \ \ \ \ \text{for}\ \forall x\in \mb{R}^d
\]
We Taylor expand the numerator and denominator:
\begin{equation*}
    \begin{aligned}
    |\nabla f(x)|^2&=\sum_{i=1}^d\left( f_i(x^*)+\nabla f_i(\xi)^T(x-x^*) \right)^2 \\
    &\le \sum_{i,j=1}^d |f_{ij}(\xi)|^2|x-x^*|^2=\lv{D^2f(\xi)}\rv_F^2|x-x^*|^2 \\
    &\le d^2M^2|x-x^*|^2\\
    f(x)-f(x^*)+1&=\nabla f(x^*)^T(x-x^*)+\frac{1}{2}D^2f(\xi)(x-x^*)^{\otimes 2}+1 \\
    &=\frac{1}{2}D^2f(\xi)(x-x^*)^{\otimes 2}+1 \\
    &\ge \frac{m}{2}|x-x^*|^2
    \end{aligned}
\end{equation*}
Then 
\[
\frac{|\nabla f(x)|^2}{f(x)-f(x^*)+1} \le \frac{2d^2M^2}{m}\ \ \ \ \ \text{for}\ \forall x\in \mb{R}^d
\]
($ii$) is equivalent to that there exists $\alpha, \beta>0$ such that
\[
|\nabla f(x)|^2\ge \alpha(f(x)-f(x^*)+1)-\beta\ \ \ \ \ \text{for}\ \forall x\in \mb{R}^d
\]
According to the strongly convexity of $f$, we have 
\begin{equation*}
    \begin{aligned}
    f(x^*)-f(x)&\ge \nabla f(x)^T(x^*-x)+\frac{m}{2}|x^*-x|^2 \\
    &=\frac{m}{2}|x^*-x+\frac{1}{m}\nabla f(x)|^2-\frac{1}{2m}|\nabla f(x)|^2 \\
   \end{aligned}
\end{equation*}
which then implies
\begin{equation*}
    \begin{aligned}
    |\nabla f(x)|^2 &\ge 2m\left( f(x)  -f(x^*)+1 \right)-2m  \ \ \ \ \ \text{for}\ \forall x\in \mb{R}^d
    \end{aligned}
\end{equation*}
($ii$) is satisfied by choosing $\alpha=\beta=2m>0$. 
\end{proof}
\begin{remark}\label{ESTV1D} 
For the $V(x)$ we choose in the proof, under assumption~\ref{assumptionsonf}, we can verify that: $V(x)=O(|x|^2)$ when $|x|\to +\infty$. We will use this fact later in the proof when we establish the CLT statement. 
\end{remark}

\subsection{Proofs for section~\ref{sec:W2ConstantRLMC}}
  \begin{lemma}\label{TempLemma}
 Let $x(t)$ be the solution to Langevin dynamics SDE with initial condition $x_0$ and $y(t)$ be the solution to Langevin dynamics SDE with initial condition $y_0$. Then we have the following estimates for Langevin dynamics when $f$ satisfies Assumption~\ref{assumptionsonf} and $Mh<\frac{1}{2}$:
  \begin{equation*}
      \begin{aligned}
      \mb{E}[\sup_{t\in [0,h]}\lv \nabla f(x(t))\rv^2]&\le 4\lv \nabla f(x_0)\rv^2+8M^2dh \\
      \mb{E}[\sup_{t\in [0,h]}\lv x(t)-x_0\rv^2]&\le O(h^2\lv \nabla f(x_0)\rv^2+M^2h^3d+2dh) \\
      \mb{E}[\lv x(t)-y(t)\rv^2]&\le e^{-2mt}\lv x_0-y_0\rv^2 
      \end{aligned}
  \end{equation*}
  \end{lemma}
 \begin{proof} 
  By triangle inequality we have
  \begin{equation*}
      \mb{E}[\sup_{t\in [0,h]}\lv \nabla f(x(t))\rv^2]\le 2 \lv \nabla f(x_0))\rv^2+2M^2\mb{E}[\sup_{t\in [0,h]}\lv x(t)-x_0\rv^2]
  \end{equation*}
  Furthermore, we have 
\begin{equation*}
    \begin{aligned}
    \mb{E}[\sup_{t\in [0,h]}\lv x(t)-x_0\rv^2]&=\mb{E}[\sup_{t\in [0,h]}\lv -\int_0^t\nabla f(x(s))ds+\sqrt{2}W_t\rv^2]\\
    &\le h^2\mb{E}[\sup_{t\in [0,h]}\lv \nabla f(x(t))\rv^2]+2dh
    \end{aligned}
\end{equation*}
Combining the two inequalities and $Mh<\frac{1}{2}$, we can obtain the first two estimates. The last estimate could be easily obtained by energy method. 
\end{proof} 

 \begin{proof}[Proof of Propositon~\ref{W2RLMC}]
  We denote $x_n=x_n(0)$ to be the algorithm iterate points, $y_n$ to be the $n$-th step of Langevin diffusion with $y_0\sim \exp(-f(y))$, $x_{n+1}^*=x_n(h)$ to be one step solution of Langevin dynamics with initial values $x_n$. When $Mh<\frac{1}{2}$, apply lemma~\ref{TempLemma} and we get:
    \begin{equation*}
        \begin{aligned}
        \mb{E}[\sup_{t\in [0,h]}\lv x_{n-1}(\alpha_n h)-x_{n-1}(t)\rv^2]&\le O(h^2\lv \nabla f(x_{n-1})\rv_{L^2}^2+M^2h^3d+2dh) \\
        \mb{E}[\lv \nabla f(x_{n-\frac{1}{2}})-\nabla f(x_{n-1}(\alpha_n h))\rv ^2]&\le M^2\mb{E} \lv \int_0^{\alpha_n h} \nabla f(x_{n-1}(s))-\nabla f(x_{n-1}(0)) ds \rv^2 \\
        &\le M4h^2\mb{E}[\alpha_n^2\sup_{t\in [0,\alpha_n h]}\lv x_{n-1}(t)-x_{n-1}(0)\rv^2] \\ 
        &\le O(M^4h^4\lv \nabla f(x_{n-1})\rv_{L^2}^2+dM^4h^3+dM^6h^5)
        \end{aligned}
    \end{equation*}
    Consider the distance between our iterates and the continuous process:
    \begin{equation*}
        \begin{aligned}
        \mb{E}_{\alpha_K}[\lv x_K-y_K\rv^2]&=\mb{E}_{\alpha_K}[\lv x_K-x_K^*+x_K^*-y_K\rv^2]\\ 
        &\le \lv y_K-x_K^*\rv^2+\mb{E}_{\alpha_K}[\lv x_K-x_K^*\rv^2]-2(y_K-x_K^*)^T(\mb{E}_{\alpha_K}x_K-x_K^*) \\ 
        &\le (1+hm)\lv y_K-x_K^*\rv^2+\frac{1}{hm}\lv \mb{E}_{\alpha_K}x_K-x_K^*\rv^2+\mb{E}_{\alpha_K}[\lv x_K-x_K^*\rv^2]
        \end{aligned}
    \end{equation*}
    Taking expectations over $\{\alpha_k, U_l, U_l'; 1\le k\le K-1, 1\le l\le K\}$, applying lemma~\ref{TempLemma} again and using induction, we have 
    \begin{equation*}
        \begin{aligned}
        \lv x_K-y_K\rv_{L^2}^2 &\le (1+hm)\lv y_K-x_K^*\rv_{L^2}^2+\frac{1}{hm}\mb{E}\lv \mb{E}_{\alpha_K}x_K-x_K^*\rv^2+\lv x_K-x_K^*\rv_{L^2}^2 \\ 
        &\le (1+hm)e^{-2mh}\lv x_{K-1}-y_{K-1}\rv_{L^2}^2+\frac{1}{hm}\mb{E}\lv \mb{E}_{\alpha_K}x_K-x_K^*\rv^2+\lv x_K-x_K^*\rv_{L^2}^2 \\ 
        &\le (1+hm)e^{-2mKh}\lv x_0-y_0\rv_{L^2}^2+\sum_{n=1}^K \frac{1}{hm}\mb{E}\lv \mb{E}_{\alpha_n}x_n-x_n^*\rv^2+\sum_{n=1}^K\lv x_n-x_n^*\rv_{L^2}^2 \\
        &\le e^{-mKh}\lv x_0-y_0\rv_{L^2}^2+A+B
        \end{aligned}
    \end{equation*}
    Next we bound part A and part B. For part A:
    \begin{equation*}
        \begin{aligned}
        \lv \mb{E}_{\alpha_n}x_n-x_n^*\rv^2&=\lv \mb{E}_{\alpha_n}[h\nabla f(x_{n-\frac{1}{2}})]-\int_0^h \nabla f(x_{n-1}(s))ds \rv^2\\
        &\le 2\mb{E}_{\alpha_n}\lv h\nabla f(x_{n-\frac{1}{2}})-h\nabla f(x_{n-1}(\alpha_n h))\rv^2+2\lv \mb{E}_{\alpha_n}[h\nabla f(x_{n-1}(\alpha_n h))]-\int_0^h \nabla f(x_{n-1}(s))ds \rv^2\\
        &\le 2h^2\mb{E}_{\alpha_n}\lv \nabla f(x_{n-\frac{1}{2}})-\nabla f(x_{n-1}(\alpha_n h))\rv^2+0
        \end{aligned}
    \end{equation*}
    Therefore
    \begin{equation*}
        \begin{aligned}
        \mb{E}\lv \mb{E}_{\alpha_n}x_n-x_n^*\rv^2&\le 2h^2\mb{E}[\lv\nabla f(x_{n-\frac{1}{2}})-\nabla f(x_{n-1}(\alpha_n h))\rv^2]\\ 
        &\le O(M^4h^6\lv \nabla f(x_{n-1})\rv_{L^2}^2+dM^4h^5)
        \end{aligned}
    \end{equation*}
    For part B, use our previous estimates:
    \begin{equation*}
        \begin{aligned}
        \lv x_n-x_n^*\rv_{L^2}^2&=\lv h\nabla f(x_{n-\frac{1}{2}})-\int_0^h \nabla f(x_{n-1}(s))ds\rv^2_{L^2}\\ 
        &\le 2\lv h\nabla f(x_{n-\frac{1}{2}})-h\nabla f(x_{n-1}(\alpha_n h))\rv^2_{L^2}+2\lv \int_0^h \nabla f(x_{n-1}(s))-\nabla f(x_{n-1}(\alpha_n h))ds\rv^2_{L^2} \\
        &\le 2h^2\lv \nabla f(x_{n-\frac{1}{2}})-\nabla f(x_{n-1}(\alpha_n h))\rv_{L^2} ^2+2M^2h^2\mb{E}[\sup_{t\in [0,h]}\lv x_{n-1}(\alpha_n h)-x_{n-1}(t)\rv^2] \\
        &\le O(M^2h^4\lv \nabla f(x_{n-1})\rv_{L^2}^2+dM^2h^3)
        \end{aligned}
    \end{equation*}
    Plug the estimates on A and B into the inequality we have 
    \begin{equation*}
        \begin{aligned}
        \lv x_K-y_K\rv_{L^2}^2 &\le e^{-mKh}\lv x_0-y_0\rv_{L^2}^2+O(m^{-1}M^4h^5\sum_{n=0}^{K-1}\lv \nabla f(x_n)\rv^2_{L^2}+dm^{-1}M^4Kh^4)\\
        &\ +O(M^2h^4\sum_{n=0}^{K-1}\lv \nabla f(x_n)\rv^2_{L^2}+dM^2Kh^3)
        \end{aligned}
    \end{equation*}
    Next we need to estimate $\sum_{n=0}^{K-1}\lv \nabla f(x_n)\rv^2_{L^2}$. Since
    \begin{equation*}
        \begin{aligned}
        f(x_{n}(h))&=f(x_n(0))+\int_0^h d f(x_n(t)) \\
        &=f(x_n(0))-\int_0^h |\nabla f(x_n(t))|^2 dt+\sqrt{2}\int_0^h \nabla f(x_n(t)) d W(t)+\int_0^h \Delta f(x_n(t)) dt
        \end{aligned}
    \end{equation*}
    we have
    \begin{equation*}
        \begin{aligned}
        \mb{E}[f(x_{n+1}(0))]-\mb{E}[f(x_n(h))]=\mb{E}[f(x_{n+1}(0))-f(x_n(0))]+\mb{E}[\int_0^h |\nabla f(x_n(t))|^2 dt]-\mb{E}[\int_0^t \Delta f(x_n(t))dt]
        \end{aligned}
    \end{equation*}
    When $Mh<\frac{1}{4}$,
    \begin{equation*}
        \begin{aligned}
        \mb{E}[\inf_{t\in [0,h]}\lv \nabla f(x(t))\rv^2]&\ge \frac{1}{2}\lv \nabla f(x(0))\rv_{L^2}^2-\mb{E}[\sup_{t\in [0,h]}\lv \nabla f(x(t))-\nabla f(x(0))\rv^2]\\
        &\ge \frac{1}{2}\lv \nabla f(x(0))\rv_{L^2}^2-M^2\mb{E}[\sup_{t\in [0,h]}\lv x(t)-x(0)\rv^2] \\
        &\ge \frac{1}{4}\lv \nabla f(x(0))\rv^2_{L^2}+O(dM^2h) \\
        |\Delta f(x_n(t))|&\le d\lv \nabla^2 f(x_n(t))\rv \le Md
        \end{aligned}
    \end{equation*}
    Plug these two estimates into our previous identity and we obtain,
    \begin{equation*}
        \begin{aligned}
        \mb{E}[f(x_{n+1}(0))-f(x_n(h))]&\ge \mb{E}[f(x_{n+1})-f(x_n)]+\frac{h}{4}\lv \nabla f(x_n)\rv_{L^2}^2-dMh+O(dM^2h^2)
        \end{aligned}
    \end{equation*}
    Next we consider that
    \begin{equation*}
        \begin{aligned}
        \mb{E}_{\alpha_{n+1}}[f(x_{n+1}(0))]&\le f(x_n(h))+\nabla f(x_n(h))^T(\mb{E}_{\alpha_{n+1}}[x_{n+1}(0)]-x_n(h))+\frac{M}{2}\mb{E}_{\alpha_{n+1}}[\lv x_{n+1}(0)-x_n(h)\rv^2]\\
        &\le f(x_n(h))+Mh^2\lv \nabla f(x_n(h))\rv^2_{L^2}+M^{-1}h^{-2}\lv \mb{E}_{\alpha_{n+1}}[x_{n+1}(0)]-x_n(h)\rv^2\\
        &\ +\frac{M}{2}\mb{E}_{\alpha_{n+1}}[\lv x_{n+1}(0)-x_n(h)\rv^2]
        \end{aligned}
    \end{equation*}
    where
    \begin{equation*}
        \begin{aligned}
        Mh^2\mb{E}[\lv \nabla f(x_n(h))\rv^2]&\le O(Mh^4\lv \nabla f(x_n)\rv_{L^2}^2+dMh^3) \\
        M^{-1}h^{-2}\lv \mb{E}_{\alpha_{n+1}}[x_{n+1}(0)]-x_n(h)\rv^2&\le O(M^3h^4\lv \nabla f(x_n)\rv^2_{L^2}+dM^3h^3)\\
        \frac{M}{2}\mb{E}[\lv x_{n+1}-x_n(h)\rv^2]&\le O(M^3h^4\lv \nabla f(x_n)\rv^2_{L^2}+dM^3h^3)
        \end{aligned}
    \end{equation*}
    Hence we have
    \begin{equation*}
        \begin{aligned}
        \mb{E}[f(x_{n+1}(0))-f(x_n(h))]\le O(M^3h^4\lv\nabla f(x_n)\rv_{L^2}^2+dM^3h^3)
        \end{aligned}
    \end{equation*}
    and 
    \begin{equation*}
        \begin{aligned}
        O(M^3h^4\lv\nabla f(x_n)\rv^2_{L^2}+dM^3h^3)&\ge \mb{E}[f(x_{n+1})-f(x_n)]+\frac{h}{4}\lv \nabla f(x_n)\rv_{L^2}^2+O(dM^2h^2)-dMh
        \end{aligned}
    \end{equation*}
    sum up over $k$ from $0$ to $K-1$:
    \begin{equation*}
        \begin{aligned}
        O(M^3h^4\sum_{k=0}^{K-1}\lv\nabla f(x_n)\rv^2_{L^2}+dM^3Kh^3)\ge \mb{E}[f(x_K)-f(x_0)]+\frac{h}{4}\sum_{k=1}^{K-1}\lv \nabla f(x_n)\rv_{L^2}^2+O(dM^2Nh^2)-dMKh
        \end{aligned}
    \end{equation*}
    Picking $x_0=\text{argmin} f(x)$, we can ensure $\mb{E}[f(x_K)-f(x_0)]\ge 0$, when $Mh<\frac{1}{2}$, we have
    \begin{equation*}
        \begin{aligned}
        \frac{h}{8}\sum_{k=0}^{K-1}\lv \nabla f(x_n)\rv^2_{L^2}&\le dKMh-O(dKM^2h^2)+O(dKM^3h^3)\\
        \implies \ \ \sum_{k=0}^{K-1}\lv \nabla f(x_n)\rv^2_{L^2}& \le O(dKM)
        \end{aligned}
    \end{equation*}
    Therefore
        \begin{equation*}
        \begin{aligned}
        \lv x_K-y_K\rv_{L^2}^2 &\le e^{-mKh}\lv x_0-y_0\rv_{L^2}^2+O(m^{-1}M^5h^5Kd+m^{-1}M^4h^4Kd)+O(M^3h^4Kd+M^2h^3Kd)\\
        &\le  e^{-mKh}\lv x_0-y_0\rv_{L^2}^2+O(\kappa M^3h^4Kd)+O(M^2h^3Kd)
        \end{aligned}
    \end{equation*}
    Therefore we have 
    \begin{equation*}
        W_2(\nu_K,\pi)^2\le e^{-mKh}\lv x_0-y_0\rv_{L^2}^2+O(M^3h^4Kd)\max\{\kappa, \frac{1}{Mh}\}
    \end{equation*}
    \begin{enumerate}
        \item [a)] When $\kappa>\frac{1}{Mh}$, by choosing $h\sim O(\frac{\epsilon^{2/3}}{\kappa^{1/3}M})$, we can ensure $W_2(\nu_K,\pi)^2\le \epsilon^2d/m$ after $K$ steps when $K\sim \Tilde{O}(\frac{\kappa^{4/3}}{\epsilon^{2/3}})$.
        \item [b)] When $\kappa\le \frac{1}{Mh}$, by choosing $h\sim O(\frac{\epsilon}{M})$, we can ensure $W_2(\nu_K,\pi)^2\le \epsilon^2d/m$ after $K$ steps when $K\sim \Tilde{O}(\frac{\kappa}{\epsilon})$. 
 \end{enumerate}
\end{proof}

\subsection{Proofs for Section~\ref{sec:RLMCChainAnalyis}}
    \begin{proof}[Proof of Theorem~\ref{ergodicityRLMC}]
    
        Under the assumption~\ref{ConditionV}, we can show that the following Lyapunov condition is satisfied for small $h$.\\
    \textbf{(Lyapunov Condition):} There exists a function $V:\ \mb{R}^d\to [1,\infty)$ such that:
    \begin{enumerate}
        \item [0)] $\lim_{|x|\to \infty} V(x)=+\infty$,
        \item [1)] There exists $\hat{\alpha}\in(0,1)$ and $\hat{\beta}\ge 0$: $\mb{E}[V(x_{n+1})|\mathcal{F}_n]\le \hat{\alpha} V(x_n)+\hat{\beta}$.
    \end{enumerate}
    \textbf{Proof:} To show that assumption~\ref{ConditionV} implies Lyapunov condition, we first do Taylor expansion of $V(x_{n+1})$ at $x_n$:
    \begin{equation*}
        \begin{aligned}
        V(x_{n+1})&=V(x_n)-h\langle \nabla V(x_n), \nabla f(x_n)\rangle +\alpha_{n+1} h^2 \langle D^2 f(x_n); \nabla f(x_n),\nabla V(x_n)\rangle \\
        &\ -\sqrt{2\alpha_{n+1}}h^{\frac{3}{2}}\langle D^2 f(x_n); \nabla V(x_n), U_{n+1}'\rangle +\sqrt{2h}\nabla V(x_n)\cdot U_{n+1}\\
        &\ +\frac{1}{2}D^2 V(\theta_n) (-h\nabla f(x_n)+\alpha_{n+1} h^2 D^2 f(x_n)\nabla f(x_n)-\sqrt{2\alpha_{n+1}}h^{\frac{3}{2}}U_{n+1}'+\sqrt{2h}U_{n+1})^{\otimes 2}
        \end{aligned}
    \end{equation*}
    where $\theta_n$ is a random point on the line segment joining $x_n$ and $x_{n+1}$. Using the fact that $f$ is $M$-gradient Lipschitz, we have:
    \begin{equation*}
        \begin{aligned}
        \mb{E}[V(x_{n+1})|\mathcal{F}_n] &\le V(x_n)-h\langle \nabla V(x_n),\nabla f(x_n)\rangle +\frac{1}{4}M h^2 (|\nabla f(x_n)|^2+|\nabla V(x_n)|^2)\\
        &\ +2\lv D^2V\rv_{\infty}(h^2|\nabla f(x_n)|^2+\frac{1}{3}M^2h^4|\nabla f(x_n)|^2+h^3d+2hd)\\
        &\ \le (1-\alpha h+\frac{1}{4}Mh^2c_V+2\lv D^2V\rv_{\infty}h^2 c_V+\frac{2}{3}c_V\lv D^2V\rv_{\infty}M^2h^4c_V)V(x_n)\\
        &\ +\beta h+2d\lv D^2V\rv_{\infty} h^3+4d\lv D^2V\rv_{\infty} h\\
        &\le \hat{\alpha} V(x_n)+\hat{\beta}
        \end{aligned}
    \end{equation*}
    for some $\hat{\alpha}\in(0,1)$ and $\hat{\beta}\ge 0$ when $h$ is small. \qed \\
    
    Once we have the Lyapunov condition, we can define the stopping time $\tau_C=\inf \{n>0: x_n\in C\}$ and show that $\sup_{x\in C}\mb{E}_x[\tau_C]\le M_C<\infty$ for all small set C. Then uniqueness of stationary probability measure and ergodicity all follow by Theorem 1.3.1 in~\cite{meyn2012markov}. Next we prove that $\sup_{x\in C}\mb{E}_x[\tau_C]\le M_C<\infty$ given Lyapunov condition. To do so,  note that we have
    \begin{equation*}
        \begin{aligned}
        \mb{E}_x[\tau_C]&=\sum_{k=1}^{\infty} k\mb{P}(\tau_C=k)\\
        &=\sum_{k\ge 1}\mb{P}(\tau_C>k-1)
        \end{aligned}
    \end{equation*}
    Under Lyapunov condition, for any stopping time $N$, according to Lemma A.3 and Corollary A.4 in~\cite{mattingly2002ergodicity}, we have 
    \begin{equation*}
        \begin{aligned}
        \mb{P}(\tau_C>k-1)&\le \mb{E}[V(x_n)1_{\tau_C>k-1}]\\
        &\le \frac{\kappa [\gamma^{k-1}V(x_0)+1]}{1-\gamma} \\
        &\le \kappa \gamma^{n-1}[V(x_0)+1]
        \end{aligned}
    \end{equation*}
    for some $\gamma\in (\hat{\alpha},1)$ and constant $\kappa$. Therefore we have
    \begin{equation*}
        \begin{aligned}
        \mb{E}_x[\tau_C]&\le \sum_{k\ge 1}\kappa \gamma^{n-1}[V(x_0)+1]\\
        &=\frac{\kappa[V(x)+1]}{1-\gamma}
        \end{aligned}
    \end{equation*}
    and 
    \[
    \sup_{x\in C}\mb{E}_x[\tau_C]\le \frac{\kappa}{1-\gamma}\sup_{x\in C}V(x)+\frac{\kappa}{1-\gamma}\le M_C<\infty
    \] 
    So as a conclusion, the statement of the theorem follows.
\end{proof}
\begin{proof}[Proof of Proposition~\ref{biasRLMC}]
Consider that $x_n\sim \pi_h$ and $x_n^*\sim \pi$ are two independent random variables. Define $x_{n+1}$ to be the one step RLMC result starting from $x_n$ and $x_n^*(h)$ to be the solution of Langevin dynamics with initial value $x_n^*$. Therefore, $x_{n+1}\sim \pi_h$ and $x_n^*(h)\sim \pi$ are also independent and $\lv x_n^*-x_n\rv_{L^2}=\lv x_n^*(h)-x_{n+1}\rv_{L^2}$. We can compute the diffenrence between $x_{n+1}$ and $x_n^*(h)$:
    \begin{equation*}
        \begin{aligned}
        x_n^*(h)-x_{n+1}&=(x_n^*-x_n)-\int_0^h \nabla f(x_n^*(s))ds+h\nabla f(x_n^*(\alpha_{n+1} h))-h(-\nabla f(x_{n+\frac{1}{2}})+\nabla f(x_n^*(\alpha_{n+1} h)))
        \end{aligned}
    \end{equation*}
    It's easy to see that $\mb{E}_{\alpha_{n+1}}[\int_0^h \nabla f(x_n^*(s))ds-h\nabla f(x_n^*(\alpha_{n+1} h))]=0$. And we can rewrite the last term as 
   \begin{equation*}
       \begin{aligned}
           h(-\nabla f(x_{n+\frac{1}{2}})+\nabla f(x_n^*(\alpha_{n+1} h)))&=h(\nabla f(x_{n+\frac{1}{2}}+x_n^*-x_n)-\nabla f(x_{n+\frac{1}{2}}))\\
           &\ +h\nabla f(x_n^*-\int_0^{\alpha_{n+1} h}\nabla f(x_n^*(s))ds+\sqrt{2}W_{\alpha_{n+1} h})\\
           &\ -h\nabla f(x_n^*-\alpha_{n+1} h \nabla f(x_n)+\sqrt{2\alpha_{n+1} h}U_{n+1}')
       \end{aligned}
   \end{equation*}
   Take $L_2$-norm on other randomness, we have
      \begin{equation*}
       \begin{aligned}
       &\lv x_n^*(h)-x_{n+1}\rv_{L^2} \le \lv (x_n^*-x_n)-h(\nabla f(x_{n+\frac{1}{2}}+x_n^*-x_n)-\nabla f(x_{n+\frac{1}{2}}))\rv_{L^2} \\
       &~~+h\lv \nabla f(x_n^*-\int_0^{\alpha_{n+1} h}\nabla f(x_n^*(s))ds+\sqrt{2}W_{\alpha_{n+1} h})-\nabla f(x_n^*-\alpha_{n+1} h \nabla f(x_n)+\sqrt{2\alpha_{n+1} h}U_{n+1}')\rv_{L^2} \\
       &~~+ \lv \int_0^h \nabla f(x_n^*(s))ds-h\nabla f(x_n^*(\alpha_{n+1} h))\rv_{L^2}
       \end{aligned}
   \end{equation*}
   Since $f$ is twice differentiable and $f$ is also $M$-gradient Lipschitz and strongly convex with parameter $m$, 
   \[
   \lv (x_n^*-x_n)-h(\nabla f(x_{n+\frac{1}{2}}+x_n^*-x_n)-\nabla f(x_{n+\frac{1}{2}}))\rv_{L^2}\le \rho \lv x_n^*-x_n\rv_{L^2}
   \]
   where $\rho=\max (1-mh, Mh-1)=1-mh$.\\
   For the second term:
   \begin{equation*}
       \begin{aligned}
       &h\lv \nabla f(x_n^*-\int_0^{\alpha_{n+1} h}\nabla f(x_n^*(s))ds+\sqrt{2}W_{\alpha_{n+1} h})-\nabla f(x_n^*-\alpha_{n+1} h \nabla f(x_n)+\sqrt{2\alpha_{n+1} h}U_{n+1}')\rv_{L^2}\\
       &\le Mh \lv \int_0^{\alpha_{n+1} h} \nabla f(x_n^*(s))-\nabla f(x_n)ds\rv_{L^2} \\
       &\le \frac{\sqrt{3}}{3}M^2h^2 \lv x_n^*-x_n\rv_{L^2}+\frac{\sqrt{3}}{3}M^2h^2\sup_{0<s<h}\lv x_n^*(s)-x_n^*\rv_{L^2} \\
       &\le \frac{\sqrt{3}}{3}M^2h^2 \lv x_n^*-x_n\rv_{L^2}+\frac{\sqrt{3}}{3}M^2h^2 (4h^2\lv \nabla f(x_n^*)\rv^2+8M^2dh^3+2dh)^{\frac{1}{2}}\\
       &\le \frac{\sqrt{3}}{3}M^2h^2 \lv x_n^*-x_n\rv_{L^2}+\frac{\sqrt{3}}{3}M^2h^2(2dh+4Mdh^2+8M^2dh^3)^{\frac{1}{2}}
       \end{aligned}
   \end{equation*}
   For the third term:
   \begin{equation*}
       \begin{aligned}
          \lv \int_0^h \nabla f(x_n^*(s))ds-h\nabla f(x_n^*(\alpha_{n+1} h))\rv_{L^2}&=\{\mb{E}\mb{E}_{\alpha_{n+1}}[(\int_0^h \nabla f(x_n^*(s))ds-h\nabla f(x_n^*(\alpha_{n+1} h)))^2]\}^{\frac{1}{2}}\\
          &=\{\mb{E}[ h\int_0^h |\nabla f(x_n^*(s))|^2 ds-(\int_0^h \nabla f(x_n^*(s)) ds )]\}^{\frac{1}{2}}\\
          &= \{\mb{E}[h\int_0^h|\nabla f(x_n^*(s))-\frac{1}{h}\int_0^h \nabla f(x_n^*(s'))ds'|^2 ds]\}^{1/2} \\
          &\le  \{\mb{E}[h\int_0^h \frac{1}{h}\int_0^h \lv \nabla f(x_n^*(s))-\nabla f(x_n^*(s'))\rv^2 ds' ds]\}^{1/2} \\
          &\le 2Mh\{\sup_{s\in (0,h)}\lv x_n^*(s)-x_n^* \rv^2\}^{1/2}\\
          &\le 2Mh(4h^2\lv \nabla f(x_n^*)\rv^2+8M^2dh^3+2dh)^{1/2}\\
          &\le 2Mh(2dh+4Mdh^2+8M^2dh^3)^{\frac{1}{2}}
       \end{aligned}
   \end{equation*}
   Combine all the bounds:
   \begin{equation*}
       \begin{aligned}
       \lv x_n^*-x_n\rv_{L^2}&\le \frac{\frac{\sqrt{3}}{3}M^2h^2(2dh+4Mdh^2+8M^2dh^3)^{\frac{1}{2}}+ 2Mh(2dh+4Mdh^2+8M^2dh^3)^{\frac{1}{2}}}{mh-\frac{\sqrt{3}}{3}M^2h^2}
       \end{aligned}
   \end{equation*}
   The final statement follows by the fact that $W_2(\pi,\pi_h)\le \lv x_n-x_n^*\rv_{L^2}$.  
\end{proof}   
     
\subsection{Proofs for Section~\ref{sec:RLMCDecreasingStepsize}}       
   
   \begin{proof}[Proof of Proposition~\ref{W2RLMC-fast}]   
    From previous analysis, if we keep track of the coefficients in all those bounds and assume that $M\gamma_n\le \frac{1}{2}$ for all $n$, we have:
    \begin{equation*}
        \begin{aligned}
      &~~~~  \mb{E}[\lv x_{n+1}-y_{n+1} \rv^2]\\
        &\le (1+m\gamma_{n+1})\mb{E}[\lv y_{n+1}-x_{n+1}^* \rv^2]+\frac{1}{m\gamma_{n+1}}\mb{E}\lv \mb{E}_{\alpha_{n+1}}x_{n+1}-x_{n+1}^* \rv^2+\mb{E}\lv x_{n+1}-x_{n+1}^* \rv^2\\
        &\le (1+m\gamma_{n+1})e^{-2m\gamma_{n+1}}\mb{E}\lv x_n-y_n \rv^2+\frac{1}{m\gamma_{n+1}}\mb{E}\lv \mb{E}_{\alpha_{n+1}}x_{n+1}-x_{n+1}^* \rv^2+\mb{E}\lv x_{n+1}-x_{n+1}^* \rv^2\\
        &\le (1+m\gamma_{n+1})e^{-2m\gamma_{n+1}}\mb{E}\lv x_n-y_n \rv^2+\frac{2\gamma_{n+1}^2}{m\gamma_{n+1}}\mb{E}\lv \nabla f(x_{n+\frac{1}{2}})-\nabla f(x_n(\alpha_{n+1} \gamma_{n+1})) \rv^2\\
        &\ +2\gamma_{n+1}^2\mb{E}\lv \nabla f(x_{n+\frac{1}{2}})-\nabla f(x_n(\alpha_{n+1} \gamma_{n+1})) \rv^2+2M^2\gamma_{n+1}^2\mb{E}\sup_{t\in [0,\gamma_{n+1}]}\lv x_n(\alpha_{n+1} \gamma_{n+1})-x_n(t)\rv^2\\
        &\le (1+m\gamma_{n+1})e^{-2m\gamma_{n+1}}\mb{E}\lv x_n-y_n \rv^2\\
        &\ +2\gamma_{n+1}^4(1+\frac{1}{m\gamma_{n+1}})M^4(\frac{1}{5}\gamma_{n+1}^2\lv \nabla f(x_n)\rv_{L^2}^2+\frac{1}{6}M^2d\gamma_{n+1}^3+\frac{2}{3}d\gamma_{n+1})\\
        &\ +4M^2\gamma_{n+1}^2\left(\frac{4\gamma_{n+1}^2}{1-2M^2\gamma_{n+1}}\lv \nabla f(x_n) \rv_{L^2}^2+\frac{8M^2d\gamma_{n+1}^3}{1-2M^2\gamma_{n+1}^2}+4d\gamma_{n+1}\right)\\
        &\le (1+m\gamma_{n+1})e^{-2m\gamma_{n+1}}\mb{E}\lv x_n-y_n \rv^2+(33+\kappa)M^2\gamma_{n+1}^4\lv \nabla f(x_n) \rv_{L^2}^2+(33+\kappa)M^2d\gamma_{n+1}^3\\
        \end{aligned}
    \end{equation*}
    We can further bound $\lv \nabla f(x_n) \rv_{L^2}^2$:
    \begin{equation*}
        \begin{aligned}
        \lv \nabla f(x_n) \rv_{L^2}^2&\le 2\lv \nabla f(y_n) \rv_{L^2}^2+2\lv \nabla f(y_n)-\nabla f(x_n) \rv_{L^2}^2\\
        &\le 2\lv \nabla f(y_n) \rv_{L^2}^2+2M^2\lv x_n-y_n \rv_{L^2}^2\\
        &\le 2Md+2M^2\lv x_n-y_n \rv_{L^2}^2
        \end{aligned}
    \end{equation*}
    Therefore we have the following iterative inequality:
    \begin{equation*}
        \begin{aligned}
        \mb{E}[\lv x_{n+1}-y_{n+1} \rv^2]&\le (1+m\gamma_{n+1})e^{-2m\gamma_{n+1}}\mb{E}\lv x_n-y_n \rv^2+2(33+\kappa)M^2d\gamma_{n+1}^3+2(33+\kappa)M^4\gamma_{n+1}^4\mb{E}\lv x_n-y_n \rv^2\\
        &\le \left[1-m\gamma_{n+1}+(\frac{m^2}{2}+\frac{M^2(33+\kappa)}{2})\gamma_{n+1}^2\right]\mb{E}\lv x_n-y_n \rv^2+2(33+\kappa)M^2d\gamma_{n+1}^3
        \end{aligned}
    \end{equation*}
    Since $(\gamma_n)$ is fast decreasing, we can assume that $\gamma_{n+1}\le \frac{m}{m^2+M^2(33+\kappa)}\le \frac{1}{m+34M}$ for large $n$, and for those $n$ we have
    \begin{equation*}
        \mb{E}[\lv x_{n+1}-y_{n+1} \rv^2]\le (1-\frac{1}{2}m\gamma_{n+1})\mb{E}\lv x_n-y_n \rv^2+2(33+\kappa)M^2d\gamma_{n+1}^3
    \end{equation*}
    Our strategy of choosing $(\gamma_n)$: for the first $K_1$ steps, we choose constant step size $h=\frac{1}{m+34M}$, $K_1$ is the first time so that $\mb{E}[\lv x_{K_1}-y_{K_1} \rv^2]\le 5\kappa(\kappa+33)M(\frac{d^{\frac{1}{2}}}{m+34M})^2$. such $K_1$ exists because 
    \begin{equation*}
        \begin{aligned}
        \mb{E}[\lv x_{K_1}-y_{K_1} \rv^2]&\le (1-\frac{m}{2m+68M})^{K_1}\mb{E}[\lv x_0-y_0 \rv^2] +\frac{2M^2(\kappa+33)d}{(m+34M)^3}\frac{2(m+34M)}{m}\\
        &=(1-\frac{m}{2m+68M})^{K_1}\mb{E}[\lv x_0-y_0 \rv^2] +4\kappa(\kappa+33)M(\frac{d^{\frac{1}{2}}}{m+34M})^2
        \end{aligned}
    \end{equation*}
    \textbf{Claim:} There exists $\lambda>0$ such that if we choose $\gamma_{n+1}=\frac{1}{m+34M+\lambda (n-K_1)}$ for all $n\ge K_1$, we can ensure that $\mb{E}[\lv x_{k}-y_{k} \rv^2]\le 5\kappa(\kappa+33)M(\frac{d^{\frac{1}{2}}}{m+34M+\lambda (n-K_1)})^2$ for all $n\ge K_1$.\\
    \textbf{Proof of Claim:} Simply use induction:
    \begin{equation*}
        \begin{aligned}
        \mb{E}[\lv x_{n+1}-y_{n+1} \rv^2]&\le (1-\frac{1}{2}m\gamma_{n+1})5\kappa(\kappa+33)Md\gamma_{n+1}^2+2M^2(\kappa+33)d\gamma_{n+1}^3\\
        &=5\kappa(\kappa+33)Md\gamma_{n+1}^2(1-\frac{m}{10}\gamma_{n+1})
        \end{aligned}
    \end{equation*}
    Our goal is to ensure $5\kappa(\kappa+33)Md\gamma_{n+1}^2(1-\frac{m}{10}\gamma_{n+1})<5\kappa(\kappa+33)M(\frac{d^{\frac{1}{2}}}{m+34M+\lambda (n+1-K_1)})^2$. It boils down to discuss the following polynomial inequality relates to $\lambda$:
    \[
    G(\lambda)=(K-\frac{1}{10}m(K+1)^2)\lambda^2+(X-\frac{1}{5}mX(K+1))\lambda-\frac{1}{10}mX^2\le 0
    \]
    where $X=m+34M$ and $K=n-K_1>0$. It's not hard to see that there's always positive $\lambda$ satisfying the inequality.\\
    At last to get small error, we require $\mb{E}\lv x_n-y_n\rv^2\le\frac{d\epsilon^2}{m}$, i.e
    \[ 
    5\kappa(\kappa+33)M\frac{d}{(m+34M+\lambda (n-K_1))^2}\le \frac{d\epsilon^2}{m}
    \]
    Then we have 
    \[
    n\ge K_1+\lambda^{-1} m^{\frac{1}{2}}M^{\frac{1}{2}}\kappa^{\frac{1}{2}}(\kappa+33)^{\frac{1}{2}}/\epsilon-\lambda^{-1}(m+34M)\sim O(\kappa^{\frac{3}{2}}/\epsilon)
    \]
    \end{proof}

\subsubsection{Proof of Theorem~\ref{RLMCCLT}}
Before we prove Theorem~\ref{RLMCCLT}, we need several intermediate results on the tightness of the ~\eqref{RLMCdisc} chain.

 \begin{lemma}\label{WCRLMC}
Under assumption~\ref{ConditionV}, for every continuous function $\varphi$ satisfying $\varphi(x)=o(V^k(x))$ for some $k\in \mb{N}$, $\lim_n \pi_n^{\gamma}(\varphi)=\pi(\varphi)$.
\end{lemma}
\begin{proof}[Proof of Lemma~\ref{WCRLMC}]
 The proof is divided into three steps:
 \begin{enumerate}[leftmargin=18pt,noitemsep]
\item [1)] For all $p\ge 1$, there exists $\tilde{\alpha}\in (0,1)$ and $\tilde{\beta}, n_0\in \mb{N}$ such that $\mb{E}[V^p(x_{n+1})|\mathcal{F}_n]\le V^p(x_n)+\gamma_{n+1}V^{p-1}(x_n)(\tilde{\beta}-\tilde{\alpha}V(x_n))$ for all $n\ge n_0$.\\
When $p=1$, the statement follows from assumption~\ref{ConditionV}.\\
When $p>1$, first we Taylor expand $V^p(x_{n+1})$ at $x_n$:
\begin{equation*}
    \begin{aligned}
    V^p(x_{n+1})&=V^p(x_n)+pV^{p-1}(x_n)\nabla V(x_n)\cdot(x_{n+1}-x_n)+\frac{1}{2}D^2(V^p)(\xi_{n+1})(x_{n+1}-x_n)^{\otimes 2}\\
    &=V^p(x_n)-\gamma_{n+1}pV^{p-1}(x_n)\nabla V(x_n)\cdot \nabla f(x_{n+\frac{1}{2}})+\sqrt{2\gamma_{n+1}}pV^{p-1}\nabla V(x_n)\cdot U_{n+1}\\
    &\ \ +\frac{1}{2}D^2(V^p)(\xi_{n+1})\left( -\gamma_{n+1}\nabla f(x_{n+\frac{1}{2}})+\sqrt{2\gamma_{n+1}}U_{n+1}\right)^{\otimes 2}\\
    &\le V^p(x_n)-\gamma_{n+1}pV^{p-1}(x_n)\nabla V(x_n)\cdot \nabla f(x_{n+\frac{1}{2}})+\sqrt{2\gamma_{n+1}}pV^{p-1}(x_n)\nabla V(x_n)\cdot U_{n+1}\\
    &\ \ +p\lambda_pV^{p-1}(\xi_{n+1})| -\gamma_{n+1}\nabla f(x_{n+\frac{1}{2}})+\sqrt{2\gamma_{n+1}}U_{n+1}|^2
    \end{aligned}
\end{equation*}
where $\xi_{n+1}$ is a point on the line segment joining $x_n$ and $x_{n+1}$ and $\lambda_p:=\frac{1}{2}\lambda_{D^2V+(p-1)(\nabla V \otimes \nabla V)/V}<+\infty$.\\
\\Due to $\nabla (\sqrt{V})=\frac{\nabla V}{2 V}$ and $|\nabla V|^2\le c_V V$, we have $\sqrt{V}$ is Lipschitz continuous and the Lipschitz constant $[\sqrt{V}]_1=\frac{1}{4}c_V<+\infty$. Hence for a point $\xi_{n+1}$ on the line segment between $x_n$ and $x_{n+1}$,
\begin{equation*}
    \begin{aligned}
    V^{p-1}(\xi_{n+1})&=(\sqrt{V})^{2(p-1)}(\xi_{n+1})\le \left(\sqrt{V}(x_n)+[\sqrt{V}]_1|x_{n+1}-x_n|\right)^{2(p-1)} \\
    &\le  \left\{
    \begin{aligned}
    &V^{p-1}(x_n)+[\sqrt{V}]_1^{2(p-1)}|x_{n+1}-x_n|^{2(p-1)},\ \ \ \ \ \ \ \ \ \ \ \ \ \ 2(p-1)\le 1\\
    &V^{p-1}(x_n)+c\left(V^{(2p-3)/2}(x_n)|x_{n+1}-x_n|+|x_{n+1}-x_n|^{2(p-1)}\right),\ \ \ \ 2(p-1)>1
    \end{aligned}
    \right.
    \end{aligned}
\end{equation*}
We can further bound 
\begin{equation*}
    \begin{aligned}
    |x_{n+1}-x_n|&=|-\gamma_{n+1}\nabla f(x_n)+\sqrt{2\gamma_{n+1}}U_{n+1}-\gamma_{n+1}\left(\nabla f(x_{n+\frac{1}{2}})-\nabla f(x_n)\right)|\\
    &\le \gamma_{n+1}|\nabla f(x_n)|+\sqrt{2\gamma_{n+1}}|U_{n+1}|+M\gamma_{n+1}|-\tilde{\gamma}_{n+1}\nabla f(x_n)+\sqrt{2\alpha_{n+1}{\gamma}_{n+1}}U_{n+1}'|\\
    &\le \gamma_{n+1}(1+M\alpha_{n+1}{\gamma}_{n+1})|\nabla f(x_n)|+\sqrt{2\gamma_{n+1}}|U_{n+1}|+\sqrt{2}M\gamma_{n+1}\alpha_{n+1}^{\frac{1}{2}}{\gamma}_{n+1}^{\frac{1}{2}}|U_{n+1}'|\\
    & \le C\sqrt{V}(x_n) \gamma_{n+1}^{\frac{1}{2}}(1+|U_{n+1}|+|U_{n+1}'|) 
    \end{aligned}
\end{equation*}
Plug these results into the last term in the first inequality we obtained from Taylor expansion:
\begin{equation*}
    \begin{aligned}
    p\lambda_pV^{p-1}(\xi_{n+1})| x_{n+1}-x_n|^2&\le p\lambda_pV^{p-1}(x_n)| x_{n+1}-x_n|^2 \\
    &\ \ +C p\lambda_p \left\{
    \begin{aligned}
    & | x_{n+1}-x_n|^{2p},\ \ \ \ \ \ \ \ 2p\le 3 \\
    &V^{(2p-3)/2}(x_n)| x_{n+1}-x_n|^3+| x_{n+1}-x_n|^{2p},\ \ \ 2p>3
    \end{aligned}
    \right.\\
    &\le p\lambda_p V^{p-1}(x_n)|x_{n+1}-x_n|^2+C\gamma_{n+1}^{p\wedge \frac{3}{2}} V^p(x_n)(1+|U_{n+1}|^{2p}+|U_{n+1}'|^{2p})
    \end{aligned}
\end{equation*}
We then take conditional expectation, there exists $\alpha>0$ and $\beta\ge 0$ such that for all $n\ge n_0$:
\begin{equation*}
    \begin{aligned}
    \mb{E}[V^p(x_{n+1})|\mathcal{F}_n]&\le V^p(x_n)-pV^{p-1}(x_n)(\alpha V(x_n)-\beta) \\
    &\ \ -p\gamma_{n+1}V^{p-1}(x_n)\mb{E}[\nabla V(x_n)\cdot \left(\nabla f(x_{n+\frac{1}{2}})-\nabla f(x_n)\right)|\mathcal{F}_n]\\
    &\ \ +2p\lambda_pV^{p-1}(x_n)\mb{E}[\gamma_{n+1}^2|\nabla f(x_{n+\frac{1}{2}})|^2+2\gamma_{n+1}|U_{n+1}|^2|\mathcal{F}_n]\\
    &\ \ +CV^p(x_n)(1+\mb{E}|U_{n+1}|^{2p}+\mb{E}|U_{n+1}'|^{2p})\gamma_{n+1}^{p\wedge \frac{3}{2}}\\
    &\le  V^p(x_n)-pV^{p-1}(x_n)(\alpha V(x_n)-\beta)+2p\lambda_p\mb{E}|U_{n+1}|^2\gamma_{n+1} V^{p-1}(x_n)\\
    &\ \ +CV^p(x_n)(1+\mb{E}|U_{n+1}|^{2p}+\mb{E}|U_{n+1}'|^{2p})\gamma_{n+1}^{p\wedge \frac{3}{2}}\\
    &\ \ + c_V M p\gamma_{n+1}^2V^p(X_n)+\sqrt{2} c_V M p\gamma_{n+1}^{\frac{3}{2}}\mb{E}|U_{n+1}'|V^{p-1/2}(x_n)\\
    &\ \ +c_V p\lambda_p \gamma_{n+1}^2 V^{p-1}(x_n) \mb{E}[V(x_{n+\frac{1}{2}})|\mathcal{F}_n]
    \end{aligned}
\end{equation*}
From $x_n$ to $x_{n+\frac{1}{2}}$, it's simply the Euler discretization with time step $\alpha_{n+1}{\gamma}_{n+1}$, we could use the result in~\cite{lamberton2002recursive}: there exists a $\bar{\alpha}>0$ and $\bar{\beta}\in \mb{R}$ such that for all $n\ge n_0$:
\[
\mb{E}[V(x_{n+\frac{1}{2}})|\mathcal{F}_n]\le V(x_n)(1-\bar{\alpha}\tilde{\gamma}_{n+1})+\bar{\beta}\tilde{\gamma}_{n+1}
\]
Therefore we have 
\[
\mb{E}[V^{x_{n+1}}|\mathcal{F}_n]\le (1-\alpha p \gamma_{n+1}+o(\gamma_{n+1}))V^p(x_n)+\gamma_{n+1}V^{p-1}(x_n)(p\beta+2p\lambda_p\mb{E}|U_{n+1}|^2+c_VMp\mb{E}|U_{n+1}'|^2)
\]
There exists $\hat{\alpha}>0$ and $\hat{\beta}\in \mb{R}$ such that for all $n\ge n_0$:
\[
\mb{E}[V^p(x_{n+1})|\mathcal{F}_n]\le V^p(x_n)+\gamma_{n+1}V^{p-1}(x_n)\left(\hat{\beta}-\hat{\alpha}V(x_n)\right)
\]
\item [2)] From step $1)$, we derive
\[
\sup_{n\ge n_0}\mb{E}[V^p(x_n)]\le (\frac{\hat{\beta}}{\hat{\alpha}})^p\vee \mb{E}[V^p(x_{n_0})]
\]
Hence $\sup_n \mb{E}[V^p(x_n)]<+\infty $ for all $p\ge 1$. Therefore $\sup_n \pi_n^{\gamma}(\omega, V^p)<+\infty\ \ \mb{P}-a.s$ for all $p\ge 1$. 
\item [3)] Identification of the weak limit: To identify the limit, we essentially follow the same steps in~\cite{lamberton2002recursive} and hence we omit the proof. 
\begin{enumerate}
\item \textbf{(Echeverrr\'{i}a-Weiss Theorem)} Let $E$ be a locally compact Polish space and $A$ a linear operator satisfying the positive maximum principle. Assume that its domain $\mathcal{D}(A)$ is an algebra everywhere dense in $(\mathcal{C}_0(E),\lv\ \rv_{\infty})$ containing a sequence $(f_n)_{n\in\mb{N}}$ satisfying 
\[
\sup_{n\in\mb{N}}\left(\lv f_n \rv_{\infty}+\lv \mathcal{L}f_n \rv_{\infty} \right)<+\infty,\ \ \forall x\in E,\ \ f_n(x)\to 1\ \ \text{and}\ \ Af_n(x)\to 0.
\]
If a distribution on $(E,\mathcal{B}(E))$ satisfies $\int_E Af d\nu =0$ for every $f\in \mathcal{D}(A)$, then there exists a stationary solution for the martingale problem $(A, \nu)$ (this means that there exists a stationary continuous-time homogeneous Markov process with infinitesimal generator $A$ and invariant distribution $\nu$). 
\item The generator of the Langevin dynamics, $\mathcal{A}$, satisfies the assumptions of the Echeverrr\'{i}a-Weiss theorem.
\item Under assumption~\ref{ConditionV}, for every bounded Lipschitz continuous function $\varphi: \mb{R}^d\to \mb{R}$, $\lim_n \frac{1}{\Gamma_n}\sum_{k=1}^n \mb{E}[\varphi(x_k)-\varphi(x_{k-1})|\mathcal{F}_{k-1}]=0\ \ \mb{P}-a.s$.
\item Under assumption~\ref{ConditionV}, for every twice continuously differentiable function $\varphi$ with compact support, $\lim_n \left(\frac{1}{\Gamma_n}\sum_{k=1}^n \mb{E}[\varphi(x_k)-\varphi(x_{k-1})|\mathcal{F}_{k-1}]-\pi_n^\gamma(\mathcal{A}\varphi) \right)=0\ \ \mb{P}-a.s$.
\end{enumerate}
$a),b),c),d)$ together imply that the weak limit of the empirical distribution $\pi_n^\gamma$ is $\pi$, i.e the stationary distribution of the Langevin dynamics. 
\end{enumerate}
\end{proof}
\begin{proof}[Proof of Theorem~\ref{RLMCCLT}]    
    Since $f$ satisfies assumption~\ref{assumptionsonf}, we can show that the Langevin dynamics satisfies assumption~\ref{ConditionV}. Therefore lemma~\ref{WCRLMC} is true. Then we may use the following method to discuss the CLT of ~\eqref{RLMCdisc}.
    \begin{equation*}
        \begin{aligned}
        x_k-x_{k-1}&=-\gamma_k\left(\nabla f(x_{k-1})+D^2f(x_{k-1})(x_{k-\frac{1}{2}}-x_{k-1})+r_2(x_{k-\frac{1}{2}},x_{k-1})\right)+\sqrt{2\gamma_k}U_k \\
        &=-\gamma_k\nabla f(x_{k-1})+\sqrt{2\gamma_k}U_k-\gamma_kD^2 f(x_{k-1})(x_{k-\frac{1}{2}}-x_{k-1})-\gamma_kr_2(x_{k-\frac{1}{2}},x_{k-1})\\
        &=-\gamma_k\nabla f(x_{k-1})+\sqrt{2\gamma_k}U_k+\alpha_k \gamma_k^2D^2f(x_{k-1})\nabla f(x_{k-1})-\sqrt{2\alpha_k}\gamma_k^{\frac{3}{2}}\nabla ^2f(x_{k-1})U_k'-\gamma_kr_2(x_{k-\frac{1}{2}},x_{k-1})
        \end{aligned}
    \end{equation*}
    where 
    \begin{equation*}
        \begin{aligned}
        r_2(x_{k-\frac{1}{2}},x_{k-1})&=\nabla f(x_{k-\frac{1}{2}})-\nabla f(x_{k-1})-D^2 f(x_{k-1})(x_{k-1}-x_{k-\frac{1}{2}})\\
        &=\frac{1}{2}D^3f(x_{k-1})(x_{k-\frac{1}{2}}-x_{k-1})^{\otimes2}+\frac{1}{6}D^4f(x_{k-1})(x_{k-\frac{1}{2}}-x_{k-1})^{\otimes3}+O(\gamma_k^2)\\
        &=\alpha_k \gamma_kD^3f(x_{k-1})U_k'^{\otimes2}-\sqrt{2}\alpha_k^{\frac{3}{2}}\gamma_k^{\frac{3}{2}}\langle D^3f(x_{k-1}); \nabla f(x_{k-1}), U_k'\rangle \\
        &\ +\frac{\sqrt{2}}{3}\alpha_k^{\frac{3}{2}}D^4f(x_{k-1})U_k'^{\otimes4}+O(\gamma_k^2)
        \end{aligned}
    \end{equation*}
    Then
        \begin{equation*}
        \begin{aligned}
        x_k-x_{k-1}&=-\gamma_k\nabla f(x_{k-1})+\sqrt{2\gamma_k}U_k-\sqrt{2\alpha_k}\gamma_k^{\frac{3}{2}}\nabla ^2f(x_{k-1})U_k'\\
        &\ +\alpha_k \gamma_k^2D^2f(x_{k-1})\nabla f(x_{k-1})-\alpha_k\gamma_k^2D^3f(x_{k-1})U_k'^{\otimes2}+O(\gamma_k^{\frac{5}{2}})
        \end{aligned}
    \end{equation*}
    We can decompose $\phi(x_k)$:
    \begin{equation*}
        \begin{aligned}
        \phi(x_k)-\phi(x_{k-1})&=\nabla \phi(x_{k-1})(x_k-x_{k-1})+\frac{1}{2}D^2\phi(x_{k-1})(x_k-x_{k-1})^{\otimes2}+\frac{1}{6}D^3\phi(x_{k-1})(x_k-x_{k-1})^{\otimes3}\\
        &\ +\frac{1}{24}D^4\phi(x_{k-1})(x_k-x_{k-1})^{\otimes4}+O(\gamma_k^{\frac{5}{2}})\\
        &=\nabla \phi(x_{k-1})(\sqrt{2}\gamma_k^{\frac{1}{2}}U_k-\gamma_k\nabla f(x_{k-1})-\sqrt{2\alpha_k}\gamma_k^{\frac{3}{2}}D^2f(x_{k-1})U_k'\\
        &\ +\alpha_k\gamma_k^2D^2f(x_{k-1})\nabla f(x_{k-1})-\alpha_k\gamma_k^2D^3f(x_{k-1})U_k'^{\otimes2} )\\
        &\ +\frac{1}{2}D^2\phi(x_{k-1})\left(\sqrt{2}\gamma_k^{\frac{1}{2}}U_k-\gamma_k\nabla f(x_{k-1})-\sqrt{2\alpha_k}\gamma_k^{\frac{3}{2}}D^2 f(x_{k-1})U_k'\right)^{\otimes2}\\
        &\ +\frac{1}{6}D^3\phi(x_{k-1})\left(\sqrt{2}\gamma_k^{\frac{1}{2}}U_k-\gamma_k\nabla f(x_{k-1})\right)^{\otimes3}+\frac{1}{24}D^4\phi(x_{k-1})(\sqrt{2}\gamma_k^{\frac{1}{2}}U_k)^{\otimes4}+O(\gamma_k^{\frac{5}{2}})
        \end{aligned}
    \end{equation*}
    If $\mathcal{A}$ is the generator of Langevin dynamics and summing up over $k$:
        \begin{align*}
        \sum_{k=1}^n \gamma_k\mathcal{A}\phi(x_{k-1})&= \phi(x_n)-\phi(x_0)-\sqrt{2}\sum_{k=1}^n \gamma_k^{\frac{1}{2}}\nabla \phi(x_{k-1})U_k -\sum_{k=1}^n \gamma_k\left(D^2\phi(x_{k-1})U_k^{\otimes2}-\mb{E}[D^2\phi(x_{k-1})U_k^{\otimes2}|\mathcal{F}_{k-1}]\right)\\
        &\ +\sqrt{2}\sum_{k=1}^n\gamma_k^{\frac{3}{2}}\langle D^2\phi(x_{k-1});\nabla f(x_{k-1}), U_k \rangle   -\frac{\sqrt{2}}{3}\sum_{k=1}^n\gamma_k^{\frac{3}{2}}D^3\phi(x_{k-1})U_k^{\otimes3}\\
        &\ \ +\sum_{k=1}^n\sqrt{2\alpha_k}\gamma_k^{\frac{3}{2}}\langle D^2f(x_{k-1}); \nabla \phi(x_{k-1}), U_k'\rangle +\sum_{k=1}^n\gamma_k^2\langle D^3\phi(x_{k-1});\nabla f(x_{k-1}), U_k^{\otimes2}\rangle \\
        &\ \ - \sum_{k=1}^n\alpha_k\gamma_k^2\langle D^2f(x_{k-1}); \nabla \phi(x_{k-1}),\nabla f(x_{k-1})\rangle + \sum_{k=1}^n\alpha_k \gamma_k^2\langle D^3f(x_{k-1}); \nabla \phi(x_{k-1}),U_k'^{\otimes2}\rangle \\
        &\ \ -\frac{1}{2}\sum_{k=1}^n \gamma_k^2 D^2\phi(x_{k-1})\nabla f(x_{k-1})^{\otimes2} +\sum_{k=1}^n2\alpha_k^{\frac{1}{2}}\gamma_k^2\langle D^2\phi(x_{k-1}); D^2\phi(x_{k-1})U_k',U_k\rangle \\
        &\ \ -\frac{1}{6}\sum_{k=1}^n\gamma_k^2D^4\phi(x_{k-1})U_k^{\otimes4}+\sum_{k=1}^n O(\gamma_k^\frac{5}{2}) \\
        &:=N_n^{(0)}+N_n^{(\frac{1}{2})}+N_n^{(1)}+N_n^{(\frac{3}{2})}+N_n^{(2)}+N_n^{(\frac{5}{2})}
        \end{align*}
    In the fast decreasing time step situation($\sum_{k=1}^n\gamma_k^2/\sqrt{\Gamma_n}\to 0$), the CLT for ~\eqref{RLMCdisc} is the same as that of LMC. In the slowly decreasing time step situation, when $\sum_{k=1}^n\gamma_k^2/\sqrt{\Gamma_n}\to \hat{\gamma}\in (0,+\infty]$:
    \begin{enumerate}
        \item [a)] $\frac{\phi(x_n)-\phi(x_0)}{\Gamma_n^{(2)}}\to 0$ because $(x_n)$ is tight and $\phi$ is continuous.
        \item [b)] $\frac{-\sqrt{2}\sum_{k=1}^n \gamma_k^{\frac{1}{2}}\nabla \phi(x_{k-1})U_k}{\sqrt{\Gamma_n}}\implies \mathcal{N}(0,2\int_{\mb{R}^d}|\nabla \phi(x)|^2 \pi(dx))$. Therefore,
        \begin{equation*}
            \frac{-\sqrt{2}\sum_{k=1}^n \gamma_k^{\frac{1}{2}}\nabla \phi(x_{k-1})U_k}{\Gamma_n^{(2)}}\implies
            \left\{
            \begin{aligned}
            \mathcal{N}(0,2\hat{\gamma}^{-2} \int_{\mb{R}^d}|\nabla \phi(x)|^2 \pi(dx)),\ \ \ \ \ \text{when }\hat{\gamma}<+\infty \\
            0\ \ \ \ \ ,\ \ \ \ \ \ \ \ \ \ \ \ \ \ \ \text{when }\hat{\gamma}=+\infty
            \end{aligned}
            \right.
        \end{equation*}
        \item [c)] $\frac{-\sum_{k=1}^n \gamma_k\left(D^2\phi(x_{k-1})U_k^{\otimes2}-\mb{E}[D^2\phi(x_{k-1})U_k^{\otimes2}|\mathcal{F}_{k-1}]\right)}{\sqrt{\Gamma_n}}\to 0$ in $L^2$.
        \item [d)] $\frac{\sqrt{2}\sum_{k=1}^n\gamma_k^{\frac{3}{2}}\langle D^2\phi(x_{k-1});\nabla f(x_{k-1}), U_k \rangle}{\sqrt{\Gamma_n}}\to 0$ in $L^2$. \\
        $\frac{-\frac{\sqrt{2}}{3}\sum_{k=1}^n\gamma_k^{\frac{3}{2}}D^3\phi(x_{k-1})U_k^{\otimes3}}{\sqrt{\Gamma_n}}\to 0$ in probability because $\mb{E}[U_k^{\otimes3}]=0$.\\
        $\frac{\sum_{k=1}^n\sqrt{2\alpha_k}\gamma_k^{\frac{3}{2}}\langle D^2f(x_{k-1}); \nabla \phi(x_{k-1}), U_k'\rangle}{\sqrt{\Gamma_n}}\to 0$ in $L^2$.\\
        Therefore $\frac{N_n^{(\frac{3}{2})}}{\Gamma_n^{(2)}}\to 0$ in probability.
        \item [e)] $\frac{\sum_{k=1}^n\gamma_k^2\langle D^3\phi(x_{k-1});\nabla f(x_{k-1}), U_k^{\otimes2}\rangle}{\Gamma_n^{(2)}}\to \int_{\mb{R}^d}\int_{\mb{R}^d}\langle D^3\phi(x);\nabla f(x), u^{\otimes2}\rangle\mu(du)\pi(dx)$ in probability.\\
        $\frac{ -\sum_{k=1}^n\alpha_k \gamma_k^2\langle D^2f(x_{k-1}); \nabla \phi(x_{k-1}),\nabla f(x_{k-1})\rangle}{\Gamma_n^{(2)}}\to -\frac{1}{2}\int_{\mb{R}^d}\langle D^2f(x); \nabla \phi(x),\nabla f(x)\rangle \pi(dx)$ in probability.\\
        $\frac{  \sum_{k=1}^n\alpha_k\gamma_k^2\langle D^3f(x_{k-1}); \nabla \phi(x_{k-1}),U_k'^{\otimes2}\rangle}{\Gamma_n^{(2)}}\to \frac{1}{2}\int_{\mb{R}^d}\int_{\mb{R}^d}\langle D^3 f(x); \nabla \phi(x), u^{\otimes2}\rangle \mu(du)\pi(dx)$ in probability.\\
        $\frac{ -\frac{1}{2}\sum_{k=1}^n \gamma_k^2 D^2\phi(x_{k-1})\nabla f(x_{k-1})^{\otimes2}}{\Gamma_n^{(2)}}\to -\frac{1}{2}\int_{\mb{R}^d}D^2\phi(x)\nabla f(x)^{\otimes2}\pi(dx)$ in probability.\\
        $\frac{\sum_{k=1}^n2\alpha_k^{\frac{1}{2}}\gamma_k^2\langle D^2\phi(x_{k-1}); D^2\phi(x_{k-1})U_k',U_k\rangle}{\Gamma_n^{(2)}}\to  \int_{\mb{R}^d}\int_0^1\int_{\mb{R}^{2d}} 2\alpha^{\frac{1}{2}}\langle D^2\phi(x),D^2\phi(x)u',u \rangle \mu_{\alpha}(du,du')d\alpha \pi(dx) $ in probability, where $\mu_{\alpha}(du,du')$ is the joint measure of $(U_n,U_n')$ for all $n$ conditioned on $\alpha_n=\alpha$.  With some calculation, we can simplify the limit as $\int_{\mb{R}^d} trace(D^2\phi(x)^2) \pi(dx)$. Note that in deriving the above limit, we used the fact that the cross-covariance matrix between $(U_n)$ and $(U_n')$ is $\sqrt{\alpha_n} I_d$.\\
        $\frac{-\frac{1}{6}\sum_{k=1}^n\gamma_k^2D^4\phi(x_{k-1})U_k^{\otimes4}}{\Gamma_n^{(2)}}\to -\frac{1}{6}\int_{\mb{R}^d}\int_{\mb{R}^d}D^4\phi(x)u^{\otimes4}\mu(du)\pi(dx)$ in probability.\\
        Therefore
        \[
        \frac{N_n^{(2)}}{\Gamma_n^{(2)}}\to \varrho\ \ \ \ \ \ \ \ \ \ \ \text{in probability}
        \]
        where
        \begin{equation*}
            \begin{aligned}
            \varrho&=\int_{\mb{R}^d}\int_{\mb{R}^d}\langle D^3\phi(x);\nabla f(x), u^{\otimes2}\rangle\mu(du)\pi(dx)
            -\frac{1}{2}\int_{\mb{R}^d}\langle D^2f(x); \nabla \phi(x),\nabla f(x)\rangle \pi(dx) \\
            &\ +\frac{1}{2}\int_{\mb{R}^d}\int_{\mb{R}^2}\langle D^3 f(x); \nabla \phi(x), u^{\otimes2}\rangle \mu(du)\pi(dx)
            -\frac{1}{2}\int_{\mb{R}^d}D^2\phi(x)\nabla f(x)^{\otimes2}\pi(dx)\\
            &\ +\int_{\mb{R}^d} trace(D^2\phi(x)^2) \pi(dx)-\frac{1}{6}\int_{\mb{R}^d}\int_{\mb{R}^d}D^4\phi(x)u^{\otimes4}\mu(du)\pi(dx)
            \end{aligned}
        \end{equation*}
and $\mu$ is the distribution for a $d$-dimensional standard Gaussian random variable.
        \item [f)] $ \frac{N_n^{\frac{5}{2}}}{\Gamma_n^{(2)}}\to 0$ in $L^1$.
    \end{enumerate}
    As a conclusion, we obtain the proof of part (1) of the theorem:
    \begin{equation*}
        \frac{\sum_{k=1}^n \gamma_k\mathcal{A}\phi(x_{k-1})}{\Gamma_n^{(2)}}\to 
        \left\{
            \begin{aligned}
            \mathcal{N}(\varrho,2\hat{\gamma}^{-2} \int_{\mb{R}^d}|\nabla \phi(x)|^2 \pi(dx)),\ \ \ \ \ \text{when }\hat{\gamma}<+\infty \\
            \varrho\ \ \ \ \ ,\ \ \ \ \ \ \ \ \ \ \ \ \ \ \ \text{when }\hat{\gamma}=+\infty
            \end{aligned}
            \right.
    \end{equation*}

For the fast decreasing step, i.e., part (2) of the theorem, the proof follows by the same arguments in the corresponding part of Theorem 10 in~\cite{lamberton2002recursive} and hence we omit it. 
\end{proof}

\section{Proofs for Section~\ref{sec:HMCresults}}
In this section, we would denote the drift function that appears in~\ref{underdampedlangevidiff} as $b(x,v)$, i.e.
\begin{align*}
b(x,v)=\begin{bmatrix}
        v \\
        -2v-u\nabla f(x) \\
        \end{bmatrix}
\end{align*} 
\begin{assumption}\label{conditionVHMC}
  There exists a twice differentiable function $V:\ \mb{R}^{2d}\to [1,\infty)$ such that: (0) $\lim_{\|(x,v)\|\to \infty} V(x,v)=+\infty$, (1) there exists $\alpha>0$ and $\beta>0$: $\langle \nabla V(x,v), b(x,v) \rangle\le -\alpha V(x,v)+\beta$ for every $(x,v)$, (2) there exists $c_V>0$: $\|\nabla V(x,v)\|^2+\| b(x,v)\|^2\le c_V V(x,v)$ for every $(x,v)$, and (3) $\lv D^2 V\rv_{\infty}: = \sup_{(x,v) \in \mathbb{R}^{2d}} \| D^2V\|_{\text{op}}<\infty$.

\end{assumption}
\begin{lemma}\label{AftoVHMC}
Assumption~\ref{assumptionsonf} implies Assumption~\ref{conditionVHMC} when $u\in (0,\frac{4}{2M-m})$.
\end{lemma}
\begin{proof}[Proof of Lemma~\ref{AftoVHMC}] 
For simplicity, 
 We choose $V(x,v)=\lv x-x_* \rv^2+\lv x-x_*+v \rv^1+1$ with $f(x_*)=\min f(x)$. Now we check conditions 0), 1), 2), 3) in $(\mathcal{L}_{V,\infty})$ are satisfied.\\
\begin{enumerate}[leftmargin=18pt,noitemsep]
    \item [0)] It's onvious that $\lim_{|(x,v)|\to +\infty} V(x,v)=+\infty$ and $V(x,v)\ge 1$ for all $(x,v)\in \mb{R}^d$.
    \item [3)] The Hessian of $V$ we choose is
    \begin{align*}
        D^2V(x,v)=\begin{bmatrix}
        4I_d & 2I_d \\
        2I_d & 2I_d
        \end{bmatrix}
    \end{align*}
    For arbitrary $(x,v)^T,(y,w)^T\in \mb{R}^{2d}$:
    \begin{align*}
    \lv D^2V(x,v)(y,w)^T \rv^2&=\lv\begin{bmatrix}
    4y+2w \\
    2y+2w
    \end{bmatrix}\rv^2 \\
    &\le 40 \lv (y,w)^T \rv^2
    \end{align*}
    Therefore $\lv D^2 V \rv_{\infty}<\infty$. \\   
    \item [2)] Take gradient of the $V$ we choose:
    \begin{align*}
        \nabla V(x,v)=\begin{bmatrix}
        2(x-x_*)+2(x-x_*+v) \\
        2(x-x_*+v) \\
        \end{bmatrix}
    \end{align*}
    Then for all $(x,v)\in \mb{R}^{2d}$,
    \begin{align*}
        |\nabla V(x,v)|^2+|b(x,v)|^2&\le 2(4\lv x-x_* \rv^2+4\lv x-x_*+v \rv^2)+4\lv x-x_*+v \rv^2\\
        &\ +\lv v \rv^2+2(4\lv v \rv^2+u^2\lv \nabla f(x) \rv^2)\\
        &\le 8\lv x-x_* \rv^2+12\lv x-x_*+v \rv^2+9\lv v \rv^2+2u^2M^2\lv x-x_* \rv^2\\
        &\le \max\{26+2u^2M^2, 30 \} V(x,v) \\
     \end{align*}
    \item [1)] Last we consider 
    \begin{align*}
        \langle \nabla V(x,v), b(x,v)\rangle &= 2(x-x_*)\cdot v+2(x-x_*+v)\cdot v-4(x-x_*+v)\cdot v\\
        &\ -2u(x-x_*+v)\cdot \nabla f(x) \\
        &\le -2\lv v \rv^2-2u\left[ f(x)-f(x_*-v)+\frac{m}{2}\lv x-x_* +v\rv^2 \right] \\
        &\le -2\lv v \rv^2-um\lv x-x_*+v \rv^2-2u\left( f(x_*)+\frac{m}{2}\lv x-x_* \rv^2 \right)\\
        &\ +2u\left( f(x_*)+\frac{M}{2}\lv v \rv^2 \right)\\
        &= -um\lv x-x_*+v \rv^2-um\lv x-x_* \rv^2-(2-uM)\lv v \rv^2 \\
    \end{align*}
The second inequality follows from the fact that $f$ is $m$-strongly convex.\\
When $u\in (0,\frac{2}{M}]$, $\langle \nabla V(x,v), b(x,v)\rangle\le -umV(x,v)+um $ for all $(x,v)\in\mb{R}^{2d}$. Therefore 1) is satisfied. \\
When $u>\frac{2}{M}$, we can use triangle inequality to further bound our result:
      \begin{align*}
       \langle \nabla V(x,v), b(x,v)\rangle &\le  -um\lv x-x_*+v \rv^2-um\lv x-x_* \rv^2+(uM-2)\lv v \rv^2\\
       &\le [-um+2(uM-2)](\lv x-x_*+v \rv^2+\lv x-x_* \rv^2)\\
       &\le -[4-u(2M-m)] V(x,v)-[4-u(2M-m)]\\
     \end{align*}
When $u\in (\frac{2}{M},\frac{4}{2M-m})$, 1) is satisfied because $4-u(2M-m)>0$. \\
Therefore,  1) holds when $u\in (0,\frac{4}{2M-m})$. 
 \end{enumerate}   
     \end{proof}
\begin{remark}\label{ESTV2D} 
For the $V(x,v)$ we choose in the proof, under assumption~\ref{assumptionsonf}, we can verify that: $V(x,v)=O(|x|^2+|v|^2)$ when $|(x,v)|\to +\infty$. We will use this fact later in the proof when we establish the CLT statement. 
\end{remark}

\subsection{Proofs for Section~\ref{ANLSRHMC}}
\begin{proof}[Proof of Theorem~\ref{RHMCErgodicity}]
Under the assumption~\ref{conditionVHMC}, we can show that the following Lyapunov condition is satisfied for small $h$.\\
    \textbf{(Lyapunov Condition):} There exists a function $V:\ \mb{R}^{2d}\to [1,\infty)$ such that:
    \begin{enumerate}[leftmargin=18pt,noitemsep]
        \item [0)] $\lim_{|(x,v)|\to \infty} V(x,v)=+\infty$,
        \item [1)] There exists $\hat{\alpha}\in(0,1)$ and $\hat{\beta}\ge 0$: $\mb{E}[V(x_{n+1},v_{n+1})|\mathcal{F}_n]\le \hat{\alpha} V(x_n,v_n)+\hat{\beta}$.
    \end{enumerate}
    \textbf{Proof:} To show that assumption~\ref{conditionVHMC} implies Lyapunov condition, we first do Taylor expansion of $V(x_{n+1},v_{n+1})$ at $(x_n,v_n)$:
    \begin{equation*}
        \begin{aligned}
        V(x_{n+1},v_{n+1})&=V(x_n,v_n)+\nabla V(x_n,v_n) \cdot (x_{n+1}-x_n, v_{n+1}-v_n)^T +\frac{1}{2}D^2V(\theta_n)[(x_{n+1}-x_n, v_{n+1}-v_n)^T]^{\otimes 2}
        \end{aligned}
    \end{equation*}
    where $\theta_n$ is a random point on the line segment joining $(x_n,v_n)$ and $(x_{n+1},v_{n+1})$. Use the RULMC algorithm and part (a) of Assumption~\ref{assumptionsonf}:
\begin{align*}
        \mb{E}[V(x_{n+1},v_{n+1})|\mathcal{F}_n] &\le V(x_n,v_n)+\nabla V(x_n,v_n)\cdot 
         \begin{bmatrix}
          \frac{1-e^{-2h}}{2}v_n-\frac{u}{2}(h-\frac{1-e^{-2h}}{2})\nabla f(x_n) \\
          -2\frac{1-e^{-2h}}{2} v_n-u \frac{1-e^{-2h}}{2} \nabla f(x_n) \\
         \end{bmatrix}\\
         &\ -\nabla V(x_n.v_n) \cdot 
         \begin{bmatrix}
          \frac{u}{2}(h-\frac{1-e^{-2h}}{2})\mb{E}[\nabla f(x_{n+\frac{1}{2}})- \nabla f(x_n)|\mathcal{F}_n] \\
          u \frac{1-e^{-2h}}{2}\mb{E}[\nabla f(x_{n+\frac{1}{2}})- \nabla f(x_n)|\mathcal{F}_n] \\
         \end{bmatrix}\\
         &\ +\frac{3M}{2}\left[5(\frac{1-e^{-2h}}{2})^2|v_n|^2+u^2h^2|\nabla f(x_n)^2|+({\sigma_{n+1}^{(2)}}^2+4{\sigma_{n+1}^{(3)}}^2)ud  \right]\\
         &\ +\frac{3M}{2}u^2h^2\mb{E}[ |\nabla f(x_{n+\frac{1}{2}})|^2-|\nabla f(x_n)|^2 |\mathcal{F}_n]
\end{align*}
where we can further estimate
\begin{align*}
        \mb{E}[\nabla f(x_{n+\frac{1}{2}})- \nabla f(x_n)|\mathcal{F}_n]&\le M \mb{E}[x_{n+\frac{1}{2}}-x_n|\mathcal{F}_n]\\
        &\le M\frac{1}{2h}(h-\frac{1-e^{-2h}}{2})|v_n|+\sqrt{ud}M\sigma_{n+1}^{(1)}\\
        &\ +\frac{u}{2}(\frac{h}{2}-\frac{h-\frac{1-e^{-2h}}{2}}{2h})|\nabla f(x_n)| \\
\end{align*}
and there exists $\xi_n$ such that $|\nabla f(x_{n+\frac{1}{2}})|^2-|\nabla f(x_n)|^2 =2(x_{n+\frac{1}{2}}-x_n)^T D^2 f(\xi_n) \nabla f(\xi_n)$ and $\xi_n$ is on the line segment joining $x_n$ and $x_{n+\frac{1}{2}}$. Therefore $|\xi_n-x_n|\le |x_{n+\frac{1}{2}}-x_n|$. then we have
\begin{align*}
       \mb{E}[ |\nabla f(x_{n+\frac{1}{2}})|^2-|\nabla f(x_n)|^2 |\mathcal{F}_n]&\le 2M\mb{E}[|\nabla f(\xi_n)||x_{n+\frac{1}{2}}-x_n||\mathcal{F}_n]\\
       &\le 2M|\nabla f(x_n)|\mb{E}[|x_{n+\frac{1}{2}}-x_n||\mathcal{F}_n]+2M^2\mb{E}[|x_{n+\frac{1}{2}}-x_n|^2|\mathcal{F}_n]\\
       &\le |\nabla f(x_n)|^2+3M^2\mb{E}[|x_{n+\frac{1}{2}}-x_n|^2|\mathcal{F}_n]\\ 
       &\le |\nabla f(x_n)|^2+6M^2(\frac{h^2}{3}|v_n|^2+\frac{u^2h^4}{20}|\nabla f(x_n)|^2+ud{\sigma_{n+1}^{(1)}}^2)
\end{align*}
When $h$ is small, we can use polynomials of $h$ to bound those exponential coefficients. We can obtain that there exists $C>0$:
\begin{align*}
     \mb{E}[V(x_{n+1},v_{n+1})|\mathcal{F}_n] &\le V(x_n,v_n)+h \nabla V(x_n,v_n) \cdot b(x_n,v_n)^T+Ch^2(d+|v_n|^2+|\nabla f(x_n)|^2)
\end{align*}
then assumption~\ref{conditionVHMC} imples that there exists $\alpha>0, \beta>0$ such that
\begin{align*}
     \mb{E}[V(x_{n+1},v_{n+1})|\mathcal{F}_n] &\le (1-\alpha h+C c_V h^2)V(x_n,v_n)+Ch^2d+\beta
\end{align*}
When $h$ is small, there exists $\hat{\alpha}=1-\alpha h+C c_V h^2\in (0,1)$ and $\hat{\beta}=Ch^2d+\beta>0$ such that $\mb{E}[V(x_{n+1},v_{n+1})|\mathcal{F}_n] \le \hat{\alpha}V(x_n,v_n)+\hat{\beta}$. \qed \\
\\
    Once we have the Lyapunov condition, we can define the stopping time $\tau_C=\inf \{n>0: (x_n,v_n)\in C\}$ and show that $\sup_{(x,v)\in C}\mb{E}_{(x,v)}[\tau_C]\le M_C<\infty$ for all small set C. Then uniqueness of stationary probability measure and ergodicity all follow by Theorem 1.3.1 in~\cite{meyn2012markov}. Next we prove that $\sup_{(x,v)\in C}\mb{E}_{(x,v)}[\tau_C]\le M_C<\infty$ given Lyapunov condition. To do so, note that we have
        \begin{equation*}
        \begin{aligned}
        \mb{E}_{(x,v)}[\tau_C]&=\sum_{n=1}^{\infty} n\mb{P}(\tau_C=n)\\
        &=\sum_{n\ge 1}\mb{P}(\tau_C>n-1)
        \end{aligned}
    \end{equation*}
    Under Lyapunov condition, for any stopping time $N$, according to Lemma A.3 and Corollary A.4 in~\cite{mattingly2002ergodicity}, we have 
    \begin{equation*}
        \begin{aligned}
        \mb{P}(\tau_C>n-1)&\le \mb{E}[V(x_n,v_n)1_{\tau_C>n-1}]\\
        &\le \frac{\kappa [\gamma^{n-1}V(x_0,v_0)+1]}{1-\gamma} \\
        &\le \kappa \gamma^{n-1}[V(x_0,v_0)+1]
        \end{aligned}
    \end{equation*}
    for some $\gamma\in (\hat{\alpha},1)$ and constant $\kappa$. Therefore,  we have
    \begin{equation*}
        \begin{aligned}
        \mb{E}_{(x,v)}[\tau_C]&\le \sum_{k\ge 1}\kappa \gamma^{n-1}[V(x_0,v_0)+1]\\
        &=\frac{\kappa[V(x,v)+1]}{1-\gamma}
        \end{aligned}
    \end{equation*}
    and 
    \[
    \sup_{(x,v)\in C}\mb{E}_{(x,v)}[\tau_C]\le \frac{\kappa}{1-\gamma}\sup_{(x,v)\in C}V(x,v)+\frac{\kappa}{1-\gamma}\le M_C<\infty
    \] 
    So as a conclusion, the statement of the theorem follows.
      \end{proof}
 Before proving Proposition~\ref{biasRHMC}, we require some preliminary estimtes from~\cite{shen2019randomized}, that we present below.
First, let $(y_n, w_n)$ be the solution of Underdamped Langevin dynamics evaluated at $t=\sum_{k=1}^n \gamma_k$ with initial value $(y_0, w_0)$. $(x_n, v_n)$ is the $n$th iterates in the \eqref{rhmc} algorithm with initial value $(x_0, v_0)$. $(x_n^*(t), v_n^*(t))$ is the solution of Underdamped Langevin dynamics with initial value $(x_{n-1}, v_{n-1})$ and $(x_n^*, v_n^*)=(x_{n-1}^*(\gamma_n), v_{n-1}^*(\gamma_n))$. 
Then, we have the following results from Lemma 2 in~\cite{shen2019randomized}. When $\gamma_{n+1}<\frac{1}{2}$ and $u=\frac{1}{M}$, we have:
        \begin{equation*}
            \begin{aligned}
            \mb{E}\lv \mb{E}_{\alpha}x_{n+1}-x_{n+1}^* \rv^2&\le 45(\gamma_{n+1}^{10}\mb{E}\lv v_n \rv^2+M^{-2}\gamma_{n+1}^{12}\mb{E}\lv \nabla f(x_n) \rv^2+M^{-1}d\gamma_{n+1}^{11}) \\
            \mb{E}\lv x_{n+1}-x_{n+1}^* \rv^2&\le 1800(\gamma_{n+1}^{6}\mb{E}\lv v_n \rv^2+M^{-2}\gamma_{n+1}^{4}\mb{E}\lv \nabla f(x_n) \rv^2+M^{-1}d\gamma_{n+1}^{7}) \\
            \mb{E}\lv \mb{E}_{\alpha}v_{n+1}-v_{n+1}^* \rv^2&\le 45(\gamma_{n+1}^{8}\mb{E}\lv v_n \rv^2+M^{-2}\gamma_{n+1}^{10}\mb{E}\lv \nabla f(x_n) \rv^2+M^{-1}d\gamma_{n+1}^{9}) \\
            \mb{E}\lv v_{n+1}-v_{n+1}^* \rv^2&\le 1300(\gamma_{n+1}^{4}\mb{E}\lv v_n \rv^2+M^{-2}\gamma_{n+1}^{4}\mb{E}\lv \nabla f(x_n) \rv^2+M^{-1}d\gamma_{n+1}^{5}) \\
            \end{aligned}
        \end{equation*}
\begin{proof}[Proof of Proposition~\ref{biasRHMC} ]
Denote $A_n^2=\mb{E}[\lv x_n-y_n \rv^2+\lv (x_n+v_n)-(y_n+w_n) \rv^2]$. Using triangle inequality we have 
    \begin{equation*}
        \begin{aligned}
        \mb{E}_{\alpha}[\lv x_n-y_n \rv^2+\lv (x_n+v_n)-(y_n+w_n) \rv^2]& \le (1+\frac{h}{2\kappa})(\lv x^*_k-y_n \rv^2+\lv (x^*_k+v^*_k)-(y_n+w_n) \rv^2)\\
        &\ +\frac{2\kappa}{h}(\lv \mb{E}_{\alpha}[x_n]-x^*_k \rv^2+\lv \mb{E}_{\alpha}[x_n+v_n]-(x_n^*+v_n^*) \rv^2)\\
        &\ +\mb{E}_{\alpha}\lv x_n-x_n^* \rv^2+\mb{E}_{\alpha}\lv (x_n+v_n)-(x_n^*+v_n^*) \rv^2 \\
        \end{aligned}
    \end{equation*}
    Furthermore, we can take expectation on $\omega$ and use the contraction of Underdamped Langevin dynamics:
    \begin{equation*}
        \begin{aligned}
        A_n^2 &\le (1+\frac{h}{2\kappa})e^{-\frac{h}{\kappa}}A_{n-1}^2+\frac{2\kappa}{h}(\mb{E}\lv \mb{E}_\alpha x_n-x_n^* \rv^2+\mb{E}\lv \mb{E}_{\alpha}[x_n+v_n]-(x_n^*+v_n^*) \rv^2)\\
        &\ +\mb{E}\lv x_n^*-x_n \rv^2+\mb{E}\lv (x_n+v_n)-(x_n^*+v_n^*) \rv^2  \\
        &\le e^{-\frac{h}{2\kappa}}A_{n-1}^2+\frac{2\kappa}{h}(3\mb{E}\lv \mb{E}_{\alpha}x_n-x_n^* \rv^2+2\mb{E}\lv \mb{E}_{\alpha}v_n-v_n^* \rv^2)\\
        &\ +3 \mb{E}\lv x_{n}-x_{n}^* \rv^2+2\mb{E}\lv v_{n}-v_{n}^* \rv^2
        \end{aligned}
    \end{equation*}
    When $h<\frac{1}{2}$, $u=\frac{1}{M}$ and $m=1$:
    \begin{equation*}
        \begin{aligned}
        A_{n}^2 &\le e^{-\frac{h}{2\kappa}}A_{n-1}^2+8250\left[ (\kappa h^7+h^4)\mb{E}\lv v_{n-1} \rv^2+(\kappa^{-1}h^8+\kappa^{-2}h^4)\mb{E}\lv \nabla f(x_{n-1}) \rv^2 +(\kappa^{-1}h^5+h^7) \right]
        \end{aligned}
    \end{equation*}
    Our next step is to bound $\mb{E}\lv v_{n-1} \rv^2$ and $\mb{E}\lv \nabla f(x_{n-1}) \rv^2$. First for Underdamped Langevin dynamics with $f$ satisfying Assumption~\ref{assumptionsonf}, it's easy to compute that:
    \begin{equation*}
        \begin{aligned}
        \mb{E}\lv w_{n-1} \rv^2&=d/M \\
        \mb{E}\lv  \nabla f(y_{n-1}) \rv^2&=\frac{1}{\int e^{-f(x)}dx} \int |\nabla f(x)|^2 e^{-f(x)}dx \\
        & = -\frac{1}{\int e^{-f(x)}dx} \int (\nabla f(x))^T \nabla e^{-f(x)}dx \\
        & = \frac{1}{\int e^{-f(x)}dx} \int \Delta f(x) e^{-f(x)}dx \\
        &\le \lv \Delta f(x)\rv_{\infty}\le Md 
        \end{aligned}
    \end{equation*}
    Therefore, we have 
    \begin{equation*}
        \begin{aligned}
        \mb{E}\lv v_{n-1} \rv^2 &\le 2d/M+2\mb{E}\lv v_{n-1}-w_{n-1}\rv^2\le 2d/M+4A_{n-1}^2 \\
        \mb{E}\lv \nabla f(x_{n-1}) \rv^2 & \le 2Md+ 2M^2 \mb{E}\lv x_{n-1}-y_{n-1} \rv^2 \le 2Md+2M^2A_{n-1}^2\\ 
        \end{aligned}
    \end{equation*}
    Plug the upper bounds into our previous result:
        \begin{equation*}
        \begin{aligned}
        A_{n}^2 &\le e^{-\frac{h}{2\kappa}}A_{n-1}^2+8250\left[ (\kappa h^7+h^4)(2d/M+4A_{n-1}^2)+(\kappa^{-1}h^8+\kappa^{-2}h^4)(2Md+2M^2A_{n-1}^2) +(\kappa^{-1}h^5+h^7) \right] \\
        &\le \left[ 1-\frac{h}{2\kappa}+\frac{h^2}{8\kappa^2}+49500(h^4+\kappa h^7) \right]A_{n-1}^2+41250d(h^7+\kappa^{-1}h^4)
        \end{aligned}
    \end{equation*}
    If we choose $(x_{n-1}, v_{n-1})\sim \pi_h^*(x,v)$ and $(y_{n-1},w_{n-1})\sim \pi^*(x,v)$ such that
    \[
    A_{n-1}^2=\min_{X\sim \pi^*_h,\ Y\sim \pi^*} \mb{E}\lv X-Y \rv^2
    \]
    Then we have 
    \begin{equation*}
        \begin{aligned}
        W_2(\pi,\pi_h)^2\le A_{n-1}^2\le \frac{82500h^3(\kappa h^3+1)d}{1-\frac{h}{4\kappa}-99000h^3\kappa(1+\kappa h^3)}
        \end{aligned}
    \end{equation*}
    We can see that $W_2(\pi,\pi_h)\to 0$ as $h\to 0$. Furthermore, as $h\to 0$, $W_2(\pi,\pi_h)< O(h^{\frac{3}{2}})$.
\end{proof}

\subsection{Proofs for Section~\ref{sec:RHMCdecreasing}}

\begin{proof}[Proof of Theorem~\ref{W2RHMC-fast}]        
Define $A_n^2=\mb{E}[\lv x_n-y_n \rv^2+\lv (x_n+v_n)-(y_n+w_n) \rv^2]$. From the proof of proposition~\ref{biasRHMC}, we know that 
    \[
    A_{n}^2 \le \left[ 1-\frac{\gamma_n}{2\kappa}+\frac{\gamma_n^2}{8\kappa^2}+49500(\gamma_n^4+\kappa \gamma_n^7) \right]A_{n-1}^2+41250d(\gamma_n^7+\kappa^{-1}\gamma_n^4)
    \]
    When time step $h$ is a constant, apply the inequality repeatedly to get
    \[
    A_n^2\le \left[ 1-\frac{h}{2\kappa}+\frac{h^2}{8\kappa^2}+49500(h^4+\kappa h^7) \right]^k A_0^2+\frac{82500h^3(\kappa h^3+1)d}{1-\frac{h}{4\kappa}-99000h^3\kappa(1+\kappa h^3)}
    \]
    Denote $\nu_n$ to be the density function of $x_n$, then $W_2(\nu_n,\pi)\le A_n$. By choosing $\gamma_n=h\sim O(\epsilon^{\frac{2}{3}})$, we can guarantee that $W_2(\nu_n,\pi)<\epsilon \sqrt{\frac{d}{m}}$ for all $n>K\sim \Tilde{O}(\epsilon^{-\frac{2}{3}})$.\\
    When the time step $\gamma_n$ is variant, the inequality we correspondingly have 
    \[
    A_{n}^2 \le \left[ 1-\frac{\gamma_n}{2\kappa}+\frac{\gamma_n^2}{8\kappa^2}+49500(\gamma_n^4+\kappa \gamma_n^7) \right]A_{n-1}^2+41250d(\gamma_n^7+\kappa^{-1}\gamma_n^4)
    \]
    When $\gamma_n< 1$, $\frac{\gamma_n^2}{8\kappa^2}<\frac{\gamma_n}{8\kappa}$. When $\gamma_n<(\frac{99000}{8\kappa^2})<24\kappa^{-\frac{2}{3}}$, we have $49500(\gamma_n^4+\kappa \gamma_n^7)<\frac{\gamma_n}{8\kappa}$. Similarly, when $\gamma_n<1$, we have $41250d(\gamma_n^7+\kappa^{-1}\gamma_n^4)<82500d \gamma_n^4$. Therefore, when $\gamma_n<\min\{1/2, 24\kappa^{-\frac{2}{3}}\}$, we have 
    \[
    A_n^2<(1-\frac{\gamma_n}{4\kappa})A_{n-1}^2+82500d\gamma_n^4
    \]
    If we choose $\gamma_n=\frac{16\kappa}{32\kappa^{\frac{5}{3}}+(n-K_1)^+}$, where $K_1$ is the smallest integer such that
    \[
    A_{K_1}^2<(1-\frac{4}{\kappa^{\frac{5}{3}}})^{K_1}A_0^2+(82500)d \frac{1}{2\kappa}<2\frac{82500 d}{\kappa}
    \]
    Then we claim that for all $n\ge K_1$, we have
    \[
    A_n^2<\frac{82500 (16)^4d\kappa^4}{(32\kappa^{\frac{5}{3}}+n-K_1)^3}
    \]
    The claim can be proved by induction: Assume that the claim hold for $A_n^2$ and denote $b=32\kappa^{\frac{5}{3}}+n-K_1$, then
    \begin{equation*}
        \begin{aligned}
        A_{n+1}^2&< (1-\frac{4}{1+b})\frac{ 82500(16)^4d\kappa^4}{b^3}+\frac{82500d(16)^4\kappa^4}{(b+1)^4} \\
        &=\frac{82500(16)^4 d\kappa^4}{(b+1)^3}\left[ \frac{(b-3)(b+1)^2}{b^3}+\frac{1}{b+1} \right]\\
        &< \frac{82500(16)^4 d\kappa^4}{(b+1)^3}\\
        &=\frac{82500 (16)^4d\kappa^4}{(32\kappa^{\frac{5}{3}}+n+1-K_1)^3}
        \end{aligned}
    \end{equation*}
    Therefore, under our choice of time step $(\gamma_n)$, we can guarantee $W_2(\nu_n,\pi)<\epsilon\sqrt{\frac{d}{m}}$ for all $n>K\sim O(\epsilon^{-\frac{2}{3}})$. Compared to the running time of constant step size RULMC, vanishing step size help reduce the factor $\log(\frac{1}{\epsilon})$ in the guarantees.
 \end{proof}

Now we introduce the CLT statement for another sampling algorithm related to \eqref{rhmc} and give a complete proof of the statement. The proof of Remark~\ref{1dCLTRHMC} can be done in the same way. In the following theorem, we give a central limit result with specific choice of weights and time step-size. The Euler-discretization of the underdamped Langevin diffusion (which we call as KLMC, following~\cite{dalalyan2018sampling}) is given by the following algorithm:
\begin{equation}\tag{\textsc{KLMC}}
        \begin{aligned}\label{klmcdisc}
        x_{n+1}&=x_n+\frac{1-e^{-2\gamma_{n+1}}}{2}v_n-\frac{u}{2}(\gamma_{n+1}-\frac{1-e^{-2\gamma_{n+1}}}{2})\nabla f(x_n)+\sqrt{u}\sigma_{n+1}^{(1)}U_{n+1}^{(1)} \\
        v_{n+1}&=v_n e^{-2\gamma_{n+1}}-u\frac{1-e^{-2\gamma_{n+1}}}{2}\nabla f(x_n)+2\sqrt{u}\sigma_{n+1}^{(2)}U_{n+1}^{(2)}
        \end{aligned}
    \end{equation}
    where $\{\gamma_n\}$ are the time steps. $\sigma_n^{(1)}$ and $\sigma_n^{(2)}$ are positive with ${\sigma_n^{(1)}}^2=\gamma_n+\frac{1-e^{-4\gamma_n}}{4}-(1-e^{-2\gamma_n})$, ${\sigma_n^{(2)}}^2=\frac{1-e^{-4\gamma_n}}{4}$. $\{(U_n^{(1)}, U_n^{(2)})\}_n$ are independent Centered Gaussian random vectors in $\mb{R}^{2d}$ with $(U_n^{(1)}, U_n^{(2)}) \sim \mathcal{N}(0,\sigma_n^2 I_d)$ and $\sigma_n^2=\frac{1+e^{-4\gamma_n}-2e^{-2\gamma_n}}{4\sigma_n^{(1)}\sigma_n^{(2)}}$. 
Numerical integration with the above sampler follows the same steps as described in Section~\ref{sec:RHMCdecreasing}. We now provide the following CLT.
\begin{theorem}\label{1dCLTKLMC}
Assume potential function $f$ satisfies Assumption~\ref{assumptionsonf}. Let $\{(x_k,v_k)\}$ and $\{(U_k^{(1)},U_k^{(2)})\}$ be the same as what we have in the \eqref{klmcdisc} algorithm and the time step-size $\{\gamma_k\}$ is non-increasing and $\lim_k(\gamma_{k-1}-\gamma_k)/\gamma_k^4=0$. If $\lim_n (1/\sqrt{\Gamma_n^{(3)}})\sum_{k=1}^n\gamma_k^4=\hat{\gamma}\in (0,+\infty]$ and $\lim_n \Gamma_n^{(4)}=+\infty$, then for all $\phi\in \mathcal{C}^3$ with $D^2\phi $, $D^3\phi$ and $D^4\phi$ bounded and Lipschitz and $\sup_{(x,v)\in \mb{R}^{2d}} |\nabla \phi(x)|^2/V(x,v)<+\infty$, we have
\begin{align*}
    \frac{\Gamma_n}{\Gamma_n^{(4)}}\nu_n^\gamma(\mathcal{L}\phi)&\to \mathcal{N}(\rho, \frac{10}{3}u\hat{\gamma}^{-2}\int_{\mb{R}^d}|\nabla \phi (x)|\pi(dx) )\ \ \ \ \ \ \ \text{if}\ \hat{\gamma}<+\infty\\
    \frac{\Gamma_n}{\Gamma_n^{(4)}}\nu_n^\gamma(\mathcal{L}\phi)&\to \rho\ \ \ \ \ \ \ \ \ \ \ \ \ \ \ \ \ \ \ \ \ \ \ \ \ \ \ \ \ \ \ \ \ \ \ \ \ \ \ \ \ \ \ \ \ \text{if}\ \hat{\gamma}=+\infty,
\end{align*}
where
\begin{align*}
       \rho&=\tfrac{u}{6}\smallint\smallint \langle D^3\phi(x); \nabla f(x), v^{\otimes 2} \rangle\nu(dx,dv)+\tfrac{u}{24}\smallint\smallint \langle D^3f(x); \nabla \phi(x), v^{\otimes 2} \rangle \nu(dx,dv)\\
    &\ +\tfrac{u}{12}\smallint\smallint(D^2\phi D^2f)(x)v^{\otimes 2}\nu(dx,dv)-\tfrac{1}{12}\smallint \smallint D^4\phi(x)v^{\otimes 4}\nu(dx,dv)\\
    &\ -\tfrac{u^2}{24}\smallint \langle D^2 f(x); \nabla \phi(x), \nabla f(x)\rangle \pi(dx).
\end{align*}
\end{theorem}
In the following context we'll discuss the weak convergence of empirical measure $\nu_n^{\eta}$ and build a central limit theorem under certain assumptions.
\begin{enumerate}[leftmargin=18pt,noitemsep]
    \item [1)] (Lyapunov Conditions) The underdamped Langevin dynamics can be rewritten as 
    \[
    d Y_t= b(Y_t) dt+\sigma(Y_t) dW_t
    \]
    where $Y_t=[X_t, V_t]^T$, $b(y)=b(x, v)=[v, -2v-u\nabla f(x)]^T$, $\sigma(y)=2\sqrt{u}[0_d, I_d]^T$ for all $x, v\in \mb{R}^d$. $\{W_t\}$ is a 2d-dimensional Brownian motion.

The Lyapunov condition is similar to the one that's introduced in\cite{lamberton2002recursive}.\\
\\
\textbf{Assumption $(\mathcal{L}_{V,\infty})$}: There's a $\mathcal{C}^2$ function $V:\mb{R}^{2d}\to [v_*,+\infty)$ for some $v_*>0$ satisfying the following conditions:
\begin{enumerate}[leftmargin=18pt,noitemsep]
    \item [a)] $\lv D^2V \rv_{\infty}=\sup_{(x,v)^T\in \mb{R}^{2d}} \lv D^2V(x,v) \rv_{op}<+\infty$ and $\lim_{|(x,v)|\to +\infty} V(x,v)=+\infty$; 
    \item [b)] $|\nabla V(x,v)|^2+|b(x,v)|^2\le c_V V(x,v)$ for all $(x,v)^T\in \mb{R}^{2d}$ and some $c_V>0$;
    \item [c)] $\langle \nabla V(x,v), b(x,v) \rangle \le -\alpha V(x,v)+\beta $ for some $\alpha>0$ and $\beta\in \mb{R}$.
\end{enumerate}
\textbf{Assumption $(\mathcal{L}_{V,p})$}: There's a $\mathcal{C}^2$ function $V:\mb{R}^{2d}\to [v_*,+\infty)$ for some $v_*>0$ satisfying for some $p\ge 1$:
\begin{enumerate}[leftmargin=18pt,noitemsep]
    \item [a)] $\lv D^2V \rv_{\infty}=\sup_{(x,v)^T\in \mb{R}^{2d}} \lv D^2V(x,v) \rv_{op}<+\infty$ and $\lim_{|(x,v)|\to +\infty} V(x,v)=+\infty$; 
    \item [b)] $|\nabla V(x,v)|^2+|b(x,v)|^2+\text{Tr}(\sigma(x,v)\sigma(x,v)^T)\le c_V V(x,v)$ for all $(x,v)^T\in \mb{R}^{2d}$ and some $c_V>0$;
    \item [c)] $\langle \nabla V(x,v), b(x,v) \rangle+\lambda_p \text{Tr}(\sigma(x,v)\sigma(x,v)^T) \le -\alpha V(x,v)+\beta $ for some $\alpha>0$ and $\beta\in \mb{R}$, where $\lambda_p=\frac{1}{2}\lambda_{D^2V+(p-1)(\nabla V\otimes \nabla V)/V}$.
\end{enumerate}
\begin{remark}\label{RmkonLyf} 
\begin{enumerate}[leftmargin=18pt,noitemsep]
\item [1)] We can show that: $(\mathcal{L}_{V,p'})\implies (\mathcal{L}_{V,p})$ if $p'\ge p\ge 1$. Especially $(\mathcal{L}_{V,\infty})\implies (\mathcal{L}_{V,p})$ for all $p\ge 1$.\\
\item [2)] If we choose $b$ and $\sigma$ the same as those in the Underdamped Langevin dynamics, then $(\mathcal{L}_{V,\infty})$ is almost the same as assumption~\ref{conditionVHMC}. We can instantly obtain that assumption~\ref{conditionVHMC} implies $(\mathcal{L}_{V,\infty})$. Therefore, according to lemma~\ref{AftoVHMC}, assumption~\ref{assumptionsonf} implies $(\mathcal{L}_{V,\infty})$.
\end{enumerate}
\end{remark}
\item [2)] (Tightness Result)
 We now establish the almost sure tightness of the weighted empirical measures. The filtration $\{\mathcal{F}_n\}$ we consider is $\mathcal{F}_n=\sigma(Y_0,(U_1^{(1)},U_1^{(2)}),\cdots,(U_n^{(1)},U_n^{(2)}))$.\\
    \begin{lemma}\label{tightness}
    \begin{itemize}
\item [(a)] If $(\mathcal{L}_{V,1})$ holds, then for every $a\ge \frac{1}{2}$, 
    \[
    |V^a(Y_{n+1})-V^a(Y_n)|\le c_a\sqrt{\gamma_{n+1}}V^a(Y_n)(1+|U_{n+1}^{(1)}|^{2a}+|U_{n+1}^{(2)}|^{2a})
    \]
  \item[(b)] If $(\mathcal{L}_{V,p})$ holds for some $p\ge 1$, then there exists real numbers $\tilde{\alpha}>0$ and $\tilde{\beta}$ and $n_0\in \mb{N}$ such that
    \[
    \mb{E}[V^p(Y_{n+1})|\mathcal{F}_n]\le V^(Y_n)+\gamma_{n+1}V^{p-1}(Y_n)(\tilde{\beta}-\tilde{\alpha}V(Y_n)),\ \ \ \ \forall\ n\ge n_0
    \]
    and furthermore
    \[
    \sup_{n\in \mb{N}} \mb{E}[V^p(Y_n)]<+\infty
    \]
    \end{itemize}
    \end{lemma}
 \begin{proof}[Proof of Lemma~\ref{tightness}]
 (a) Using mean value theorem and $(\mathcal{L}_{V,1})$:
    \begin{align*}
        |V^a(Y_{n+1})-V^a(Y_n)|&=a|V^{a-1}(\xi_{n+1})\langle \nabla V(\xi_{n+1}), Y_{n+1}-Y_n \rangle|\\
        &\le CV^{a-\frac{1}{2}}(\xi_{n+1})|Y_{n+1}-Y_n|
    \end{align*}
    From $(\mathcal{L}_{V,1})$-b) we get that $\nabla \sqrt{V}$ is bounded, i.e $\sqrt{V}$ is Lipschitz with parameter $[\sqrt{V}]_1$. Hence
    \begin{align*}
        V^{a-\frac{1}{2}}(\xi_{n+1})&\le (\sqrt{V}(Y_n)+[\sqrt{V}]_1|Y_{n+1}-Y_n|)^{2a-1} \\
        &\le 2^{2a-1}\left(V^{a-\frac{1}{2}}(Y_n)+[\sqrt{V}]_1^{2a-1}|Y_{n+1}-Y_n|^{2a-1}\right)
    \end{align*}
    Meanwhile,
    \begin{align*}
        |Y_{n+1}-Y_n|^2&=|\begin{bmatrix}
        \frac{1-e^{-2\gamma_{n+1}}}{2}v_n-\frac{u}{2}(\gamma_{n+1}-\frac{1-e^{-2\gamma_{n+1}}}{2})\nabla f(x_n)+\sqrt{u}\sigma_{n+1}^{(1)}U_{n+1}^{(1)} \\
        -2\frac{1-e^{-2\gamma_{n+1}}}{2}v_n-u\frac{1-e^{-2\gamma_{n+1}}}{2}\nabla f(x_n)+2\sqrt{u}\sigma_{n+1}^{(2)}U_{n+1}^{(2)}
        \end{bmatrix}|^2\\
        &\le 15(\frac{1-e^{-2\gamma_{n+1}}}{2})^2|v_n^2|+[\frac{3u^2}{4}(\gamma_{n+1}-\frac{1-e^{-2\gamma_{n+1}}}{2})^2+3u^2(\frac{1-e^{-2\gamma_{n+1}}}{2})^2]|\nabla f(x_n)|^2\\
        &\ +3u{\sigma_{n+1}^{(1)}}^2|U_{n+1}^{(1)}|^2+12u{\sigma_{n+1}^{(2)}}^2|U_{n+1}^{(2)}|^2
    \end{align*}
    Since $\gamma_n\to 0$ as $\n\to \infty$ and $\frac{1-e^{-2\gamma_n}}{2}\sim O(\gamma_n)$, $\gamma_n-\frac{1-e^{-2\gamma_n}}{2}\sim O(\gamma_n^2)$, $\sigma_n^{(1)}\sim O(\gamma_n^{\frac{3}{2}})$ and $\sigma_n^{(2)}\sim O(\gamma_n^{\frac{1}{2}})$, there exist $C_1,C_2,C_3>0$ such that
    \begin{align*}
        |Y_{n+1}-Y_n|^2&\le C_1\left[\gamma_{n+1}^2(|v_n|^2+|\nabla f(x_n)|^2)+\gamma_{n+1}(|U_{n+1}^{(1)}|^2+|U_{n+1}^{(2)}|^2)\right]\\
        &\le C_2\left[\gamma_{n+1}^2 V(Y_n)+\gamma_{n+1}(|U_{n+1}^{(1)}|^2+|U_{n+1}^{(2)}|^2+1)\right]\\
        \implies\ \ |Y_{n+1}-Y_{n}|&\le C_3\sqrt{\gamma_{n+1}}\sqrt{V(Y_n)}(|U_{n+1}^{(1)}|+|U_{n+1}^{(2)}|+1)
    \end{align*}
    Combining our estimations, since $a\ge 1/2$. we get
    \begin{align*}
        |V^a(Y_{n+1})-V^a(Y_n)|&\le C 2^{2a-1}\left(V^{a-\frac{1}{2}}(Y_n)+[\sqrt{V}]_1^{2a-1}|Y_{n+1}-Y_n|^{2a-1}\right)|Y_{n+1}-Y_n|\\
        &\le c_a'\left( \sqrt{\gamma_{n+1}}V^a(Y_n)(|U_{n+1}^{(1)}|+|U_{n+1}^{(2)}|+1)+\gamma_{n+1}^a V^a(Y_n)(|U_{n+1}^{(1)}|+|U_{n+1}^{(2)}|+1)^{2a}\right) \\
        &\le c_a\sqrt{\gamma_{n+1}}V^a(Y_n)(|U_{n+1}^{(1)}|^{2a}+|U_{n+1}^{(2)}|^{2a}+1)
    \end{align*}
    (b) We Taylor expand $V^p(Y_{n+1})$ at $Y_n$:
    \begin{align*}
        V^p(Y_{n+1})&=V^p(Y_n)+pV^{p-1}(Y_n)\langle \nabla V(Y_n), Y_{n+1}-Y_n\rangle +\frac{1}{2}D^2(V^p)(\xi_{n+1})(Y_{n+1}-Y_n)^{\otimes 2}
    \end{align*}
    Since $D^2(V^p)=pV^{p-1}D^2V+p(p-1)V^{p-1}\nabla V \nabla V^T$, by the definition of $\lambda_p$:
    \begin{align*}
        D^2(V^p)(\xi_{n+1})(Y_{n+1}-Y_n)^{\otimes 2}\le 2p\lambda_pV^{p-1}(\xi_{n+1})|Y_{n+1}-Y_n|^2
    \end{align*}
    Therefore
    \begin{align*}
        V^p(Y_{n+1})\le V^p(Y_n)+pV^{p-1}(Y_n)\langle \nabla V(Y_n), Y_{n+1}-Y_n\rangle +p\lambda_p V^{p-1}(\xi_{n+1})|Y_{n+1}-Y_n|
    \end{align*}
    When $p=1$, take conditional expectation on $\mathcal{F}_n$:
    \begin{align*}
        \mb{E}[V(Y_{n+1})|\mathcal{F}_n]&\le V(Y_n)+\frac{1-e^{-2\gamma_{n+1}}}{2} \langle \nabla V(x_n,v_n), b(x_n,v_n)\rangle \\
        &\ -\frac{u}{2}(\gamma_{n+1}-\frac{1-e^{-2\gamma_{n+1}}}{2})\nabla_x V(x_n,v_n)\cdot \nabla f(x_n)\\
        &\ +\lambda_1 (\frac{1-e^{-2\gamma_{n+1}}}{2})^2[5|v_n|^2+u^2|\nabla f(x_n)|^2+4u\nabla f(x_n)\cdot v_n]\\
        &\ -\frac{u}{2}\frac{1-e^{-2\gamma_{n+1}}}{2}(\gamma_{n+1}-\frac{1-e^{-2\gamma_{n+1}}}{2}) \nabla f(x_n)\cdot v_n \\
        &\ +\frac{u^2}{4}(\gamma_{n+1}-\frac{1-e^{-2\gamma_{n+1}}}{2})^2|\nabla f(x_n)|^2+u({\sigma_{n+1}^{(1)}}^2+4{\sigma_{n+1}^{(2)}}^2)d
        \end{align*}
        There exists $n_0\in \mb{N}$ such that for all $n\ge n_0$
        \begin{align*}
           &\frac{1-e^{-2\gamma_{n+1}}}{2} \langle \nabla V(x_n,v_n), b(x_n,v_n)\rangle \le \gamma_{n+1}(-\alpha V(Y_n)+\beta),\ \ \ \ \text{for some }\alpha>0,\beta\in \mb{R};\\
           &-\frac{u}{2}(\gamma_{n+1}-\frac{1-e^{-2\gamma_{n+1}}}{2})\nabla_x V(x_n,v_n)\cdot \nabla f(x_n)\le C \gamma_{n+1}^2(|\nabla V(Y_n)|^2+|b(Y_n)|^2)\le C\gamma_{n+1}^2V(Y_n); \\
           &\lambda_1 (\frac{1-e^{-2\gamma_{n+1}}}{2})^2[5|v_n|^2+u^2|\nabla f(x_n)|^2+4u\nabla f(x_n)\cdot v_n]\le C\gamma_{n+1}^2|b(Y_n)|^2\le C\gamma_{n+1}^2V(Y_n);\\
           &-\frac{u}{2}\frac{1-e^{-2\gamma_{n+1}}}{2}(\gamma_{n+1}-\frac{1-e^{-2\gamma_{n+1}}}{2}) \nabla f(x_n)\cdot v_n\le C\gamma_{n+1}^3|b(Y_n)|^2\le  C\gamma_{n+1}^3 V(Y_n);\\
           &\frac{u^2}{4}(\gamma_{n+1}-\frac{1-e^{-2\gamma_{n+1}}}{2})^2|\nabla f(x_n)|^2\le C\gamma_{n+1}^4|b(Y_n)|^2\le C\gamma_{n+1}^4 V(Y_n); \\
           &u({\sigma_{n+1}^{(1)}}^2+4{\sigma_{n+1}^{(2)}}^2)d\le C\gamma_{n+1}.
        \end{align*}
        Therefore, for all $n\ge n_0$, there exist $\tilde{\alpha}>0,\ \tilde{\beta}\in\mb{R}$ such that
        \begin{align*}
            \mb{E}[V(Y_{n+1})|\mathcal{F}_n]&\le V(Y_n)(1-\alpha \gamma_{n+1}+C(2\gamma_{n+1}^2+\gamma_{n+1}^3+\gamma_{n+1}^4))+\gamma_{n+1}(\beta+C)\\
            &\le V(Y_n)(1-\tilde{\alpha}\gamma_{n+1})+\tilde{\beta}\gamma_{n+1}
        \end{align*}
        and $1-\tilde{\alpha}\gamma_{n+1}>0$. This leads to 
        \begin{align*}
            \mb{E}[V(Y_{n+1})]\le \mb{E}[V(Y_n)](1-\tilde{\alpha}\gamma_{n+1})+\tilde{\beta}\gamma_{n+1}
        \end{align*}
        We could use induction to prove:
        \[
        \sup_{n\ge n_0}\mb{E}[V(Y_n)]\le \frac{\tilde{\beta}}{\tilde{\alpha}}\vee \mb{E}[V(Y_{n_0})]
        \]
        Assume now $p> 1$. Due to $(\mathcal{L}_{V,p})$-b), we derive that $\sqrt{V}$ is Lipschitz with parameter $[\sqrt{V}]_1$. Consequently, 
        \begin{align*}
            V^{p-1}(\xi_{n+1})&=\sqrt{V}^{2(p-1)}(\xi_{n+1})\le \left(\sqrt{V}(Y_n)+[\sqrt{V}]_1|Y_{n+1}-Y_n|\right)^{2(p-1)} \\
            &\le \left\{
            \begin{aligned}
                &V^{p-1}(Y_n)+([\sqrt{V}]_1|Y_{n+1}-Y_n|)^{2(p-1)}\ \ \ \ \ \ \ \ \ \ \ \ \ \ \ \ \ \ \ \ \ \ \ \ \ \ \ \ \ \ \ \ \ \ \ \ \ \text{if }2(p-1)\le 1,\\
                &V^{p-1}(Y_n)+C\left(V^{(2p-3)/2}(Y_n)|Y_{n+1}-Y_n|+|Y_{n+1}-Y_n|^{2(p-1)}\right)\ \ \ \text{if }2(p-1)>1.
            \end{aligned}
            \right.
        \end{align*}
       Using the fact we've proved in part a): 
        \begin{align*}
            |Y_{n+1}-Y_{n}|^2&\le C_2\left[\gamma_{n+1}^2 V(Y_n)+\gamma_{n+1}(|U_{n+1}^{(1)}|^2+|U_{n+1}^{(2)}|^2+1)\right]
        \end{align*}
        We derive
        \begin{align*}
            V^{p-1}(\xi_{n+1})|Y_{n+1}-Y_n|^2\le V^{p-1}(Y_n)|Y_{n+1}-Y_n|^2+C\gamma_{n+1}^{p\wedge \frac{3}{2}}V^p(Y_n)(1+|U_{n+1}^{(1)}|^{2p}+|U_{n+1}^{(2)}|^{2p})
        \end{align*}
        Then we take conditional expectation:
        \begin{align*}
            \mb{E}[V^p(Y_{n+1})|\mathcal{F}_n]&\le V^p(Y_n)+pV^{p-1}\frac{1-e^{-2\gamma_{n+1}}}{2}\langle \nabla V(x_n,v_n),b(x_n,v_n)\rangle \\
            &\ -pV^{p-1}(Y_n) \frac{u}{2}(\gamma_{n+1}-\frac{1-e^{-2\gamma_{n+1}}}{2})\nabla_x V(x_n,v_n)\cdot \nabla f(x_n)\\
            &\ +C V^{p-1}(Y_n)\left[\gamma_{n+1}^2 V(Y_n)+\gamma_{n+1}(|U_{n+1}^{(1)}|^2+|U_{n+1}^{(2)}|^2+1)\right]\\
            &\ +C\gamma_{n+1}^{p\wedge \frac{3}{2}}V^p(Y_n)(1+|U_{n+1}^{(1)}|^{2p}+|U_{n+1}^{(2)}|^{2p})
        \end{align*}
         There exists $n_0\in \mb{N}$ such that for all $n\ge n_0$
        \begin{align*}
           &\frac{1-e^{-2\gamma_{n+1}}}{2} \langle \nabla V(x_n,v_n), b(x_n,v_n)\rangle \le \gamma_{n+1}(-\alpha V(Y_n)+\beta),\ \ \ \ \text{for some }\alpha>0,\beta\in \mb{R};\\
           &-\frac{u}{2}(\gamma_{n+1}-\frac{1-e^{-2\gamma_{n+1}}}{2})\nabla_x V(x_n,v_n)\cdot \nabla f(x_n)\le C \gamma_{n+1}^2(|\nabla V(Y_n)|^2+|b(Y_n)|^2)\le C\gamma_{n+1}^2V(Y_n).
        \end{align*}
        Since $\gamma_n^{p\wedge \frac{3}{2}},\gamma_n^2\sim o(\gamma_n)$, there exists $\tilde{\alpha}>0,\tilde{\beta}\in \mb{R}$, such that for all $n\ge n_0$:
        \begin{align*}
            \mb{E}[V^p(Y_n)|\mathcal{F}_n]\le V^p(Y_n)+\gamma_{n+1}V^{p-1}(Y_n)(\tilde{\beta}-\tilde{\alpha}V(Y_n))
        \end{align*}
        Same as the proof for $p=1$, we can show
        \[
        \sup_{n\in \mb{N}} \mb{E}[V^p(Y_n)]<+\infty
        \]
\end{proof}
    
\begin{theorem}\label{intermediatethm}
 Let $p\in [0,+\infty)$. Assume $(\mathcal{L}_{V,p})$, If there exists $s\in (0,1]$ such that
\[
\sum_{n\ge 1} \frac{1}{H_n}(\Delta\frac{\eta_n}{\gamma_n})_+<+\infty\ \ \text{and}\ \ \sum_{n\ge 1}(\frac{\eta_n}{H_n\sqrt{\gamma_n}})^{1+s}<+\infty
\]
then 
\[
\mb{P}(d\omega)-a.s\ \ \ \ \sup_{n\in\mb{N}}\nu_n^{\eta}(\omega,V^{p/(1+s)})<+\infty
\]
\end{theorem}
Based on Lemma~\ref{tightness}, the proof of Theorem~\ref{intermediatethm} immediately follows, by using the same steps in the proof of Theorem 4 in~\cite{lamberton2002recursive}. Hence we don't replicate the proof here.\\
\item [3)] (Identification of the limit)
\begin{theorem}[Echeverrr\'{i}a-Weiss Theorem]\label{EWthm}
 Let $E$ be a locally compact Polish space and $\mathcal{L}$ a linear operator satisfying the positive maximum principle. Assume that its domain $\mathcal{D}(A)$ is an algebra everywhere dense in $(\mathcal{C}_0(E),\lv\ \rv_{\infty})$ containing a sequence $(f_n)_{n\in\mb{N}}$ satisfying 
\[
\sup_{n\in\mb{N}}\left(\lv f_n \rv_{\infty}+\lv \mathcal{L}f_n \rv_{\infty} \right)<+\infty,\ \ \forall x\in E,\ \ f_n(x)\to 1\ \ \text{and}\ \ \mathcal{L}f_n(x)\to 0.
\]
If a distribution on $(E,\mathcal{B}(E))$ satisfies $\int_E \mathcal{L}f d\nu =0$ for every $f\in \mathcal{D}(A)$, then there exists a stationary solution for the martingale problem $(\mathcal{L}, \nu)$ (this means that there exists a stationary continuous-time homogeneous Markov process with infinitesimal generator $\mathcal{L}$ and invariant distribution $\nu$). 
\end{theorem}
\begin{lemma}\label{intermediate2}
If the potential function $f$ is Gradient Lipschitz and strongly convex, then the generator of kinetic, $\mathcal{L}$, satisfies the assumptions of the Echeverrr\'{i}a-Weiss theorem. 
\end{lemma}
\begin{proof}[Proof of lemma~\ref{intermediate2}] 
First it's well-known that the infinitesimal generator of a Fellerian semigroup satisfies the maximum principle. We can choose our $f_n(y)=\phi(y/n)$ for any $y\in\mb{R}^{2d}$ where $\phi$ is $\mathcal{C}^2$ with compact support and $\phi(0)=1$. It's easy to check that $\forall y\in\mb{R}^{2d}$, $f_n(y)\to 0$ and $\mathcal{L} f_n(y)\to 0$. It's also straightforward that $\sup_{n\in\mb{N}}\lv f_n \rv_{\infty}<+\infty$. The last thing to check is $\sup_{n\in\mb{N}}\lv \mathcal{L}f_n \rv_{\infty}<+\infty$. Since $\mathcal{L}$ can also be written as $b(x,v)\cdot [\nabla_x,\ \nabla_v]^T+2u\Delta_v$ and we've shown that under our assumptions on $f$, $(\mathcal{L}_{V,\infty})$ is satisfied, we have the Lyapunov function $V(y)=O(|y|^2)$ and $|b(x,v)|\le C(1+|(x,v)|)$. Therefore we get $\sup_{n\in\mb{N}}\lv \mathcal{L}f_n \rv_{\infty}<+\infty$.
\end{proof}
\end{enumerate}
\begin{theorem} \label{theorem3}
Assume that $f$ is gradient Lipschitz and strongly convex. Assume also
\[
\lim_n \frac{1}{H_n}\sum_{k=1}^n |\Delta\frac{\eta_n}{\gamma_n}|=0\ \ \text{and}\ \ \sum_{n\ge 1} (\frac{\eta_n}{\sqrt{\gamma_n} H_n})^2<+\infty
\]
Let $a\ge \frac{1}{2}$. Assume $\sup_n \nu_n^{\eta}(V^a)<+\infty\ \mb{P}-a.s$. If $a< 1$, assume also that $\sum_{n\ge 1}\eta_n\gamma_n/H_n<+\infty$. Then $\mb{P}-a.s$, every limiting distribution $\nu_{\infty}(\omega, dx)$ of the sequence $(\nu_n^{\eta}(\omega, dx))$ is an invariant distribution of the underdamped Langevin dynamics introduced in the previous section.  
\end{theorem}
\noindent The proof of theorem~\ref{theorem3} follows immediately from Theorem~\ref{EWthm}, lemma~\ref{intermediate2}, lemma~\ref{lemma4} and lemma ~\ref{lemma5}.

\begin{lemma}\label{lemma4} 
Under the assumptions in theorem~\ref{theorem3}, then for every bounded Lipschitz continuous function $g:\mb{R}^{2d}\to \mb{R}$,
\[
\mb{P}-a.s\ \ \lim_n \frac{1}{H_n}\sum_{k=1}^n \frac{\eta_k}{\gamma_k}\mb{E}[g(Y_k)-g(Y_{k-1})|\mathcal{F}_{k-1}]=0
\]
\end{lemma}
\begin{proof}[Proof of lemma~\ref{lemma4}]
 Setting $\eta_0/\gamma_0=0$ gives
\begin{align*}
    \frac{1}{H_n}\sum_{k=1}^n\mb{E}[g(Y_k)-g(Y_{k-1})|\mathcal{F}_{k-1}]=\frac{1}{H_n}\sum_{k=1}^n \frac{\eta_k}{\gamma_k}(g(Y_k)-g(Y_{k-1}))-\frac{1}{H_n}\sum_{k=1}^n \frac{\eta_k}{\gamma_k}\left(g(Y_k)-\mb{E}[g(Y_k)|\mathcal{F}_{k-1}]\right).
\end{align*}
As $g$ is bounded, it follows by lemma 3-b) in\cite{lamberton2002recursive} that 
\[
\mb{P}-a.s\ \ \lim_n \frac{1}{H_n}\sum_{k=1}^n (g(Y_k)-g(Y_{k-1}))=0.
\]
Then 
\[
\frac{1}{H_n}\sum_{k=1}^n \frac{\eta_k}{\gamma_k}\left(g(Y_k)-\mb{E}[g(Y_k)|\mathcal{F}_{k-1}]\right)
\]
will converge to $0$ once the martingale 
\[
M_n^g:=\sum_{k=1}^n \frac{\eta_k}{\gamma_k H_k}\left( g(Y_k)-\mb{E}[g(Y_k)|\mathcal{F}_{k-1}]\right)
\]
converge  a.s in $\mb{R}$. 
\begin{align*}
    \mb{E}\langle M_n^g \rangle_{\infty}&=\sum_{n\ge 1} (\frac{\eta_n}{\gamma_n H_n})^2\lv g(Y_n)-\mb{E}[g(Y_n)|\mathcal{F}_{n-1}] \rv_2^2\le \sum_{n\ge 1} (\frac{\eta_n}{\gamma_n H_n})^2\lv g(Y_n)-g(Y_{n-1}) \rv_2^2\\ 
    &\le [f]_1^2\sum_{n\ge 1} (\frac{\eta_n}{\gamma_n H_n})^2 \lv Y_n-Y_{n-1} \rv_2^2
\end{align*}
Since $(\mathcal{L}_{V,1})$ holds under our assumptions on $f$ and by lemma 2-b)
\[
\lv Y_n-Y_{n-1} \rv_2^2\le C'\mb{E}[\gamma_n^2V(Y_{n-1})+(2d+1)\gamma_n]\le C\gamma_n
\]
Therefore 
\[
\mb{E}\langle M_n^g \rangle_{\infty}\le C\sum_{n\ge 1} (\frac{\eta_n}{\sqrt{\gamma_n} H_n})^2<+\infty
\]
\end{proof}
\begin{lemma}\label{lemma5} 
Under  the assumptions in theorem~\ref{theorem3}, then for every $g\in \mathcal{C}^2(\mb{R}^{2d})$ with compact support,
\[
\lim_n \left( \frac{1}{H_n} \sum_{k=1}^n \frac{\eta_k}{\gamma_k} \mb{E}[g(Y_k)-g(Y_{k-1})|\mathcal{F}_{k-1}]-\nu_n^{\eta}(\mathcal{L}g) \right)=0\ \ a.s
\]
\end{lemma}
\begin{proof}[Proof of lemma~\ref{lemma5}] Setting $R_2(y_1,y_2):= g(y_2)-g(y_1)-\langle \nabla g(y_1), y_2-y_1\rangle-\frac{1}{2}D^2 g(y_1)(y_2-y_1)^{\otimes 2}$, we obtain for every $k\in \mb{N}$,
\begin{align*}
   &\qquad g(Y_k)-g(Y_{k-1}) \\
    &=\langle \nabla g(Y_{k-1}), Y_k-Y_{k-1} \rangle+\frac{1}{2}D^2g(Y_{k-1})(Y_k-Y_{k-1})^{\otimes 2}+R_2(Y_{k-1},Y_k)\\
    &=\nabla_x g(x_{k-1},v_{k-1})\cdot [\frac{1-e^{-2\gamma_{k}}}{2}v_{k-1}-\frac{u}{2}(\gamma_{k}-\frac{1-e^{-2\gamma_{k}}}{2})\nabla f(x_{k-1})+\sqrt{u}\sigma_{k}^{(1)}U_{k}^{(1)}]\\
    &\ +\nabla_v g(x_{k-1},v_{k-1})\cdot [-2 \frac{1-e^{-2\gamma_{k}}}{2} v_{k-1}- u\frac{1-e^{-2\gamma_{k}}}{2}\nabla f(x_{k-1})+2\sqrt{u}\sigma_{k}^{(2)}U_{k}^{(2)}] \\
    &\ +\frac{1}{2}D_x^2 g(x_{k-1},v_{k-1})[\frac{1-e^{-2\gamma_{k}}}{2}v_{k-1}-\frac{u}{2}(\gamma_{k}-\frac{1-e^{-2\gamma_{k}}}{2})\nabla f(x_{k-1})+\sqrt{u}\sigma_{k}^{(1)}U_{k}^{(1)}]^{\otimes 2} \\
    &\ +\frac{1}{2}D_v^2 g(x_{k-1},v_{k-1})[-2 \frac{1-e^{-2\gamma_{k}}}{2} v_{k-1}- u\frac{1-e^{-2\gamma_{k}}}{2}\nabla f(x_{k-1})+2\sqrt{u}\sigma_{k}^{(2)}U_{k}^{(2)}]^{\otimes 2} \\
    &\ +\langle D_{xv}g(x_{k-1},v_{k-1}); \frac{1-e^{-2\gamma_{k}}}{2}v_{k-1}-\frac{u}{2}(\gamma_{k}-\frac{1-e^{-2\gamma_{k}}}{2})\nabla f(x_{k-1})+\sqrt{u}\sigma_{k}^{(1)}U_{k}^{(1)},\\
    &\ \ \ \ \ \ \ \ \ \ \ \ \ \ \ \ \ \ \ \ \ \ \ \ \ \ \ \ \ \ \  -2 \frac{1-e^{-2\gamma_{k}}}{2} v_{k-1}- u\frac{1-e^{-2\gamma_{k}}}{2}\nabla f(x_{k-1})+2\sqrt{u}\sigma_{k}^{(2)}U_{k}^{(2)}\rangle \\
    &\ +R_2(Y_{k-1},Y_k) \\
    &=\gamma_k \mathcal{L}g(Y_{k-1})-(\gamma_k-\frac{1-e^{-2\gamma_k}}{2})\nabla_x g(Y_{k-1})\cdot v_{k-1}-\frac{u}{2}(\gamma_k-\frac{1-e^{-2\gamma_k}}{2})\nabla_x g(Y_{k-1})\cdot\nabla f(x_{k-1})\\
    &\ +2(\gamma_k-\frac{1-e^{-2\gamma_k}}{2})\nabla_v g(Y_{k-1})\cdot v_{k-1}+u(\gamma_k-\frac{1-e^{-2\gamma_k}}{2})\nabla_v g(Y_{k-1})\cdot \nabla f(x_{k-1})\\
    &\ +\sqrt{u}\sigma_k^{(1)}\nabla g(Y_{k-1})\cdot U_k^{(1)}+2\sqrt{u}\sigma_k^{(2)}\nabla g(Y_{k-1})\cdot U_k^{(2)}\\
    &\ +\frac{1}{2}(\frac{1-e^{-2\gamma_k}}{2})^2D_x^2g(Y_{k-1})v_{k-1}^{\otimes 2}+\frac{u^2}{8}(\gamma_k-\frac{1-e^{-2\gamma_k}}{2})^2D_x^2g(Y_{k-1})\nabla f(x_{k-1})^{\otimes 2}\\
    &\ +\frac{u}{2}{\sigma_k^{(1)}}^2D_x^2g(Y_{k-1}){U_k^{(1)}}^{\otimes 2}-\frac{u}{2}\frac{1-e^{-2\gamma_k}}{2}(\gamma_k-\frac{1-e^{-2\gamma_k}}{2})\langle D_x^2g(Y_{k-1}); v_{k-1}, \nabla f(x_{k-1}) \rangle\\
    &\ +\sqrt{u}\frac{1-e^{-2\gamma_k}}{2}\sigma_k^{(1)}\langle D_x^2g(Y_{k-1}); v_{k-1},U_k^{(1)} \rangle\\
    &\ -\frac{u^{3/2}}{2}(\gamma_k-\frac{1-e^{-2\gamma_k}}{2})\sigma_k^{(1)}\langle D_x^2g(Y_{k-1});\nabla f(x_{k-1}), U_k^{(1)} \rangle\\
    &\ +2(\frac{1-e^{-2\gamma_k}}{2})^2D_v^2g(Y_{k-1})v_{k-1}^{\otimes 2}+\frac{u^2}{2}(\frac{1-e^{-2\gamma_k}}{2})^2D_v^2g(Y_{k-1})\nabla f(x_{k-1})^{\otimes 2}\\
    &\ +2u\left( {\sigma_k^{(2)}}^2D_v^2g(Y_{k-1}){U_k^{(2)}}^{\otimes 2}-\gamma_k\mb{E}[D_v^2g(Y_{k-1}){U_k^{(2)}}^{\otimes 2}|\mathcal{F}_{k-1}] \right)\\
    &\ +2u(\frac{1-e^{-2\gamma_k}}{2})^2\langle D_v^2g(Y_{k-1});v_{k-1},\nabla f(x_{k-1}) \rangle-4\sqrt{u} \frac{1-e^{-2\gamma_k}}{2}\sigma_k^{(2)}\langle D_v^2g(Y_{k-1}); v_{k-1},U_k^{(2)} \rangle\\
    &\ -2u^{3/2}\frac{1-e^{-2\gamma_k}}{2}\sigma_k^{(2)}\langle D_v^2g(Y_{k-1});\nabla f(x_{k-1}),U_k^{(2)} \rangle+R_2(Y_{k-1},Y_k)
\end{align*}
Take conditional expectation:
\begin{align*}
   &\qquad \mb{E}[g(Y_{k})-g(Y_{k-1})|\mathcal{F}_{k-1}]-\gamma_k \mathcal{L}g(Y_{k-1})\\&=-(\gamma_k-\frac{1-e^{-2\gamma_k}}{2})\nabla_x g(Y_{k-1})\cdot v_{k-1}\\
    &\ -\frac{u}{2}(\gamma_k-\frac{1-e^{-2\gamma_k}}{2})\nabla_x g(Y_{k-1})\cdot\nabla f(x_{k-1})\\
    &\ +2(\gamma_k-\frac{1-e^{-2\gamma_k}}{2})\nabla_v g(Y_{k-1})\cdot v_{k-1}\\
    &\ +u(\gamma_k-\frac{1-e^{-2\gamma_k}}{2})\nabla_v g(Y_{k-1})\cdot \nabla f(x_{k-1})\\
    &\ +\frac{1}{2}(\frac{1-e^{-2\gamma_k}}{2})^2D_x^2g(Y_{k-1})v_{k-1}^{\otimes 2}\\
    &\ +\frac{u^2}{8}(\gamma_k-\frac{1-e^{-2\gamma_k}}{2})^2D_x^2g(Y_{k-1})\nabla f(x_{k-1})^{\otimes 2}\\
    &\ +\frac{u}{2}{\sigma_k^{(1)}}^2\Delta_x g(Y_{k-1})\\
    &\ -\frac{u}{2}\frac{1-e^{-2\gamma_k}}{2}(\gamma_k-\frac{1-e^{-2\gamma_k}}{2})\langle D_x^2g(Y_{k-1}); v_{k-1}, \nabla f(x_{k-1}) \rangle\\
    &\ +2(\frac{1-e^{-2\gamma_k}}{2})^2D_v^2g(Y_{k-1})v_{k-1}^{\otimes 2}+\frac{u^2}{2}(\frac{1-e^{-2\gamma_k}}{2})^2D_v^2g(Y_{k-1})\nabla f(x_{k-1})^{\otimes 2}\\
    &\ +2u\left( {\sigma_k^{(2)}}^2-\gamma_k\right)\Delta_v g(Y_{k-1})\\
    &\ +2u(\frac{1-e^{-2\gamma_k}}{2})^2\langle D_v^2g(Y_{k-1});v_{k-1},\nabla f(x_{k-1}) \rangle\\
    &\ +\mb{E}[R_2(Y_{k-1},Y_k)|\mathcal{F}_{k-1}]
\end{align*}
Observe that for all the terms, except for $R_2(Y_{k-1},Y_{k})$, on the right hand side of the equation, their coefficients are of order $O(\gamma_k^2)$ or $o(\gamma_k^2)$. Furthermore, $\nabla g$ and $D^2 g$ are bounded because $g$ is $\mathcal{C}^2$ and compact supported. Since $(\mathcal{L}_{V,\infty})$ is satisfied under our assumptions, $\sup_{n\in\mb{N}} \mb{E} [ |v_{n}|^2+|\nabla f(x_{n})|^2 ]<C\sup_{n\in\mb{N}}\mb{E}[V(Y_n))]<+\infty$.
Therefore, we obtain that as $n\to 0$,
\begin{align*}
    \frac{1}{H_n}\sum_{k=1}^n \frac{\eta_k}{\gamma_k}\mb{E}[g(Y_{k})-g(Y_{k-1})|\mathcal{F}_{k-1}]-\eta_k\mathcal{L}g(Y_{k-1})-\frac{\eta_k}{\gamma_k}\mb{E}[R_2(Y_{k-1},Y_k)|\mathcal{F}_{k-1}]\to 0
\end{align*}
because $\frac{1}{H_n}\sum_{k=1}^n {\eta_k}{\gamma_k}\to 0$ as $n\to 0$. \\
Now we deal with $\mb{E}[R_2(Y_{k-1},Y_k)|\mathcal{F}_{k-1}]$. For any $x,y\in \mb{R}^{2d}$, define
\[
r_2(x,y):=\frac{1}{2}\sup_{t\in (0,1)}\lv D^2 g(x+t(y-x))-D^2 g(x) \rv
\]
It's easy to see that $r_2$ is a bounded continuous function on $\mb{R}^d\times\mb{R}^d$, $r_2(x,x)=0$ and
\[
|R_2(x,y)|\le r_2(x,y)|x-y|^2
\]
Therefore we obtain
\begin{align*}
    \frac{\eta_k}{\gamma_k}|\mb{E}[R_2(Y_{k-1},Y_k)|\mathcal{F}_{k-1}]|&\le C\left(\eta_k\gamma_k \lv r_2\rv_{\infty} V(Y_{k-1})+(2d+1)\eta_k \mb{E}[r_2(Y_{k-1},Y_k)(|U_k^{(1)}|^2+|U_k^{(2)}|^2)|\mathcal{F}_{k-1}]\right)
\end{align*}
If $a\ge 1$, $\mb{P}-a.s$
\begin{align*}
    \frac{1}{H_n}\sum_{k=1}^n C \eta_k\gamma_k \lv r_2\rv_{\infty}V(Y_{k-1})\le C'\frac{1}{H_n}\sum_{k=1}^n  \eta_k\gamma_k V(Y_{k-1})\to 0\ \ \ \ \text{as }\sup_{n\in\mb{N}}\nu_n^{\eta}(V)<+\infty\ \text{and }\gamma_n\to 0
\end{align*}
If $a\in [1/2,1)$, the same limit follows from the Kronecker lemma mentioned in\cite{lamberton2002recursive} and $$\sum_{n\ge 1}\eta_n\gamma_n/H_n <+\infty.$$\\
Meanwhile, we also have
\begin{align*}
    J(\gamma,x,v)&=\int_{\mb{R}^d\times\mb{R}^d} r_2((x,v), (x',v'))(|r_1|^2+|r_2|^2) \mu(dr_1,dr_2)\\
    \textbf{where }& \\
    (x',v')&=(x+\frac{1-e^{-2\gamma}}{2}v-\frac{u}{2}(\gamma-\frac{1-e^{-2\gamma}}{2})\nabla f(x)+\sqrt{u}\sigma^{(1)}r_1, e^{-2\gamma}v-u\frac{1-e^{-2\gamma}}{2}\nabla f(x)+2\sqrt{u}\sigma^{(2)}r_2)\\
    \sigma^{(1)}&=\left(\gamma+\frac{1-e^{-4\gamma}}{4}-(1-e^{-2\gamma})\right)^{1/2},\ \ \sigma^{(2)}=\left(\frac{1-e^{-4\gamma}}{4}\right)^{1/2}\\
    \text{and }&(U^{(1)},U^{(2)})\sim \mu=\mathcal{N}(0, \frac{1+e^{-4\gamma}-2e^{-2\gamma}}{4\sigma^{(1)}\sigma^{(2)}}I_{2d})
\end{align*}
We can see that $J$ is a bounded continuous function on $\mb{R}_+\times\mb{R}^d\times\mb{R}^d$ and $J(0,x,v)=0$. Since $\lim_{|y|\to \infty}V(y)=+\infty $. We can also write 
\begin{align*}
    (2d+1)\eta_k \mb{E}[r_2(Y_{k-1},Y_k)(|U_k^{(1)}|^2+|U_k^{(2)}|^2)|\mathcal{F}_{k-1}]&= \eta_k V^a((x_{k-1},v_{k-1}))\theta((x_{k-1},v_{k-1}))J(\gamma_k,x_{k-1},v_{k-1})\\
\end{align*}
where $\lim_{|(x_{k-1},v_{k-1})|\to \infty} \theta((x_{k-1},v_{k-1}))=0$
It remains to show that
\[
\mb{P}-a.s\ \ \ \lim_n \frac{1}{H_n}\sum_{k=1}^n \eta_k V^a((x_{k-1},v_{k-1}))\theta((x_{k-1},v_{k-1}))J(\gamma_k,x_{k-1},v_{k-1}) =0
\]
For a fixed number $A>0$, $J$ is uniformly continuous on $[0, \sup_n \gamma_n]\times \bar{B}_{2d}(0,A)$, then $$J(\gamma_k,x_{k-1},v_{k-1})1_{|(x_{k-1},v_{k-1})|\le A}\to 0\ \ \mb{P}-a.s.$$ 
And $V^a((x_{k-1},v_{k-1}))\theta((x_{k-1},v_{k-1}))$ is bounded on $\bar{B}_{2d}(0,A)$. Therefore
\[
\mb{P}-a,s\ \ \ \lim_{n} \frac{1}{H_n}\sum_{k=1}^n \eta_k V^a((x_{k-1},v_{k-1}))\theta((x_{k-1},v_{k-1}))J(\gamma_k,x_{k-1},v_{k-1})1_{|(x_{k-1},v_{k-1})|\le A}=0
\]
On the other hand side
\begin{align*}
    \limsup_n &\frac{1}{H_n}\sum_{k=1}^n \eta_k V^a((x_{k-1},v_{k-1}))\theta((x_{k-1},v_{k-1}))J(\gamma_k,x_{k-1},v_{k-1})1_{|(x_{k-1},v_{k-1})|> A}\\
   &\  \le \sup_{|(x,v)|>A} |\theta(x,v)|\lv J \rv_{\infty} \sup_n \nu_n^\eta(V^a) \to 0 \ \ \ \ \textbf{as }A\to +\infty
\end{align*}
So taking $A\to +\infty$ completes the proof.
\end{proof}

\begin{theorem}\label{theorem4} 
Let $p\in [1,+\infty)$. Assume $(\mathcal{L}_{V,p})$. Let $s\in (0,1]$. Assume that 
\[
\sum_{n\ge 1}\frac{1}{H_n}\left( \Delta \frac{\eta_n}{\gamma_n} \right)_+<+\infty.\ \ \lim_{n}\frac{1}{H_n}\sum_{k=1}^n |\Delta \frac{\eta_k}{\gamma_k}|=0\ \text{and }\sum_{n\ge 1} \left( \frac{\eta_n}{H_n\sqrt{\gamma_n}} \right)^{1+s}<+\infty
\]
(a) Then
\[
\mb{P}-a.s\ \ \ \sup_{n\in\mb{N}}\nu_n^{\eta}(\omega, V^{p/(1+s)})<+\infty.
\]
(b) When $p\le 1+s$, assume also $\sum_{n\ge 1}\eta_n\gamma_n/H_n<+\infty$. Then with probability $1$, any weak limit of the sequence $(\nu_m^{\eta})$ is an invariant distribution of the underdamped Langevin dynamics.
\end{theorem}
Theorem~\ref{theorem4} follows directly from theorem~\ref{intermediatethm} and theorem~\ref{theorem3}.

\begin{proof}[Proof of Theorem~\ref{1dCLTKLMC}] First we try to decompose $\sum_{k=1}^n \gamma_k \mathcal{L}\phi(x_{k-1})$ using Taylor expansion.
\begin{align*}
    \phi(x_k)&=\phi(x_{k-1})+\nabla \phi(x_{k-1})\cdot (x_k-x_{k-1})+\frac{1}{2}D^2\phi(x_{k-1})(x_k-x_{k-1})^{\otimes 2}+R_2^{(k)}
\end{align*}
where $R_2^{(k)}=\phi(x_k)-\phi(x_{k-1})-\nabla \phi(x_{k-1})\cdot (x_k-x_{k-1})-\frac{1}{2}D^2\phi(x_{k-1})(x_k-x_{k-1})^{\otimes 2}$. We can plug our discretization into the equation and obtain:
\begin{align*}
    \phi(x_k)-\phi(x_{k-1})&=\gamma_k\mathcal{L}\phi(x_{k-1})-(\gamma_k-\frac{1-e^{-2\gamma_k}}{2})v_{k-1}\cdot \nabla \phi(x_{k-1})\\
    &\ -\frac{u}{2}(\gamma_k-\frac{1-e^{-2\gamma_k}}{2})\nabla f(x_{k-1})\cdot \nabla \phi(x_{k-1})+\sqrt{u}\sigma_k^{(1)}\nabla \phi(x_{k-1})\cdot U_k^{(1)} \\
    &\ +\frac{1}{2}(\frac{1-e^{-2\gamma_k}}{2})^2D^2 \phi(x_{k-1})v_{k-1}^{\otimes 2}+\frac{u}{2}{\sigma_k^{(1)}}^2D^2\phi(x_{k-1}){U_k^{(1)}}^{\otimes 2}\\
    &\ +\frac{u^2}{8}(\gamma_k-\frac{1-e^{-2\gamma_k}}{2})^2D^2\phi(x_{k-1})\nabla f(x_{k-1})^{\otimes 2} \\
    &\ -\frac{u}{2}\frac{1-e^{-2\gamma_k}}{2}(\gamma_k-\frac{1-e^{-2\gamma_k}}{2})\langle D^2\phi(x_{k-1}); v_{k-1}, \nabla f(x_{k-1}) \rangle\\
    &\ +\sqrt{u}\sigma_k^{(1)}\frac{1-e^{-2\gamma_k}}{2}\langle D^2\phi(x_{k-1}); v_{k-1}, U_k^{(1)} \rangle \\
    &\ -\frac{u^{3/2}}{2}\sigma_k^{(1)}(\gamma_k-\frac{1-e^{-2\gamma_k}}{2})\langle D^2\phi(x_{k-1}); \nabla f(x_{k-1}), U_k^{(1)} \rangle\\
    &\ +R_2^{(k)}
\end{align*}
where 
\begin{align*}
    R_2^{(k)}    &=\frac{1}{6}(\frac{1-e^{-2\gamma_k}}{2})^3D^3\phi(x_{k-1})v_{k-1}^{\otimes 3}-\frac{u}{4}(\frac{1-e^{-2\gamma_k}}{2})^2(\gamma_k-\frac{1-e^{-2\gamma_k}}{2}) \langle D^3\phi(x_{k-1}); v_{k-1}^{\otimes 2}, \nabla f(x_{k-1}) \rangle\\
    &\ +\frac{\sqrt{u}}{2}\sigma_k^{(1)}(\frac{1-e^{-2\gamma_k}}{2})^2\langle D^3\phi(x_{k-1});v_{k-1}^{\otimes 2}, U_k^{(1)} \rangle+\frac{u}{2}{\sigma_k^{(1)}}^2\frac{1-e^{-2\gamma_k}}{2}\langle D^3\phi(x_{k-1}); v_{k-1}, {U_{k}^{(1)}}^{\otimes 2} \rangle\\
    &\ +\frac{1}{24} (\frac{1-e^{-2\gamma_k}}{2})^4 D^4\phi(x_{k-1})v_{k-1}^{\otimes 4} +r^{(k)}
\end{align*}
Since $f$ is gradient Lipschitz and strongly convex, we've shown $(\mathcal{L}_{V,\infty})$ holds. Using $(\mathcal{L}_{V,\infty})$ the fact that $D^4\phi$ is bounded and Lipschitz, we can show there exists a constant $C>0$ such that 
\begin{align*}
    |r_k|\le C \gamma_k^{9/2} V^2(x_{k-1},v_{k-1})
\end{align*}
Apply theorem~\ref{intermediatethm} for $p=4$ and $s=1$, we have $\sup_n \nu_n^\gamma (V^2)<+\infty\ \ \ \mb{P}-a.s$. Therefore 
\begin{align*}
    \frac{1}{\Gamma_n^{(4)}}\sum_{k=1}^n r^{(k)} \to 0\ \ \ \ \ \ \ \ \ \text{in}\ \mb{L}^1
\end{align*}
In the following proof, we will use $o(\gamma_k^4)$ to denote the sum of those terms $b_k$ such that $\frac{1}{\Gamma_n^{(4)}}\sum_{k=1}^n b_k \to 0 \ \ \ \mb{P}-a.s$. According to our decomposition, we can pull out polynomials of $\gamma_k$ from factors $\frac{1-e^{-2\gamma_k}}{2}$, $\gamma_k-\frac{1-e^{-2\gamma_k}}{2}$ and $\sigma_k^{(1)}$ so that the terms left could be included in $o(\gamma_k^{4})$. Then we obtain
\begin{align*}
    \sum_{k=1}^n \gamma_k\mathcal{L}\phi(x_{k-1})&=\sum_{k=1}^n \left\{ [\phi(x_k)-\phi(x_{k-1})]+(\gamma_k^2-\frac{2}{3}\gamma_k^3+\frac{1}{3}\gamma_k^4) v_{k-1}\cdot\nabla \phi(x_{k-1}) \right.\\
    &\ +\frac{u}{2}(\gamma_k^2-\frac{2}{3}\gamma_k^3+\frac{1}{3}\gamma_k^4)\nabla f(x_{k-1})\cdot \nabla \phi(x_{k-1})\\
    &\ -\frac{2\sqrt{3u}}{3}\gamma_k^{\frac{3}{2}}\nabla \phi(x_{k-1})\cdot U_k^{(1)}\\
    &\ -\frac{1}{2}(\gamma_k^2-2\gamma_k^3+\frac{7}{3}\gamma_k^4)D^2\phi(x_{k-1})v_{k-1}^{\otimes 2} \\
    &\ -\frac{u^2}{8}\gamma_k^4 D^2\phi(x_{k-1})\nabla f(x_{k-1})^{\otimes 2} \\
    &\ -\frac{u}{2}(\frac{4}{3}\gamma_k^3-2\gamma_k^4)D^2\phi(x_{k-1}){U_k^{(1)}}^{\otimes 2}\\
    &\ +\frac{u}{2}(\gamma_k^3-\frac{5}{3}\gamma_k^4)\langle D^2\phi(x_{k-1}); v_{k-1},\nabla f(x_{k-1}) \rangle \\
    &\ -\frac{1}{6}(\gamma_k^3-3\gamma_k^4)D^3\phi(x_{k-1})v_{k-1}^{\otimes 3} \\
    &\ +\frac{u}{4}\gamma_k^4\langle D^3\phi(x_{k-1}); v_{k-1}^{\otimes 2} ,\nabla f(x_{k-1})\rangle\\
    &\ -\frac{2u}{3}\gamma_k^4\langle D^3\phi(x_{k-1}); v_{k-1}, {U_k^{(1)}}^{\otimes 2} \rangle\\
    &\ \left. -\frac{1}{24}\gamma_k^4D^4\phi(x_{k-1})v_{k-1}^{\otimes 4}+o(\gamma_k^4)\right\}\\
    &:= Z_n^{(0)}+Z_n^{(2)}+Z_n^{(3)}+Z_n^{(4)}+N_n+r_n
\end{align*}
where 
\begin{align*}
    Z_n^{(0)}&= \phi(x_n)-\phi(x_0)\\
    Z_n^{(2)}&=\sum_{k=1}^n \gamma_k^2\left[v_{k-1}\cdot \nabla \phi(x_{k-1})+\frac{u}{2}\nabla f(x_{k-1})\cdot\nabla \phi(x_{k-1})-\frac{1}{2}D^2\phi(x_{k-1})v_{k-1}^{\otimes 2} \right]\\
    &:=\sum_{k=1}^n\gamma_k^2z_{k-1}^{(2)} \\
    Z_n^{(3)}&=\sum_{k=1}^n \gamma_k^3\left[ -\frac{2}{3}v_{k-1}\cdot\nabla\phi(x_{k-1})-\frac{u}{3}\nabla f(x_{k-1})\cdot\nabla\phi(x_{k-1})+D^2\phi(x_{k-1})v_{k-1}^ 2 \right. \\
    &\ \left. -\frac{2u}{3}D^2\phi(x_{k-1}){U_k^{(1)}}^{\otimes 2}+\frac{u}{2}\langle D^2\phi(x_{k-1});v_{k-1},\nabla f(x_{k-1})\rangle -\frac{1}{6}D^3\phi(x_{k-1})v_{k-1}^{\otimes 3}  \right]\\
     &:=\textcolor{red}{-}\sum_{k=1}^n\gamma_k^3z_{k-1}^{(3)}\\
    Z_n^{(4)}&=\sum_{k=1}^n \gamma_k^4\left[ \frac{1}{3}v_{k-1}\cdot\nabla \phi(x_{k-1})+\frac{u}{6}\nabla f(x_{k-1})\cdot\nabla \phi(x_{k-1})-\frac{7}{6}D^2\phi(x_{k-1})v_{k-1}^{\otimes 2} \right.\\
    &\ \left. -\frac{u^2}{8}D^2\phi(x_{k-1})\nabla f(x_{k-1})^{\otimes 2}+uD^2\phi(x_{k-1}){U_k^{(1)}}^{\otimes 2}-\frac{5u}{6}\langle D^2\phi(x_{k-1}); v_{k-1},\nabla f(x_{k-1}) \rangle \right.\\
    &\ \left. +\frac{1}{2}D^3\phi(x_{k-1})v_{k-1}^{\otimes 3}++\frac{u}{4}\langle D^3\phi(x_{k-1}); v_{k-1}^{\otimes 2} ,\nabla f(x_{k-1})\rangle-\frac{2u}{3}\langle D^3\phi(x_{k-1}); v_{k-1}, {U_k^{(1)}}^{\otimes 2} \right.\\
    &\ \left. -\frac{1}{24}D^4\phi(x_{k-1})v_{k-1}^{\otimes 4} \right] :=\sum_{k=1}^n \gamma_k^4 z_{k-1}^{(4)}\\
    N_n&=\sum_{k=1}^n \frac{2\sqrt{3u}}{3}\gamma_k^{\frac{3}{2}}\nabla \phi(x_{k-1})\cdot U_k^{(1)}\\
    r_n&=\sum_{k=1}^n o(\gamma_k^4)
 \end{align*}
First, it's easy to see that $r_n/\Gamma_n^{(4)} \to 0\ \ \mb{P}-a.s$ as $n\to +\infty$. Apply lemma~\ref{tightness} and we obtain $\sup_n \mb{E}[V(x_n,v_n)]<+\infty$. Therefore we can further obtain the tightness of sequence $\{x_n\}$ and it follows from the continuity of $\phi$ that $\{\phi(x_n)\}$ is also tight. According to the tightness, $Z_n^{(0)}/\Gamma_n^{(4)}\to 0 \ \ \mb{P}-a.s$. For $Z_n^{(4)}$, under our assumptions on $\phi$ and $f$, we can show that 
\begin{align*}
    \lim_{|(x_n,v_n)|\to +\infty} z_n^{(4)}/V^4(x_n,v_n)=0
\end{align*}
Therefor apply theorem~\ref{theorem4} with $p=8,\ s=1$ and we obtain:
\begin{align*}
\mb{P}-a.s\ \ \ \ \ Z_n^{(4)}/\Gamma_n^{(4)}\to &\frac{u}{4}\int_{\mb{R}^{2d}}\langle D^3\phi(x);\nabla f(x),v^{\otimes 2} \rangle \nu(dx,dv)-\frac{u^2}{8}\int_{\mb{R}^d}D^2\phi(x)\nabla f(x)^{\otimes 2}\pi(dx)\\
&\ -\frac{1}{24}\int_{\mb{R}^{2d}}D^4\phi(x)v^{\otimes 4}\nu(dx,dv)  
\end{align*}
To consider the limit of $Z_n^{(i)}/\Gamma_n^{(4)}$ for $i=2,3$, We first Taylor expand $\mathcal{L}\phi(x_{k-1})$ at $x_{k-2}$:
\begin{align*}
    \mathcal{L}\phi(x_{k-1})&=v_{k-2}\cdot\nabla \phi(x_{k-2})+\langle D^2\phi(x_{k-2}); v_{k-2}, x_{k-1}-x_{k-2}\rangle +\nabla \phi(x_{k-2})\cdot(v_{k-1}-v_{k-2})\\
    &\ +\frac{1}{2}\langle D^3\phi(x_{k-2}); v_{k-2}, (x_{k-1}-x_{k-2})^{\otimes 2} \rangle+\langle D^2\phi(x_{k-2}); v_{k-1}-v_{k-2}, x_{k-1}-x_{k-2} \rangle\\
    &\ +\frac{1}{6}\langle D^4\phi(x_{k-2}); v_{k-2}, (x_{k-1}-x_{k-2})^{\otimes 3} \rangle\\
    &\ +\frac{1}{2}\langle D^3\phi(x_{k-2});v_{k-1}-v_{k-2}, (x_{k-1}-x_{k-2})^{\otimes 2} \rangle\\
    &\ +o(\gamma_k^3)\\
\end{align*}
Plug the discretization into the Taylor expansions and preserve the "large" terms, then we obtain:
\begin{align*}
    \mathcal{L}\phi(x_{k-1})&=\mathcal{L}\phi(x_{k-2})+(\gamma_{k-1}-\gamma_{k-1}^2+\frac{2}{3}\gamma_{k-1}^3)D^2\phi(x_{k-2})v_{k-2}^{\otimes 2}\\
    &\ -\frac{u}{2}(\gamma_{k-1}^2-\frac{2}{3}\gamma_{k-1}^3)\langle D^2\phi(x_{k-2});v_{k-2},\nabla f(x_{k-2}) \rangle\\
    &\ +\frac{2\sqrt{3u}}{3}\gamma_{k-1}^{\frac{3}{2}}\langle D^2\phi(x_{k-2}); v_{k-2},U_{k-1}^{(1)} \rangle\\
    &\ -(2\gamma_{k-1}-2\gamma_{k-1}^2+\frac{4}{3}\gamma_{k-1}^3) v_{k-2}\cdot \nabla \phi(x_{k-2}) \\
    &\ -u(\gamma_{k-1}-\gamma_{k-1}^2+\frac{2}{3}\gamma_{k-1}^3)\nabla f(x_{k-2})\cdot\nabla \phi(x_{k-2})\\
    &\ +2\sqrt{u}\gamma_{k-1}^{\frac{1}{2}}\nabla \phi(x_{k-2})\cdot U_{k-1}^{(2)} \\
    &\ +\frac{1}{2}(\gamma_{k-1}^2-2\gamma_{k-1}^3)D^3\phi(x_{k-2})v_{k-2}^{\otimes 3}\\
    &\ +\frac{2u}{3}\gamma_{k-1}^3\langle D^3\phi(x_{k-2}); v_{k-2}, {U_{k-1}^{(1)}}^{\otimes 2} \rangle\\
    &\ +\sqrt{u}\gamma_{k-1}^{\frac{5}{2}}\langle D^3\phi(x_{k-2}); v_{k-2}^{\otimes 2}, U_{k-1}^{(1)} \rangle\\
    &\ -\frac{u}{2}\gamma_{k-1}^{3}\langle D^3\phi(x_{k-2}); v_{k-2}^{\otimes 2}, \nabla f(x_{k-2}) \rangle \\
    &\ -2(\gamma_{k-2}^2-2\gamma_{k-1}^3)D^2\phi(x_{k-2})v_{k-2}^{\otimes 2}\\
    &\ -u(\gamma_{k-1}^2-2\gamma_{k-1}^3)\langle D^2\phi(x_{k-2});v_{k-2},\nabla f(x_{k-2}) \rangle\\
    &\ +2\sqrt{u}\gamma_{k-1}^{\frac{3}{2}}\langle D^2\phi(x_{k-2}); v_{k-2},U_{k-1}^{(2)} \rangle \\
    &\ +u\gamma_{k-1}^3\langle D^2\phi(x_{k-2});v_{k-2},\nabla f(x_{k-2}) \rangle\\
    &\ +\frac{u^2}{2}\gamma_{k-1}^3 D^2\phi(x_{k-2})\nabla f(x_{k-2})^{\otimes 2}\\
    \end{align*}
    \begin{align*}
    &\ -u^{\frac{3}{2}}\gamma_{k-1}^{\frac{5}{2}}\langle D^2\phi(x_{k-2});\nabla f(x_{k-2}),U_{k-1}^{(2)} \rangle\\
    &\ -2\sqrt{u}\gamma_{k-1}^{\frac{5}{2}}\langle D^2\phi(x_{k-2});v_{k-2},U_{k-1}^{(1)} \rangle \\
    &\ -u^{\frac{3}{2}}\gamma_{k-1}^{\frac{5}{2}}\langle D^2\phi(x_{k-2});\nabla f(x_{k-2}),U_{k-1}^{(1)} \rangle\\
    &\ +\langle D^2\phi(x_{k-2});\sqrt{u}\sigma_{k-1}^{(1)}U_{k-1}^{(1)}, 2\sqrt{u}\sigma_{k-1}^{(2)}U_{k-1}^{(2)}\rangle+o(\gamma_{k-1}^3)
\end{align*}
Apply theorem~\ref{theorem4} with $p=4,s=1$ to the terms of order $o(V^2(x_{k-2},v_{k-2}))$ in the decomposition. We obtain
\begin{align*}
    \lim_n \frac{\sum_{k=2}^n \gamma_k\mathcal{L}\phi(x_{k-1})}{\Gamma_n^{(4)}}&=\lim_n \frac{1}{\Gamma_n^{(4)}}\left[ \sum_{k=2}^n \gamma_k\mathcal{L}\phi(x_{k-2})+\sum_{k=2}^n \gamma_k(\gamma_{k-1}-3\gamma_{k-1}^2)D^2\phi(x_{k-2})v_{k-2}^{\otimes 2} \right.\\
    &\ \ \ \ \ \ \ \ \ \ \ \ \ \ \ \ \ \ -\sum_{k=2}^n \frac{3u}{2}\gamma_k \gamma_{k-1}^2 \langle D^2\phi(x_{k-2});\nabla f(x_{k-2}), v_{k-2} \rangle\\
    &\ \ \ \ \ \ \ \ \ \ \ \ \ \ \ \ \ \ -\sum_{k=2}^n \gamma_k(2\gamma_{k-1}-2\gamma_{k-1}^2) \nabla \phi(x_{k-2})\cdot v_{k-2}\\
    &\ \ \ \ \ \ \ \ \ \ \ \ \ \ \ \ \ \ -\sum_{k=2}^n u\gamma_k(\gamma_{k-1}-\gamma_{k-1}^2)\nabla f(x_{k-2})\cdot \nabla \phi(x_{k-2}) \\
    &\ \ \ \ \ \ \ \ \ \ \ \ \ \ \ \ \ \ +\sum_{k=2}^n \frac{1}{2}\gamma_k\gamma_{k-1}^2 D^3\phi(x_{k-2})v_{k-2}^{\otimes 3}\\
    &\ \ \ \ \ \ \ \ \ \ \ \ \ \ \ \ \ \ +\sum_{k=2}^n 2\sqrt{u}\gamma_k\gamma_{k-1}^{\frac{1}{2}}\nabla \phi(x_{k-2})\cdot U_{k-1}^{(2)} \\
    &\ \ \ \ \ \ \ \ \ \ \ \ \ \ \ \ \ \left. +\sum_{k=2}^n \gamma_{k}\langle D^2\phi(x_{k-2}); \sqrt{u}\sigma_{k-1}^{(1)}U_{k-1}^{(1)}, 2\sqrt{u}\sigma_{k-1}^{(2)}U_{k-1}^{(2)} \rangle \right]\\
    &\ +4u\int_{\mb{R}^d} \Delta \phi(x) \pi(dx)-\frac{u}{2}\int_{\mb{R}^{2d}} \langle D^3\phi(x); v^{\otimes 2}, \nabla f(x) \rangle \nu(dx,dv)\\
    &\ +\frac{u^2}{2}\int_{\mb{R}^d} D^2\phi(x) \nabla f(x)^{\otimes 2}\pi(dx)
\end{align*}
Since $\gamma_{k-1}-\gamma_k=o(\gamma_k^4)$, we can substitute all the $\gamma_k$ on the right hand side with $\gamma_{k-1}$ and it won't change the limits. For the last term inside the square bracket, notice that $Var(\sqrt{u}\sigma_{k-1}^{(1)}U_{k-1}^{(1)},2\sqrt{u}\sigma_{k-1}^{(2)}U_{k-1}^{(2)} )=\frac{u}{2}(1+e^{-4\gamma_{k-1}}-2e^{-2\gamma_{k-1}})I_d\sim u(2\gamma_{k-1}^2-4\gamma_{k-1}^3)I_d$. Therefore
\begin{align*}
    \lim_n \frac{1}{\Gamma_n^{(4)}}\sum_{k=2}^n \gamma_{k}\langle D^2\phi(x_{k-2}); \sqrt{u}\sigma_{k-1}^{(1)}U_{k-1}^{(1)}, 2\sqrt{u}\sigma_{k-1}^{(2)}U_{k-1}^{(2)} \rangle&=\lim_n \frac{1}{\Gamma_n^{(4)}}\sum_{k=2}^n 2u\gamma_{k-1}^3\Delta \phi(x_{k-2}) \\
    &\ -4u\int_{\mb{R}^d} \Delta \phi(x) \pi(dx)
\end{align*}
We can rewrite the equation as 
\begin{align*}
    \lim_n \frac{\sum_{k=2}^n \gamma_k\mathcal{L}\phi(x_{k-1})}{\Gamma_n^{(4)}}&=\lim_n \frac{1}{\Gamma_n^{(4)}}\sum_{k=2}^n \gamma_k\mathcal{L}\phi(x_{k-2})+\lim_n \frac{1}{\Gamma_n^{(4)}}\sum_{k=2}^n 2\sqrt{u}\gamma_{k-1}^{\frac{3}{2}}\nabla \phi(x_{k-2})\cdot U_{k-1}^{(2)}\\
    &\ +\lim_n \frac{1}{\Gamma_n^{(4)}}\sum_{k=2}^n \gamma_{k-1}^2[D^2\phi(x_{k-2})v_{k-2}^{\otimes 2}-2\nabla \phi(x_{k-2})\cdot v_{k-2}-u\nabla f(x_{k-2})\cdot\nabla \phi(x_{k-2})]\\
    &\ +\lim_n \frac{1}{\Gamma_n^{(4)}}\sum_{k=2}^n \gamma_{k-1}^3[-3D^\phi(x_{k-2})v_{k-2}^{\otimes 2}-\frac{3u}{2}\langle D^2\phi(x_{k-2}); \nabla f(x_{k-2}), v_{k-2} \rangle\\
    &\ +2\nabla \phi(x_{k-2})\cdot v_{k-2}+u\nabla f(x_{k-2})\cdot\nabla \phi(x_{k-2})+\frac{1}{2}D^3\phi(x_{k-2})v_{k-2}^{\otimes 3}+2u\Delta\phi(x_{k-2})]\\
    &\ -\frac{u}{2}\int_{\mb{R}^{2d}} \langle D^3\phi(x); v^{\otimes 2}, \nabla f(x) \rangle \nu(dx,dv)+\frac{u^2}{2}\int_{\mb{R}^d} D^2\phi(x) \nabla f(x)^{\otimes 2}\pi(dx)\\
    &=\lim_n \frac{1}{\Gamma_n^{(4)}}\sum_{k=2}^n \gamma_k\mathcal{L}\phi(x_{k-2})+\lim_n \frac{1}{\Gamma_n^{(4)}}\sum_{k=2}^n 2\sqrt{u}\gamma_{k-1}^{\frac{3}{2}}\nabla \phi(x_{k-2})\cdot U_{k-1}^{(2)}\\
    &\ +\lim_n \frac{1}{\Gamma_n^{(4)}}(-2Z_n^{(2)}-3Z_n^{(3)})\\
    &\ -\frac{u}{2}\int_{\mb{R}^{2d}} \langle D^3\phi(x); v^{\otimes 2}, \nabla f(x) \rangle \nu(dx,dv)+\frac{u^2}{2}\int_{\mb{R}^d} D^2\phi(x) \nabla f(x)^{\otimes 2}\pi(dx)\\
\end{align*}
We can instantly get that
\begin{align*}
    \lim_n \frac{1}{\Gamma_n^{(4)}}(2Z_2^{(n)}+3Z_n^{(3)})&=\lim_n \frac{1}{\Gamma_n^{(4)}}\sum_{k=2}^n 2\sqrt{u}\gamma_{k-1}^{\frac{3}{2}}\nabla \phi(x_{k-2})\cdot U_{k-1}^{(2)} \\
    &\ +\frac{u^2}{2}\int_{\mb{R}^d} D^2\phi(x)\nabla f(x)^{\otimes 2}\pi(dx)\\
    &\ -\frac{u}{2}\int_{\mb{R}^{2d}} \langle D^3\phi(x); \nabla f(x), v^{\otimes 2} \rangle\nu(dx,dv)
\end{align*}
Similarly, apply Taylor expansion to $z_{k-1}^{(2)}$ at $x_{k-2}$, we achieve:
\begin{align*}
     \nabla f(x_{k-1})\cdot\nabla \phi(x_{k-1})&=\nabla f(x_{k-2})\cdot \nabla \phi(x_{k-2})+\langle D^2f(x_{k-2});\nabla \phi(x_{k-2}), x_{k-1}-x_{k-2}\rangle \\
     &\ +\langle D^2\phi(x_{k-2}); \nabla f(x_{k-2}), x_{k-1}-x_{k-2} \rangle\\
     &\ +\frac{1}{2}D^3(\nabla f \cdot\nabla \phi)(x_{k-2})(x_{k-1}-x_{k-2})^{\otimes 2}+o(\gamma_k^2)
\end{align*}
\begin{align*}
    \frac{1}{2}D^2\phi(x_{k-1})v_{k-1}^{\otimes 2}&=\frac{1}{2}D^2\phi(x_{k-2})v_{k-2}^{\otimes 2}+\frac{1}{2}\langle D^3\phi(x_{k-1}); v_{k-2}^{\otimes 2}, x_{k-1}-x_{k-2}\rangle \\
    &\ +\langle D^2\phi(x_{k-2}); v_{k-2}, v_{k-1}-v_{k-2}\rangle +\frac{1}{2}D^2\phi(x_{k-2})(v_{k-1}-v_{k-2})^{\otimes 2} \\
    &\ +\frac{1}{4}(\frac{1-e^{-2\gamma_{k-1}}}{2})^2D^4\phi(x_{k-2})v_{k-2}^{\otimes 4}\\
    &\ +\frac{1}{2}\langle D^3\phi(x_{k-2}); v_{k-2}, x_{k-1}-x_{k-2}, v_{k-1}-v_{k-2} \rangle\\
    &\ +\frac{1}{6}\langle D^3\phi(x_{k-2}); x_{k-1}-x_{k-2}, (v_{k-1}-v_{k-2})^{\otimes 2}  \rangle+o(\gamma_k^2)
\end{align*}
Simplifying the coefficients lead us to
\begin{align*}
    \nabla f(x_{k-1})\cdot\nabla \phi(x_{k-1})&=\nabla f(x_{k-2})\cdot\nabla \phi(x_{k-2})+(\gamma_{k-1}-\gamma_{k-1}^2)\langle D^2f(x_{k-2}); \nabla \phi(x_{k-2}), v_{k-2} \rangle\\
   &\ -\frac{u}{2}\gamma_{k-1}^2\langle D^2f(x_{k-2}); \nabla \phi(x_{k-2}), \nabla f(x_{k-2}) \rangle\\
    &\ +\frac{2\sqrt{3u}}{3}\gamma_{k-1}^{\frac{3}{2}}\langle D^2f(x_{k-2}); \nabla \phi(x_{k-2}), U_{k-1}^{(1)} \rangle\\
    &\ +(\gamma_{k-1}-\gamma_{k-1}^2)\langle D^2\phi(x_{k-2}); \nabla f(x_{k-2}), v_{k-2} \rangle\\
    &\ -\frac{u}{2}\gamma_{k-1}^2 D^2\phi(x_{k-2})\nabla f(x_{k-2})^{\otimes 2}+\sqrt{u}\gamma_{k-1}^{\frac{3}{2}}\langle D^2\phi(x_{k-2}); \nabla f(x_{k-2}), U_{k-1}^{(1)} \rangle\\ 
    &\ +\frac{1}{2}\gamma_{k-1}^2 (D^3f \nabla \phi+2D^2\phi D^2f+D^3\phi \nabla f)(x_{k-2})v_{k-2}^{\otimes 2}+o(\gamma_{k-1}^2)
\end{align*}
\begin{align*}
    \frac{1}{2}D^2\phi(x_{k-1})v_{k-1}^{\otimes 2}&=\frac{1}{2}D^2\phi(x_{k-2})v_{k-2}^{\otimes 2}+\frac{1}{2}(\gamma_{k-1}-\gamma_{k-1}^2)D^3\phi(x_{k-2})v_{k-2}^{\otimes 3}\\
    &\ -\frac{u}{4}\gamma_{k-1}^2\langle D^3\phi(x_{k-2}); v_{k-2}^{\otimes 2}, \nabla f(x_{k-2}) \rangle+\frac{\sqrt{u}}{2}\gamma_{k-1}^{\frac{3}{2}}\langle D^3\phi(x_{k-2}); v_{k-2}^{\otimes 2}, U_{k-1}^{(1)} \rangle \\
    &\ -2(\gamma_{k-1}-\gamma_{k-1}^2)D^2\phi(x_{k-2})v_{k-2}^{\otimes 2}-u(\gamma_{k-1}-\gamma_{k-1}^2)\langle D^2\phi(x_{k-2});\nabla f(x_{k-2}), v_{k-2} \rangle\\
    &\ +2\sqrt{u}\gamma_{k-1}^{\frac{1}{2}}\langle D^2\phi(x_{k-2}); v_{k-2}, U_{k-1}^{(2)} \rangle+2\gamma_{k-1}^2D^2\phi(x_{k-2})v_{k-2}^{\otimes 2}\\
    &\ +\frac{u^2}{2}\gamma_{k-1}^2 D^2\phi(x_{k-2})\nabla f(x_{k-2})^{\otimes 2}+2u(\gamma_{k-1}-2\gamma_{k-1}^2)D^2\phi(x_{k-2}){U_{k-1}^{(2)}}^{\otimes 2}\\
    &\ +2u\gamma_{k-1}^2\langle D^2\phi(x_{k-2});\nabla f(x_{k-2}), v_{k-2} \rangle-4\sqrt{u}\gamma_{k-1}^{\frac{3}{2}}\langle D^2\phi(x_{k-2}); v_{k-2}, U_{k-1}^{(2)} \rangle\\
   &\ -2u^{\frac{3}{2}}\gamma_{k-1}^{\frac{3}{2}}\langle D^2\phi(x_{k-2}); \nabla f(x_{k-2}), U_{k-1}^{(2)} \rangle+\frac{1}{4}\gamma_{k-1}^2D^4\phi(x_{k-2})v_{k-2}^{\otimes 4}\\
    &\ -\gamma_{k-1}^2D^3\phi(x_{k-2})v_{k-2}^{\otimes 3}-\frac{u}{2}\gamma_{k-1}^2\langle D^3\phi(x_{k-2}); v_{k-2}^{\otimes 2},\nabla f(x_{k-2}) \rangle\\
    &\ +\frac{1}{2}\langle D^3\phi(x_{k-2}); v_{k-2}, \sqrt{u}\sigma_{k-1}^{(1)}U_{k-1}^{(1)}, 2\sqrt{u}\sigma_{k-1}^{(2)}U_{k-1}^{(12} \rangle\\
    &\ +\frac{2u}{3}\gamma_{k-1}^2\langle D^3\phi(x_{k-1}); v_{k-2}; {U_{k-1}^{(2)}}^{\otimes 2} \rangle+o(\gamma_k^2) \\
\end{align*}
Take the limits and we obtain:
\begin{align*}
    \lim_n \frac{\sum_{k=2}^n\gamma_k^2 \nabla f(x_{k-1})\cdot\nabla \phi(x_{k-1})}{\Gamma_n^{(4)}}&=\lim_n \frac{1}{\Gamma_n^{(4)}}\left[\ \sum_{k=2}^n \gamma_{k-1}^2 \nabla f(x_{k-2})\cdot \nabla \phi(x_{k-2}) \right. \\
    &\ \ \ \ \ \ \ \ \ \ \ \ \ +\sum_{k=2}^n \gamma_{k-1}^3 \langle D^2f(x_{k-2}); \nabla \phi(x_{k-2}), v_{k-2} \rangle\\
    &\ \ \ \ \ \ \ \ \ \ \ \ \left. +\sum_{k=2}^n \gamma_{k-1}^3 \langle D^2\phi(x_{k-2}); \nabla f(x_{k-2}), v_{k-2} \rangle \right]\\
    &\ \ \ \ \ \ \ \ \ \ \ \ -\frac{u}{2}\int_{\mb{R}^d} \langle D^2 f(x); \nabla \phi(x), \nabla f(x)\rangle \pi(dx) \\
    &\ \ \ \ \ \ \ \ \ \ \ \ -\frac{u}{2}\int_{\mb{R}^d} D^2 \phi(x)\nabla f(x)^{\otimes 2} \pi(dx) \\
    +\frac{1}{2}& \int_{\mb{R}^d}\left( D^3f(x)\nabla \phi(x)+2D^2f(x)D^2\phi(x)+D^3\phi(x)\nabla f(x)\right)v^{\otimes 2} \nu(dx,dv)
\end{align*}
\begin{align*}
    \lim_n \frac{1}{2}\frac{\sum_{k=2}^n \gamma_k^2D^2\phi(x_{k-1})v_{k-1}^{\otimes 2}}{\Gamma_n^{(4)}}&= \lim_n \left\{ \frac{1}{\Gamma_n^{(4)}}\sum_{k=2}^n\frac{1}{2} \gamma_{k-1}^2D^2\phi(x_{k-2})v_{k-2}^{\otimes 2} \right. \\
    &\ +\frac{1}{\Gamma_n^{(4)}}\sum_{k=2}^n \gamma_{k-1}^3 \left[ \frac{1}{2}D^3\phi(x_{k-2})v_{k-2}^{\otimes 3} -2D^2\phi(x_{k-2})v_{k-2}^{\otimes 2}\right.\\
    &\ \left. \left.-u\langle D^2\phi(x_{k-2}); \nabla f(x_{k-2}), v_{k-2} \rangle+2uD^2\phi(x_{k-2}){U_{k-1}^{(2)}}^{\otimes 2} \right] \right\}\\
    &\ -\frac{3u}{4}\int_{\mb{R}^{2d}}\langle D^3\phi(x); \nabla f(x), v^{\otimes 2} \rangle \nu(dx,dv)+\frac{1}{4}\int_{\mb{R}^{2d}}D^4\phi(x)v^{\otimes 4}\pi(dx)\\
    &\ +\frac{u^2}{2}\int_{\mb{R}^d}D^2\phi(x)\nabla f(x)^{\otimes 2}\pi(dx)\\
\end{align*}
\textbf{Claim:}
\begin{enumerate}[leftmargin=18pt,noitemsep]
    \item [a)] $\lim_{n}\frac{1}{\Gamma_n^{(4)}} \sum_{k=1}^n \gamma_{k}^2 \nabla \phi(x_{k-1})\cdot v_{k-1}=0$.
    \item [b)]  $\lim_n \frac{1}{\Gamma_n^{(4)}}\sum_{k=1}^n \gamma_k^3(\frac{u}{2}\nabla \phi(x_{k-1})\cdot\nabla f(x_{k-1})-\frac{1}{2}D^2\phi(x_{k-1})v_{k-1}^{\otimes 2})=0$.
    \item [c)] $\lim_{n}\frac{1}{\Gamma_n^{(4)}} \sum_{k=1}^n \gamma_{k}^3 \nabla \phi(x_{k-1})\cdot v_{k-1}=0$.
\end{enumerate}
We'll prove the \textbf{Claim} at the end of our proof. We can use the \textbf{Claim} and our expansion of $Z_n^{(2)}$ to find the following relation:
\begin{align*}
    \lim_n \frac{1}{\Gamma_n^{(4)}}Z_n^{(2)}&=\lim_n \frac{1}{\Gamma_n^{(4)}}\sum_{k=2}^n \gamma_k^2\nabla \phi(x_{k-1})\cdot v_{k-1} \\
    &\ +\lim_n \frac{1}{\Gamma_n^{(4)}}\sum_{k=2}^n \gamma_k^2\left[\frac{u}{2}\nabla f(x_{k-1})\cdot\nabla \phi(x_{k-1})-\frac{1}{2}D^2\phi(x_{k-1})v_{k-1}^{\otimes 2} \right]\\
    &=\lim_n \frac{1}{\Gamma_n^{(4)}}\sum_{k=2}^n \gamma_{k-1}^2\left[\frac{u}{2}\nabla f(x_{k-2})\cdot\nabla \phi(x_{k-2})-\frac{1}{2}D^2\phi(x_{k-2})v_{k-2}^{\otimes 2} \right]\\
    &\ +\lim_n \frac{1}{\Gamma_n^{(4)}}\sum_{k=2}^n \gamma_{k-1}^3\left[\frac{u}{2} \langle D^2f(x_{k-2}); \nabla \phi(x_{k-2}), v_{k-2} \rangle+\frac{3u}{2}\langle D^2\phi(x_{k-2}); \nabla f(x_{k-2}), v_{k-2} \rangle \right.\\
    &\ \ \ \ \ \ \ \ \ \ \ \ \ \ \ \ \ \ \ \ \ \ \ \ \ \left. -\frac{1}{2}D^3\phi(x_{k-2})v_{k-2}^{\otimes 3} +2D^2\phi(x_{k-2})v_{k-2}^{\otimes 2}-2u\Delta\phi(x_{k-2}) \right]\\
    &\ -\frac{u^2}{4}\int_{\mb{R}^d} \langle D^2 f(x); \nabla \phi(x), \nabla f(x)\rangle \pi(dx) -\frac{u^2}{4}\int_{\mb{R}^d} D^2 \phi(x)\nabla f(x)^{\otimes 2} \pi(dx) \\
    &\ +\frac{u}{4} \int_{\mb{R}^d}\left( D^3f(x)\nabla \phi(x)+2D^2f(x)D^2\phi(x)+D^3\phi(x)\nabla f(x)\right)v^{\otimes 2} \nu(dx,dv)\\
    &\ +\frac{3u}{4}\int_{\mb{R}^{2d}}\langle D^3\phi(x); \nabla f(x), v^{\otimes 2} \rangle \nu(dx,dv)-\frac{1}{4}\int_{\mb{R}^{2d}}D^4\phi(x)v^{\otimes 4}\pi(dx)\\
    &\ -\frac{u^2}{2}\int_{\mb{R}^d}D^2\phi(x)\nabla f(x)^{\otimes 2}\pi(dx)\\
    &=\lim_n \frac{1}{\Gamma_n^{(4)}}[Z_n^{(2)}+3Z_n^{(3)}]-\frac{u^2}{4}\int_{\mb{R}^d} \langle D^2 f(x); \nabla \phi(x), \nabla f(x)\rangle \pi(dx)\\
    &\ -\frac{3u^2}{4}\int_{\mb{R}^d} D^2 \phi(x)\nabla f(x)^{\otimes 2} \pi(dx)-\frac{1}{4}\int_{\mb{R}^{2d}}D^4\phi(x)v^{\otimes 4}\pi(dx)\\
    &\ +\frac{u}{4} \int_{\mb{R}^d}\left( D^3f(x)\nabla \phi(x)+2D^2f(x)D^2\phi(x)+4D^3\phi(x)\nabla f(x)\right)v^{\otimes 2} \nu(dx,dv)\\
\end{align*}
The last identity follows from \textbf{Claim}-a),b) and the fact that $\lim_n \frac{1}{\Gamma_n^{(4)}}\sum_{k=1}^n \gamma_k^3\langle D^2f(x_{k-1}); \nabla \phi(x_{k-1}), v_{k-1} \rangle=0$. To prove $\lim_n \frac{1}{\Gamma_n^{(4)}}\sum_{k=1}^n \gamma_k^3\langle D^2f(x_{k-1}); \nabla \phi(x_{k-1}), v_{k-1} \rangle=0$, we can assume $\psi$ is a new test function satisfying $\nabla \psi(x)=D^2f(x)\nabla \phi(x)$. Then the statement follows from \textbf{Claim}-c). This could be done because $\psi$ satisfies the all assumptions on $\phi$ stated in the theorem. Therefore we obtain
\begin{align*}
    \lim_n \frac{1}{\Gamma_n^{(4)}}Z_n^{(3)}&= \frac{u^2}{12}\int_{\mb{R}^d} \langle D^2 f(x); \nabla \phi(x), \nabla f(x)\rangle \pi(dx)+\frac{u^2}{4}\int_{\mb{R}^d} D^2 \phi(x)\nabla f(x)^{\otimes 2} \pi(dx) \\
    &\ -\frac{u}{12} \int_{\mb{R}^d}\left( D^3f(x)\nabla \phi(x)+2D^2f(x)D^2\phi(x)+4D^3\phi(x)\nabla f(x)\right)v^{\otimes 2} \nu(dx,dv)\\
    &\ +\frac{1}{12}\int_{\mb{R}^{2d}}D^4\phi(x)v^{\otimes 4}\pi(dx)\\
\end{align*}
Combine with our previous results on $2Z_n^{(2)}+3Z_n^{(3)}$ and we obtain
\begin{align*}
    \lim_n \frac{1}{\Gamma_n^{(4)}}[Z_n^{(2)}+Z_n^{(3)}]&=\lim_n \frac{1}{\Gamma_n^{(4)}}\sum_{k=1}^n \sqrt{u}\gamma_k^{\frac{3}{2}}\nabla \phi(x_{k-1})\cdot {U_{k}^{(2)}}+\frac{u^2}{8}\int_{\mb{R}^d} D^2 \phi(x)\nabla f(x)^{\otimes 2} \pi(dx)\\
    &\ -\frac{u}{12}\int_{\mb{R}^{2d}} \langle D^3\phi(x); \nabla f(x), v^{\otimes 2} \rangle \nu(dx,dv)+\frac{u}{24}\int_{\mb{R}^{2d}} \langle D^3f(x); \nabla \phi(x), v^{\otimes 2} \rangle \nu(dx,dv)\\
    &\ +\frac{u}{12}\int_{\mb{R}^{2d}} (D^2f D^2\phi)(x)v^{\otimes 2}\nu(dx,dv)-\frac{u^2}{24}\int_{\mb{R}^d} \langle D^2 f(x); \nabla \phi(x), \nabla f(x)\rangle \pi(dx)\\
    &\ -\frac{1}{24}\int_{\mb{R}^{2d}}D^4\phi(x)v^{\otimes 4}\pi(dx)\\
\end{align*}
Then we plug this result in our original decomposition:
\begin{align*}
    \lim_n \frac{1}{\Gamma_n^{(4)}}\sum_{k=1}^n \gamma_k \mathcal{L}\phi(x_{k-1})&= \lim_n \frac{1}{\Gamma_n^{(4)}}\sum_{k=1}^n \gamma_k^{\frac{3}{2}} \frac{2\sqrt{3}}{3}\nabla \phi(x_{k-1})\cdot(\sqrt{u}U_{k}^{(1)})\\
     &+\ \lim_n \frac{1}{\Gamma_n^{(4)}}\sum_{k=1}^n \sqrt{u}\gamma_k^{\frac{3}{2}}\nabla \phi(x_{k-1})\cdot {U_{k}^{(2)}}+\frac{u^2}{8}\int_{\mb{R}^d} D^2 \phi(x)\nabla f(x)^{\otimes 2} \pi(dx)\\
    &\ -\frac{u}{12}\int_{\mb{R}^{2d}} \langle D^3\phi(x); \nabla f(x), v^{\otimes 2} \rangle \nu(dx,dv)+\frac{u}{24}\int_{\mb{R}^{2d}} \langle D^3f(x); \nabla \phi(x), v^{\otimes 2} \rangle \nu(dx,dv)\\
    &\ +\frac{u}{12}\int_{\mb{R}^{2d}} (D^2f D^2\phi)(x)v^{\otimes 2}\nu(dx,dv)-\frac{u^2}{24}\int_{\mb{R}^d} \langle D^2 f(x); \nabla \phi(x), \nabla f(x)\rangle \pi(dx)\\
    &\ -\frac{1}{24}\int_{\mb{R}^{2d}}D^4\phi(x)v^{\otimes 4}\pi(dx)+\frac{u}{4}\int_{\mb{R}^{2d}}\langle D^3\phi(x);\nabla f(x),v^{\otimes 2} \rangle \nu(dx,dv)\\
    &\ -\frac{u^2}{8}\int_{\mb{R}^d}D^2\phi(x)\nabla f(x)^{\otimes 2}\pi(dx)-\frac{1}{24}\int_{\mb{R}^{2d}}D^4\phi(x)v^{\otimes 4}\nu(dx,dv) \\
    &= \lim_n \frac{1}{\Gamma_n^{(4)}}\sum_{k=1}^n \gamma_k^{\frac{3}{2}} \nabla \phi(x_{k-1})\cdot
    (\frac{2\sqrt{3}}{3}\sqrt{u}U_k^{(1)}+\frac{1}{2}2\sqrt{u}U_k^{(2)})\\
    &\ +\frac{u}{6}\int_{\mb{R}^{2d}} \langle D^3\phi(x); \nabla f(x), v^{\otimes 2} \rangle\nu(dx,dv)+\frac{u}{24}\int_{\mb{R}^{2d}} \langle D^3f(x); \nabla \phi(x), v^{\otimes 2} \rangle \nu(dx,dv)\\
    &\ +\frac{u}{12}\int_{\mb{R}^{2d}}(D^2\phi D^2f)(x)v^{\otimes 2}\nu(dx,dv)-\frac{1}{12}\int_{\mb{R}^{2d}}D^4\phi(x)v^{\otimes 4}\nu(dx,dv)\\
    &\ -\frac{u^2}{24}\int_{\mb{R}^d} \langle D^2 f(x); \nabla \phi(x), \nabla f(x)\rangle \pi(dx)\\
\end{align*}
 It remains to determine the normal limit. Since $(U_k^{(1)},U_k^{(2)})$ is Gaussian in $\mb{R}^{2d}$ with mean zero and covariance matrix $\frac{1+e^{-4\gamma_k}-2e^{-2\gamma_k}}{4\sigma_k^{(1)}\sigma_k^{(2)}}I_d$, we can find the distribution of $U_k:= (\frac{2\sqrt{3}}{3}\sqrt{u}U_k^{(1)}+\frac{1}{2}2\sqrt{u}U_k^{(2)})$. $\{U_k\}$ are independent $2d$-Gaussian Random vectors with $U_k\sim \mathcal{N}(0, \Sigma_k)$, where
\begin{align*}
    \Sigma_k&=\mb{E}[ (\frac{2\sqrt{3}}{3}\sqrt{u}U_k^{(1)}+\frac{1}{2}2\sqrt{u}U_k^{(2)})^T (\frac{2\sqrt{3}}{3}\sqrt{u}U_k^{(1)}+\frac{1}{2}2\sqrt{u}U_k^{(2)})] \\
    &= \frac{4u}{3}I_d +\frac{4u\sqrt{3}}{3}\frac{1+e^{-4\gamma_k}-2e^{-2\gamma_k}}{4\sigma_k^{(1)}\sigma_k^{(2)}}I_d+uI_d \\
    & \sim \frac{10}{3}uI_d+O(\gamma_k)I_d
\end{align*}
Apply our weak convergence result and CLT for arrays of square-integrable martingale increments, we have that when $0<\hat{\gamma}<+\infty$:  
\begin{align*}
    \frac{1}{\Gamma_n^{(4)}}\sum_{k=1}^n \gamma_{k}^{\frac{3}{2}} \nabla \phi(x_{k-1})\cdot U_k \implies \mathcal{N}(0, \sigma^2)
\end{align*}
where 
\begin{align*}
    \sigma^2&= \lim_n \frac{1}{{\Gamma_n^{(4)}}^2}\sum_{k=1}^n \gamma_k^3 |\nabla \phi(x_{k-1})|^2 (\frac{10}{3}u+O(\gamma_k)) =\frac{10}{3}u{\hat{\gamma}}^{-2}\int_{\mb{R}^d} |\nabla \phi(x)|^2 \pi(dx) 
\end{align*}
In conclusion, when $\hat{\gamma}\in (0,+\infty)$:
\begin{align*}
    \frac{\Gamma_n}{\Gamma_{n}^{(4)}}\nu_n^{\gamma} (\mathcal{L}\phi)\implies \mathcal{N}(\rho, \frac{10}{3}u{\hat{\gamma}}^{-2}\int_{\mb{R}^d} |\nabla \phi(x)|^2 \pi(dx) )
\end{align*}
where
\begin{align*}
    \rho&=\frac{u}{6}\int_{\mb{R}^{2d}} \langle D^3\phi(x); \nabla f(x), v^{\otimes 2} \rangle\nu(dx,dv)+\frac{u}{24}\int_{\mb{R}^{2d}} \langle D^3f(x); \nabla \phi(x), v^{\otimes 2} \rangle \nu(dx,dv)\\
    &\ +\frac{u}{12}\int_{\mb{R}^{2d}}(D^2\phi D^2f)(x)v^{\otimes 2}\nu(dx,dv)-\frac{1}{12}\int_{\mb{R}^{2d}}D^4\phi(x)v^{\otimes 4}\nu(dx,dv)\\
    &\ -\frac{u^2}{24}\int_{\mb{R}^d} \langle D^2 f(x); \nabla \phi(x), \nabla f(x)\rangle \pi(dx)\\
\end{align*}
When $\hat{\gamma}=0$, 
\begin{align*}
    \frac{\Gamma_n}{\sqrt{\Gamma_{n}^{(3)}}}\nu_n^{\gamma} (\mathcal{L}\phi)\implies \mathcal{N}(0, \frac{10}{3}u\int_{\mb{R}^d} |\nabla \phi(x)|^2 \pi(dx) )
\end{align*}
When $\hat{\gamma}=+\infty$, 
\begin{align*}
     \frac{1}{\Gamma_n^{(4)}}\sum_{k=1}^n \gamma_k^{\frac{3}{2}} \nabla \phi(x_{k-1})\cdot
    (\frac{2\sqrt{3}}{3}\sqrt{u}U_k^{(1)}+\frac{1}{2}2\sqrt{u}U_k^{(2)})\to 0\ \ \ \ \ \text{in probability}
\end{align*}
Therefore when $\hat{\gamma}=+\infty$,
\begin{align*}
    \frac{\Gamma_n}{\Gamma_{n}^{(4)}}\nu_n^{\gamma} (\mathcal{L}\phi)\to  \rho\ \ \ \ \ \ \ \ \ \ \ \ \ \ \ \text{in probability}
\end{align*}
\textbf{Proof of the claim:} First we'll show that $\frac{1}{\Gamma_n^{(3)}}\sum_{k=1}^n\gamma_k^2\mathcal{L}\phi(x_{k-1})\to 0$. We can use our decomposition of $\mathcal{L}\phi(x_{k-1})$ and obtain:
\begin{align*}
    \sum_{k=1}^n \gamma_k^2\mathcal{L}\phi(x_{k-1})&=\sum_{k=1}^n \left\{ \gamma_k\left( \phi(x_k)-\phi(x_{k-1})\right)+\gamma_k^3v_{k-1}\cdot\nabla \phi(x_{k-1}) \right.\\
    &\ \left. +\frac{u}{2}\gamma_k^3\nabla f(x_{k-1})\cdot \nabla \phi(x_{k-1})-\frac{1}{2}\gamma_k^3D^2\phi(x_{k-1})v_{k-1}^{\otimes 2} \right\} \\
\end{align*}
Since $\gamma_{k-1}-\gamma_k\sim o(\gamma_k^4)$ and $\{\phi(x_n)\}$ is tight, $\frac{1}{\Gamma_n^{(3)}}\sum_{k=1}^n\gamma_k\left( \phi(x_k)-\phi(x_{k-1})\right)\to 0$. Then we can apply theorem~\ref{theorem4} with $p=6, s=1$ and obtain
\begin{align*}
    \frac{1}{\Gamma_n^{(3)}}\sum_{k=1}^n\gamma_k^2\mathcal{L}\phi(x_{k-1})\to & \int_{\mb{R}^{2d}} v\cdot\nabla \phi(x)\nu(dx,dv)+\frac{u}{2}\int_{\mb{R}^d} \nabla \phi(x)\cdot\nabla f(x)\pi(dx)\\
    &\ -\frac{1}{2}\int_{\mb{R}^{2d}}D^2\phi(x)v^{\otimes 2}\nu(dx,dv)\\
    &=0
\end{align*}
The last identity follows from integration by parts and Fubini theorem. In the same way, we can also prove {$\frac{1}{\Gamma_n^{(4)}}\sum_{k=1}^n\gamma_k^3\mathcal{L}\phi(x_{k-1})\to 0$}.\\
Next, we'll show $\lim_n \frac{1}{\Gamma_n^{(3)}}\sum_{k=1}^n \gamma_k^2(\frac{u}{2}\nabla \phi(x_{k-1})\cdot\nabla f(x_{k-1})-\frac{1}{2}D^2\phi(x_{k-1})v_{k-1}^{\otimes 2})=0$, we'll use the same trick as we did in the proof of theorem~\ref{1dCLTKLMC}. We Taylor expand $\mathcal{L}\phi(x_{k-1})$ at $(x_{k-2}, v_{k-2})$:
\begin{align*}
    \gamma_k^2\mathcal{L}\phi(x_{k-1})&=\gamma_k^2\mathcal{L}\phi(x_{k-2})+\gamma_k^2(\gamma_{k-1}-\gamma_{k-1}^2)D^2\phi(x_{k-2})v_{k-2}^{\otimes 2}\\
    &\ -\frac{u}{2}\gamma_k^2\gamma_{k-1}^2\langle D^2\phi(x_{k-2});v_{k-2},\nabla f(x_{k-2}) \rangle\\
    &\ -\gamma_k^2(2\gamma_{k-1}-2\gamma_{k-1}^2) v_{k-2}\cdot \nabla \phi(x_{k-2}) \\
    &\ -u\gamma_k^2(\gamma_{k-1}-\gamma_{k-1}^2)\nabla f(x_{k-2})\cdot\nabla \phi(x_{k-2})\\
    &\ +\frac{1}{2}\gamma_k^2\gamma_{k-1}^2D^3\phi(x_{k-2})v_{k-2}^{\otimes 3}-2\gamma_k^2\gamma_{k-2}^2 D^2\phi(x_{k-2})v_{k-2}^{\otimes 2}\\
    &\ -u\gamma_k^2 \gamma_{k-1}^2\langle D^2\phi(x_{k-2});v_{k-2},\nabla f(x_{k-2}) \rangle\\
    &\ +\gamma_k^2 \langle D^2\phi(x_{k-2});\sqrt{u}\sigma_{k-1}^{(1)}U_{k-1}^{(1)}, 2\sqrt{u}\sigma_{k-1}^{(2)}U_{k-1}^{(2)}\rangle+o(\gamma_{k-1}^3)
\end{align*}
Since $\gamma_{k-1}-\gamma_k=o(\gamma_k^4)$, we can change $\gamma_k$ on the left hand side to $\gamma_{k-1}$ when we take limits with scale $\Gamma_n^{(4)}$. Apply theorem~\ref{theorem4} with $p=8,s=1$ to terms with order $o(\gamma_k^3)$-coefficients. 
\begin{align*}
    \lim_n \frac{1}{\Gamma_n^{(4)}}\sum_{k=2}^n \gamma_k^2\mathcal{L}\phi(x_{k-1})&=\lim_n \frac{1}{\Gamma_n^{(4)}}\sum_{k=2}^n \gamma_{k-1}^2\mathcal{L}\phi(x_{k-2})-2\lim_n \frac{1}{\Gamma_n^{(4)}}\sum_{k=2}^n \gamma_{k-1}^3\mathcal{L}\phi(x_{k-2})\\
    &\ -2\lim_n \frac{1}{\Gamma_n^{(4)}}\sum_{k=2}^n \gamma_{k-1}^3 (\frac{u}{2}\nabla \phi(x_{k-2})\cdot\nabla f(x_{k-2})-\frac{1}{2}D^2\phi(x_{k-2})v_{k-2}^{\otimes 2})\\
    &\ -3\int_{\mb{R}^{2d}} D^2\phi(x)v^{\otimes 2}\nu(dx,dv)+u\int_{\mb{R}^d}\nabla \phi(x)\cdot \nabla f(x)\pi(dx)\\
    &\ +\lim_n \frac{1}{\Gamma_n^{(4)}}\sum_{k=2}^n \gamma_{k-1}^2 \langle D^2\phi(x_{k-2});\sqrt{u}\sigma_{k-1}^{(1)}U_{k-1}^{(1)}, 2\sqrt{u}\sigma_{k-1}^{(2)}U_{k-1}^{(2)}\rangle
\end{align*}
Since we proved $\frac{1}{\Gamma_n^{(4)}}\sum_{k=1}^n\gamma_k^3\mathcal{L}\phi(x_{k-1})\to 0$ and from Theorem 5, we've shown that
\begin{align*}
    \lim_n \frac{1}{\Gamma_n^{(4)}}\sum_{k=2}^n \gamma_{k-1}^2 \langle D^2\phi(x_{k-2});\sqrt{u}\sigma_{k-1}^{(1)}U_{k-1}^{(1)}, 2\sqrt{u}\sigma_{k-1}^{(2)}U_{k-1}^{(2)}\rangle=2u\int_{\mb{R}^d} \Delta\phi(x)\pi(dx)
\end{align*}
We obtain 
\begin{align*}
    &\lim_n \frac{1}{\Gamma_n^{(4)}}\sum_{k=2}^n \gamma_{k-1}^2 (\frac{u}{2}\nabla \phi(x_{k-2})\cdot\nabla f(x_{k-2})-\frac{1}{2}D^2\phi(x_{k-2})v_{k-2}^{\otimes 2})\\
    &\ =\frac{1}{2} \left[ \lim_n \frac{1}{\Gamma_n^{(4)}}\sum_{k=2}^n \gamma_k^2\mathcal{L}\phi(x_{k-1})-\lim_n \frac{1}{\Gamma_n^{(4)}}\sum_{k=2}^n \gamma_{k-1}^2\mathcal{L}\phi(x_{k-2})\right]\\
    &\ =0
\end{align*}
Therefore, {$\lim_n \frac{1}{\Gamma_n^{(4)}}\sum_{k=1}^n \gamma_{k}^3 (\frac{u}{2}\nabla \phi(x_{k-1})\cdot\nabla f(x_{k-1})-\frac{1}{2}D^2\phi(x_{k-1})v_{k-1}^{\otimes 2})=0$}.\\
To prove the \textbf{Claim}, we need to use the decomposition again:  
\begin{align*}
    \sum_{k=1}^n \gamma_k^2\mathcal{L}\phi(x_{k-1})&=\sum_{k=1}^n \left\{ \gamma_k[\phi(x_k)-\phi(x_{k-1})]+(\gamma_k^3-\frac{2}{3}\gamma_k^4) v_{k-1}\cdot\nabla \phi(x_{k-1}) \right.\\
    &\ +\frac{u}{2}(\gamma_k^3-\frac{2}{3}\gamma_k^4)\nabla f(x_{k-1})\cdot \nabla \phi(x_{k-1})\\
    &\ -\frac{1}{2}(\gamma_k^3-2\gamma_k^4)D^2\phi(x_{k-1})v_{k-1}^{\otimes 2} \\
    &\ -\frac{2u}{3}\gamma_k^4 D^2\phi(x_{k-1}){U_k^{(1)}}^{\otimes 2}\\
    &\ +\frac{u}{2}\gamma_k^4\langle D^2\phi(x_{k-1}); v_{k-1},\nabla f(x_{k-1}) \rangle \\
    &\ \left. -\frac{1}{6}\gamma_k^4D^3\phi(x_{k-1})v_{k-1}^{\otimes 3}+o(\gamma_k^4)\right\} \\
\end{align*}
Since $\{\phi(x_n)\}$ is tight and $\gamma_{k-1}-\gamma_k=o(\gamma_k^4)$, we have $\frac{1}{\Gamma_n^{(4)}}\sum_{k=1}^n \gamma_k(\phi(x_k)-\phi(x_{k-1}))\to 0$. For the terms with coefficients of order $\gamma_k^3$, we can apply theorem~\ref{theorem4} with $p=8,s=1$. Then we obtain:
\begin{align*}
    \lim_n \frac{1}{\Gamma_n^{(4)}}\sum_{k=1}^n \gamma_k^2\mathcal{L}\phi(x_{k-1})&=\lim_{n}\frac{1}{\Gamma_n^{(4)}}\sum_{k=1}^n \gamma_k^3(\mathcal{L}\phi(x_{k-1})+\frac{u}{2}\nabla \phi(x_{k-1})\cdot\nabla f(x_{k-1})-\frac{1}{2}D^2\phi(x_{k-1})v_{k-1}^{\otimes 2}) \\
    &\ -\frac{u}{3}\int_{\mb{R}^d} \nabla\phi(x)\cdot\nabla f(x)\pi(dx)+\int_{\mb{R}^{2d}}D^2\phi(x)v^{\otimes 2}\nu(dx,dv)\\
    &\ -\frac{2u}{3}\int_{\mb{R}^d}\int_{\mb{R}^d} D^2\phi(x) z^{\otimes 2} \mu(dz)\pi(dx) \\
    &=\lim_{n}\frac{1}{\Gamma_n^{(4)}}\sum_{k=1}^n \gamma_k^3(\mathcal{L}\phi(x_{k-1})+\frac{u}{2}\nabla \phi(x_{k-1})\cdot\nabla f(x_{k-1})-\frac{1}{2}D^2\phi(x_{k-1})v_{k-1}^{\otimes 2})\\
    &=0
\end{align*}
The second identity follows from integration by parts and Fubini theorem. The last identity follows from the two statements we just proved. 
\end{proof}

\bibliography{randomizedmp,bib1,bib2}

\newcommand{\etalchar}[1]{$^{#1}$}
\providecommand{\bysame}{\leavevmode\hbox to3em{\hrulefill}\thinspace}
\providecommand{\MR}{\relax\ifhmode\unskip\space\fi MR }
\providecommand{\MRhref}[2]{%
  \href{http://www.ams.org/mathscinet-getitem?mr=#1}{#2}
}
\providecommand{\href}[2]{#2}
\begin{thebibliography}{CCAY{\etalchar{+}}18}

\bibitem[CB18]{cheng2017convergence}
Xiang Cheng and Peter Bartlett, \emph{Convergence of {L}angevin {MCMC} in
  {KL}-divergence}, Algorithmic Learning Theory, 2018, pp.~186--211.

\bibitem[CCAY{\etalchar{+}}18]{cheng2018sharp}
Xiang Cheng, Niladri~S Chatterji, Yasin Abbasi-Yadkori, Peter~L Bartlett, and
  Michael~I Jordan, \emph{Sharp convergence rates for {L}angevin dynamics in
  the nonconvex setting}, arXiv preprint arXiv:1805.01648 (2018).

\bibitem[CCBJ17]{cheng2017underdamped}
Xiang Cheng, Niladri~S Chatterji, Peter~L Bartlett, and Michael~I Jordan,
  \emph{Underdamped {Langevin} {MCMC}: A non-asymptotic analysis}, arXiv
  preprint arXiv:1707.03663 (2017).

\bibitem[CFG14]{chen2014stochastic}
Tianqi Chen, Emily Fox, and Carlos Guestrin, \emph{Stochastic gradient
  {H}amiltonian monte carlo}, International conference on machine learning,
  2014, pp.~1683--1691.

\bibitem[CLW20]{cao2020complexity}
Yu~Cao, Jianfeng Lu, and Lihan Wang, \emph{Complexity of randomized algorithms
  for underdamped {L}angevin dynamics}, arXiv preprint arXiv:2003.09906 (2020).

\bibitem[Dal17]{dalalyan2017furthur}
Arnak Dalalyan, \emph{Further and stronger analogy between sampling and
  optimization: {Langevin Monte Carlo} and gradient descent}, Proceedings of
  the 2017 Conference on Learning Theory, Proceedings of Machine Learning
  Research, vol.~65, PMLR, 07--10 Jul 2017, pp.~678--689.

\bibitem[DK19]{dalalyan2019user}
Arnak~S Dalalyan and Avetik Karagulyan, \emph{User-friendly guarantees for the
  {Langevin Monte Carlo} with inaccurate gradient}, Stochastic Processes and
  their Applications \textbf{129} (2019), no.~12, 5278--5311.

\bibitem[DM17]{durmus2017nonasymptotic}
Alain Durmus and Eric Moulines, \emph{Nonasymptotic convergence analysis for
  the {Unadjusted Langevin algorithm}}, The Annals of Applied Probability
  \textbf{27} (2017), no.~3, 1551--1587.

\bibitem[DM19]{durmus2019high}
\bysame, \emph{High-dimensional {B}ayesian inference via the {Unadjusted
  Langevin} algorithm}, Bernoulli \textbf{25} (2019), no.~4A, 2854--2882.

\bibitem[DMM19]{durmus2019analysis}
Alain Durmus, Szymon Majewski, and Blazej Miasojedow, \emph{Analysis of
  langevin monte carlo via convex optimization.}, Journal of Machine Learning
  Research \textbf{20} (2019), no.~73, 1--46.

\bibitem[DMP18]{durmus2018efficient}
Alain Durmus, Eric Moulines, and Marcelo Pereyra, \emph{Efficient {B}ayesian
  computation by proximal {Markov Chain Monte Carlo}: {When Langevin meets
  Moreau}}, SIAM Journal on Imaging Sciences \textbf{11} (2018), no.~1,
  473--506.

\bibitem[DMS19]{durmus2017convergence}
Alain Durmus, Eric Moulines, and Eero Saksman, \emph{On the convergence of
  {Hamiltonian Monte Carlo}}, The Annals of Statistics (to appear) (2019+).

\bibitem[DRD20]{dalalyan2018sampling}
Arnak Dalalyan and Lionel Riou-Durand, \emph{On sampling from a log-concave
  density using kinetic {Langevin} diffusions}, Bernoulli \textbf{26} (2020),
  no.~3, 1956--1988.

\bibitem[Ebe16]{eberle2016reflection}
Andreas Eberle, \emph{Reflection couplings and contraction rates for
  diffusions}, Probability theory and related fields \textbf{166} (2016),
  no.~3-4, 851--886.

\bibitem[EGZ19]{eberle2017couplings}
Andreas Eberle, Arnaud Guillin, and Raphael Zimmer, \emph{Couplings and
  quantitative contraction rates for langevin dynamics}, The Annals of
  Probability \textbf{47} (2019), no.~4, 1982--2010.

\bibitem[EH20]{erdogdu2020convergence}
Murat~A Erdogdu and Rasa Hosseinzadeh, \emph{On the convergence of {Langevin
  Monte Carlo}: The interplay between tail growth and smoothness}, arXiv
  preprint arXiv:2005.13097 (2020).

\bibitem[EMS18]{erdogdu2018global}
Murat~A Erdogdu, Lester Mackey, and Ohad Shamir, \emph{Global non-convex
  optimization with discretized diffusions}, Advances in Neural Information
  Processing Systems, 2018, pp.~9671--9680.

\bibitem[GDVM16]{gorham2016measuring}
Jackson Gorham, Andrew~B Duncan, Sebastian~J Vollmer, and Lester Mackey,
  \emph{Measuring sample quality with diffusions}, arXiv preprint
  arXiv:1611.06972 (2016).

\bibitem[HLW06]{hairer2006geometric}
Ernst Hairer, Christian Lubich, and Gerhard Wanner, \emph{Geometric numerical
  integration: structure-preserving algorithms for ordinary differential
  equations}, vol.~31, Springer Science \& Business Media, 2006.

\bibitem[KF09]{koller2009probabilistic}
Daphne Koller and Nir Friedman, \emph{Probabilistic graphical models:
  principles and techniques}, MIT press, 2009.

\bibitem[LBBG19]{livingstone2019geometric}
Samuel Livingstone, Michael Betancourt, Simon Byrne, and Mark Girolami,
  \emph{On the geometric ergodicity of {Hamiltonian Monte Carlo}}, Bernoulli
  \textbf{25} (2019), no.~4A, 3109--3138.

\bibitem[LCCC16]{li2016preconditioned}
Chunyuan Li, Changyou Chen, David Carlson, and Lawrence Carin,
  \emph{Preconditioned {Stochastic Gradient Langevin} dynamics for deep neural
  networks}, Thirtieth AAAI Conference on Artificial Intelligence, 2016.

\bibitem[Liu08]{liu2008monte}
Jun~S Liu, \emph{Monte {C}arlo strategies in scientific computing}, Springer
  Science \& Business Media, 2008.

\bibitem[LM16]{leimkuhler2016molecular}
Ben Leimkuhler and Charles Matthews, \emph{Molecular dynamics: With
  deterministic and stochastic numerical methods}, Springer, 2016.

\bibitem[LP02]{lamberton2002recursive}
Damien Lamberton and Gilles Pages, \emph{Recursive computation of the invariant
  distribution of a diffusion}, Bernoulli \textbf{8} (2002), no.~3, 367--405.

\bibitem[LS16]{lelievre2016partial}
Tony Lelievre and Gabriel Stoltz, \emph{Partial differential equations and
  stochastic methods in molecular dynamics}, Acta Numerica \textbf{25} (2016),
  681--880.

\bibitem[LT93]{luo1993error}
Zhi-Quan Luo and Paul Tseng, \emph{Error bounds and convergence analysis of
  feasible descent methods: a general approach}, Annals of Operations Research
  \textbf{46} (1993), no.~1, 157--178.

\bibitem[LWME19]{li2019stochastic}
Xuechen Li, Yi~Wu, Lester Mackey, and Murat~A Erdogdu, \emph{Stochastic
  {Runge-Kutta} accelerates {Langevin Monte Carlo} and beyond}, Advances in
  Neural Information Processing Systems, 2019, pp.~7748--7760.

\bibitem[MCC{\etalchar{+}}19]{ma2019there}
Yi-An Ma, Niladri Chatterji, Xiang Cheng, Nicolas Flammarion, Peter Bartlett,
  and Michael~I Jordan, \emph{Is there an analog of {Nesterov} acceleration for
  {MCMC}?}, arXiv preprint arXiv:1902.00996 (2019).

\bibitem[MHB17]{mandt2017stochastic}
Stephan Mandt, Matthew~D Hoffman, and David~M Blei, \emph{Stochastic gradient
  descent as approximate {B}ayesian inference}, The Journal of Machine Learning
  Research \textbf{18} (2017), no.~1, 4873--4907.

\bibitem[MMW{\etalchar{+}}19]{mou2019high}
Wenlong Mou, Yi-An Ma, Martin~J Wainwright, Peter~L Bartlett, and Michael~I
  Jordan, \emph{High-order {L}angevin diffusion yields an accelerated {MCMC}
  algorithm}, arXiv preprint arXiv:1908.10859 (2019).

\bibitem[MPM{\etalchar{+}}20]{mazumdar2020thompson}
Eric Mazumdar, Aldo Pacchiano, Yi-an Ma, Peter~L Bartlett, and Michael~I
  Jordan, \emph{On thompson sampling with langevin algorithms}, arXiv preprint
  arXiv:2002.10002 (2020).

\bibitem[MSH02]{mattingly2002ergodicity}
Jonathan~C Mattingly, Andrew~M Stuart, and Desmond~J Higham, \emph{Ergodicity
  for {SDE}s and approximations: {L}ocally lipschitz vector fields and
  degenerate noise}, Stochastic processes and their applications \textbf{101}
  (2002), no.~2, 185--232.

\bibitem[MT12]{meyn2012markov}
Sean~P Meyn and Richard~L Tweedie, \emph{Markov chains and stochastic
  stability}, Springer Science \& Business Media, 2012.

\bibitem[MT13]{milstein2013stochastic}
Grigori~Noah Milstein and Michael~V Tretyakov, \emph{Stochastic numerics for
  mathematical physics}, Springer Science \& Business Media, 2013.

\bibitem[Nea11]{neal2011mcmc}
Radford~M Neal, \emph{{MCMC} using {H}amiltonian dynamics}, Handbook of {Markov
  chain Monte Carlo} \textbf{2} (2011), no.~11, 2.

\bibitem[RRT17]{raginsky2017non}
Maxim Raginsky, Alexander Rakhlin, and Matus Telgarsky, \emph{Non-convex
  learning via {Stochastic Gradient Langevin Dynamics}: a nonasymptotic
  analysis}, Proceedings of the 2017 Conference on Learning Theory, vol.~65,
  2017, pp.~1674--1703.

\bibitem[RT96]{roberts1996exponential}
Gareth~O Roberts and Richard~L Tweedie, \emph{Exponential convergence of
  {Langevin} distributions and their discrete approximations}, Bernoulli
  \textbf{2} (1996), no.~4, 341--363.

\bibitem[SL19]{shen2019randomized}
Ruoqi Shen and Yin~Tat Lee, \emph{The randomized midpoint method for
  log-concave sampling}, Advances in Neural Information Processing Systems,
  2019, pp.~2098--2109.

\bibitem[SZ19]{sabanis2019higher}
Sotirios Sabanis and Ying Zhang, \emph{Higher order {Langevin Monte Carlo}
  algorithm}, Electronic Journal of Statistics \textbf{13} (2019), no.~2,
  3805--3850.

\bibitem[TTV16]{teh2016consistency}
Yee~Whye Teh, Alexandre~H Thiery, and Sebastian~J Vollmer, \emph{Consistency
  and fluctuations for {Stochastic Gradient Langevin Dynamics}}, The Journal of
  Machine Learning Research \textbf{17} (2016), no.~1, 193--225.

\bibitem[Vem10]{vempala2010recent}
Santosh~S Vempala, \emph{Recent progress and open problems in algorithmic
  convex geometry}, IARCS Annual Conference on Foundations of Software
  Technology and Theoretical Computer Science (FSTTCS 2010), Schloss
  Dagstuhl-Leibniz-Zentrum fuer Informatik, 2010.

\bibitem[Vil09]{villani2009wasserstein}
C{\'e}dric Villani, \emph{The {W}asserstein distances}, Optimal Transport,
  Springer, 2009, pp.~93--111.

\bibitem[VW19]{vempala2019rapid}
Santosh Vempala and Andre Wibisono, \emph{Rapid convergence of the {Unadjusted
  Langevin Algorithm: Isoperimetry Suffices}}, Advances in Neural Information
  Processing Systems, 2019, pp.~8092--8104.

\bibitem[WT11]{welling2011bayesian}
Max Welling and Yee~W Teh, \emph{Bayesian learning via {Stochastic Gradient
  Langevin Dynamics}}, Proceedings of the 28th international conference on
  machine learning (ICML-11), 2011, pp.~681--688.

\end{thebibliography}
\bibliographystyle{amsalpha}
\end{document}